\documentclass{article}

\usepackage[main, final]{neurips_2026}

\usepackage[utf8]{inputenc}
\usepackage[T1]{fontenc}
\usepackage{hyperref}
\hypersetup{pdftitle={Language-Induced Priors for Domain Adaptation}, pdfauthor={Qiyuan Chen, Jiayu Zhou, Raed Al Kontar}}
\usepackage{url}
\usepackage{booktabs}
\usepackage{amsfonts}
\usepackage{nicefrac}
\usepackage{microtype}
\usepackage{xcolor}

\usepackage{amsmath}
\usepackage{amssymb}
\usepackage{amsthm}
\usepackage{mathtools}
\usepackage{graphicx}
\usepackage{enumerate}
\usepackage{subcaption}
\usepackage{tikz}
\usetikzlibrary{shapes.geometric, arrows.meta, positioning, fit, calc, backgrounds, matrix}
\usepackage{algorithm}
\usepackage{algorithmic}
\usepackage[capitalize,noabbrev]{cleveref}
\usepackage{enumitem}
\usepackage{wrapfig}

\theoremstyle{plain}
\newtheorem{theorem}{Theorem}[section]
\newtheorem{proposition}[theorem]{Proposition}
\newtheorem{lemma}[theorem]{Lemma}
\newtheorem{corollary}[theorem]{Corollary}
\theoremstyle{definition}

\newtheorem{assumption}[theorem]{Assumption}
\theoremstyle{remark}
\newtheorem{remark}[theorem]{Remark}
\theoremstyle{plain}
\newcommand{\formalthmname}{Theorem}
\newtheorem*{formalthm}{\formalthmname}

\DeclareMathOperator*{\argmax}{arg\,max}

\title{Language-Induced Priors for Domain Adaptation}

\author{%
  Qiyuan Chen\quad Jiayu Zhou\quad Raed Al Kontar \\
  University of Michigan \\
  \texttt{\{cqiyuan, jiayuz, alkontar\}@umich.edu}
}

\begin{document}

\maketitle

\begin{abstract}
Domain adaptation faces a fundamental paradox in the cold-start regime. When target data is scarce, statistical methods fail to distinguish relevant source domains from irrelevant ones, which often leads to negative transfer. In this paper, we address this challenge by leveraging expert textual descriptions of the target domain, a resource that is often available but overlooked. We propose a probabilistic framework that translates these semantic descriptions into a choice model, namely a Language-Induced Prior (LIP), that learns the preferences from a pretrained Large Language Model. The LIP is then integrated into an Expectation-Maximization algorithm to identify source relevance. Methodologically, this framework is compatible with any parametric model where a likelihood is available. It allows the LIP to guide the selection of sources when target signals are weak, while gradually refining these choices as samples accumulate. Theoretically, we prove that the estimator roughly matches an oracle cold-start mean squared error under a correct prior, while remaining asymptotically consistent regardless of the quality of the LIP. Empirically, we validated the framework on a descriptive (Gaussian estimation), a predictive (C-MAPSS dataset), and a prescriptive task (MuJoCo Hopper).
\end{abstract}

\section{Introduction}

The success of modern machine learning is driven by the abundance of data, but it can be fragile in the cold-start regime. When a system starts its task in a new environment, whether it is a new machine deployed to an assembly line or a diagnostic model for a new patient, it becomes statistically challenging to estimate an accurate model from the limited observations. The standard remedy for this problem is Domain Adaptation (DA). DA leverages the abundant data from existing systems that potentially operate under different environments. These environments from which data is borrowed are called source domains, while the new environment we care about is called the target domain.

However, the mere existence of source data is not enough. In realistic settings, source domains are usually heterogeneous, including both relevant and irrelevant domains. If a learner indiscriminately pools data from all available sources, it can suffer from negative transfer, where the borrowed knowledge biases the model and results in even worse performance than a simple model trained on the limited target data alone. For example, a machine operating in a humid environment can corrode faster than one operating in a dry environment. If we train a model for a humid environment with the degradation data generated in a dry one, the predicted remaining life can be significantly longer than its actual value.

Although identifying the relevant sources is important for DA, it is not easier than learning a good model with limited data. In fact, if one can correctly identify the relevant sources, then an accurate model can be trained using the data from these sources. This creates a paradox in the cold-start regime. To reliably borrow data from the source domains, the learner needs a reliable understanding of the target domain, which is unrealistic given the limited data. Fundamentally, regardless of how rich the source datasets are, the performance of the resulting model is bottlenecked by the limited information regarding the target domain. This calls for attention beyond the source datasets.

Fortunately, strictly numerical data is rarely the only information available when deploying a new system. In many practical applications, we have some contextual information about the target domain. This contextual information, usually taking the form of expert textual descriptions, is a rich modality of information that has been overlooked in the DA literature. A maintenance log might note that a machine is ``vibrating more than usual,'' or a clinical chart might describe a patient as ``working irregular night shifts.'' While these descriptions do not contain labeled training samples, they induce a strong prior regarding the underlying physics of the environment.

Humans naturally evade the cold-start paradox by leveraging this context. An engineer need not observe a hundred failure events to know that a vibrating machine fails more often. Semantic descriptions immediately remind us of relevant prior experiences. In this paper, we bridge the gap between semantic reasoning and statistical learning. Our contributions are summarized as follows.
\begin{itemize}[leftmargin=*]
    \item \textbf{Language-Induced Prior (LIP).} To the best of our knowledge, this is the first work to define and leverage LIP for identifying source relevance in DA (Sec.~\ref{sec:related_work}). We use a choice model, fitted to the responses of a pretrained Large Language Model (LLM), to quantify the relevance of the sources, which guides DA in the cold-start regime.
    \item \textbf{LIP-EM.} LIP is integrated into Multi-Source Domain Adaptation (MDA) through a Bayesian hierarchical model with latent source relevance, which is solved by a tempered Expectation-Maximization (EM) algorithm. This allows the model to be guided by expert intuition when the target signals are weak, while automatically refining the source relevance as data accumulates.
    \item \textbf{Theoretical guarantees.} Our theory supports the methodological claims. Theorem~\ref{thm:fs_main} shows that, under a correct prior, the tempered EM limit behaves as if it were trained on all the relevant sources, demonstrating why and how semantic knowledge helps. Our asymptotic analysis shows that both the E-step (Theorem~\ref{thm:asymptotic_relevance}) and the M-step (Theorem~\ref{thm:asymptotic_consistency}) are consistent under any LIP, including the incorrect ones.
    \item \textbf{Empirical validation.} We empirically validate our approach on three tasks: an illustrative example of Gaussian estimation (both 1-d and 2-d), a case study on the NASA Commercial Modular Aero-Propulsion System Simulation (C-MAPSS) dataset where no approximations are needed, and a deep reinforcement learning example for MuJoCo Hopper. Our LIP-EM shows strong performance, especially when the target data is scarce. Our code is available at \mbox{\url{https://github.com/UMDataScienceLab/LIP-EM}}.
\end{itemize}

\subsection{Related Work}
\label{sec:related_work}
Our work sits at an unexplored intersection of three lines, namely multi-source domain adaptation, LLM-based knowledge elicitation, and metadata-aware domain generalization.

\paragraph{Multi-Source Domain Adaptation.}
Our problem setup is closest to MDA, which leverages labeled data from $K$ distinct source domains to learn a model for a target domain. A fundamental challenge in MDA is \emph{negative transfer} \citep{wang2019characterizing,kontar2020minimizing,zhang2022survey}, where the inclusion of irrelevant or adversarial sources degrades target performance \citep{pan2009survey, zhao2024more}. Two prior lines distinguish sources by exploiting target features. The first reweights sources by feature proximity \citep{sun2011two}, moment alignment \citep{peng2019moment}, adversarial training \citep{zhao2018adversarial}, generalization bounds \citep{wen2020domain}, or theoretically optimal convex combinations \citep{mansour2008domain}. The second uses EM-style alternation to \emph{discover} unknown domain structure within source data \citep{hoffman2012discovering, xiong2014latent}. Both lines are data-driven and become unreliable when target data is scarce, since the discrepancies and latent clusters they rely on cannot be reliably estimated from a small target sample. We instead model source relevance as a \emph{binary} latent indicator whose prior, the LIP, needs no target data. This lets source selection proceed in the cold-start regime.

\paragraph{LLMs as Expert Reasoners.}
Historically, translating qualitative expert intuition into quantitative probability distributions required extensive manual labor \citep{garthwaite2005statistical, ohagan2006uncertain}. The premise of leveraging natural language for statistical inference is rooted in treating LLMs as implicit knowledge bases and expert reasoners \citep{petroni2019language}. Because LLMs are pretrained on a vast corpus of digital textual knowledge, they exhibit strong zero-shot reasoning capabilities across specialized domains \citep{brown2020language, kojima2022large}. The closest precursor is AutoElicit \citep{capstick2025autoelicit}, which queries an LLM directly for the numerical values of a prior over linear model parameters. In contrast, we ask the LLM to answer a multiple choice question over subgroups of sources rather than to state numerical probabilities. We then fit a conditional logit model with an outside option to the resulting preference data to recover the prior. The multiple choice format builds on evidence that LLMs are more reliable at discrete classification than at numerical estimation \citep{gu2024probabilistic}, and produces a prior even when the LLM cannot articulate explicit numbers, as similarly observed in concurrent work using LLM-elicited binary judgments for Bayesian graphical modeling \citep{sekulovski2026llm}.

\paragraph{Metadata- and Language-Guided Methods.} Metadata-aware domain generalization leverages supplementary domain attributes to capture inter-domain relationships. \citet{zhu2026domain} use domain metadata as an additional input, and $\mathrm{D^3G}$ \citep{yao2024improving} builds domain similarity matrices from structured metadata. Language-guided methods such as LanDA \citep{wang2025landa} and CoPT \citep{mata2024copt} align text descriptions with image features in the shared space of a vision-language model, which makes them task-specific and restricted to the modalities on which the vision-language model was pretrained. Both groups of methods require metadata or text descriptions of \emph{all} domains. We instead need text only for the target, delegate the relevance reasoning to an LLM, and remain model-agnostic. Appendix~\ref{sec:appendix_related_work} discusses these methods in more detail.

\section{Setting}
\paragraph{Problem Formulation.}
Consider $K+1$ independent domains indexed by $k$, where the behavior of domain $k$ is controlled by a set of latent parameters $\theta_k$ of dimension $d$. Without loss of generality, we let domain $k=0$ be the target domain, and the rest of the domains $k\in\{1,\dots,K\}$ are considered the source domains. For each domain $k$, we denote its dataset by $\mathcal{D}_k$, which contains $N_k$ i.i.d.\ data points generated by parameter $\theta_k$. In predictive tasks, every $\mathcal{D}_k$, including the small target sample $\mathcal{D}_0$, consists of labeled input--response pairs. Note that the source domains are purely for borrowing knowledge, and we care only about the performance in the target domain, which boils down to the estimation of $\theta_0$. We note that this problem setup is a general framework for any likelihood-based parameter estimation problem.

\begin{wrapfigure}{r}{0.49\textwidth}
\centering
\resizebox{\linewidth}{!}{
\begin{tikzpicture}[
    node distance=1.2cm and 2cm,
    >={Stealth[length=2mm]},
    observed/.style={circle, draw=black, fill=gray!20, thick, minimum size=1.0cm},
    latent/.style={circle, draw=black, thick, minimum size=1.0cm},
    fixed/.style={rectangle, draw=black, fill=black!5, thick, minimum size=0.9cm},
    plate/.style={draw=gray, dashed, thick, rounded corners, inner sep=0.4cm, label={[text=gray, anchor=south east]south east:#1}}
]

\node[fixed] (pi) {$\pi_k$};

\node[latent, right=of pi] (c) {$c_k$};

\node[latent, right=of c] (theta_k) {$\theta_k$};

\node[observed, right=of theta_k] (Dk) {$\mathcal{D}_k$};

\node[latent, above=of theta_k] (theta) {$\theta_0$};

\node[observed, right=of theta] (D0) {$\mathcal{D}_0$};

\node[fixed, above=of c] (phi) {$\phi_{\text{null}}$};

\draw[->] (pi) -- (c);
\draw[->] (c) -- (theta_k);
\draw[->] (theta_k) -- (Dk);

\draw[->] (theta) -- (theta_k) node[midway, right] {\scriptsize $\mathcal{N}(\theta_0,\tau^2I)$};
\draw[->] (theta) -- (D0);

\draw[->] (phi) -- (theta_k);

\node[plate={$k=1, \dots, K$}, fit=(pi)(c)(theta_k)(Dk)] {};

\end{tikzpicture}
}
\caption{Generative Directed Acyclic Graph. Shaded circles are observed, white circles are latent, and squares are fixed.}
\label{fig:dag}
\vspace{-1em}
\end{wrapfigure}

We model the data generating process of an MDA problem as the Bayesian hierarchical model shown in Fig.~\ref{fig:dag}. A prior $\pi_k$ generates a latent indicator variable $c_k \in \{0,1\}$, which determines whether source domain $k$ shares a similar distribution with the target, i.e.,
\begin{equation}
    \theta_k \sim \begin{cases}
        \mathcal{N}(\theta_0, \tau^2 I) & c_k=1\\
        \phi_{\text{null}} & c_k=0
    \end{cases}.
\end{equation}

If $c_k=1$, then the parameter of domain $k$ is sampled from an isotropic normal distribution centered at $\theta_0$ with per-coordinate variance $\tau^2$, where the hyperparameter $\tau$ is a \emph{small} number that defines how similar the relevant domains are. In the special case where $\tau=0$, it is enforced that all the relevant source domains should have the exact same parameter $\theta_k=\theta_0$. If $c_k=0$, domain $k$ has a domain-specific parameter $\theta_k$ generated by the null component $\phi_{\text{null}}$. 

\paragraph{Motivation for the LIP.} As highlighted earlier, determining the memberships $c_k$ in the cold-start regime presents a circular challenge: when $N_0$ is small, identifying relevant sources requires a reliable estimate of $\theta_0$, yet estimating $\theta_0$ requires knowing which sources to borrow from. We break this dilemma by supplying additional information through a \emph{Language-Induced Prior} (LIP) $\pi = (\pi_1, \ldots, \pi_K)$, where $\pi_k = \mathbb{P}(c_k = 1)$. LIP leverages a textual description of the target domain to encode source relevance before any target data is processed. We note that our framework requires a text description \textit{only} for the target domain. This reflects how such information arises in practice: domain experts typically describe the current system of interest (e.g., ``the machine is vibrating abnormally''), but equivalent descriptions for historical source datasets are rarely recorded and are infeasible to reconstruct retrospectively. While experts could, in principle, assess each source's relevance from the target description, doing so manually is costly and does not scale with the number of sources $K$. We propose to use pretrained LLMs as a practical alternative, leveraging their reasoning ability and broad pretrained knowledge to parse the description and assess relevance.

\section{Methodology}

\subsection{LIP Construction}
\label{sec:lip}
To construct the LIP, we use empirical Bayes to find the prior that maximizes the likelihood of the LLM responses. In this step, one can be creative in the elicitation method, and any prior (even one that is not created by an LLM) is compatible with the proposed tempered EM algorithm in Sec.~\ref{sec:em_summary}. Here, we propose one such method based on subgroup preference elicitation (see details in Appendix~\ref{sec:appendix_lip}). The method consists of two steps (see Fig.~\ref{fig:lip_pipeline}). First, we elicit a preference dataset $M$ from an LLM judge. Then, the LIP $\pi$ is fitted by maximizing its likelihood on $M$.

\begin{figure}[b]
    \centering
    \includegraphics[width=\linewidth]{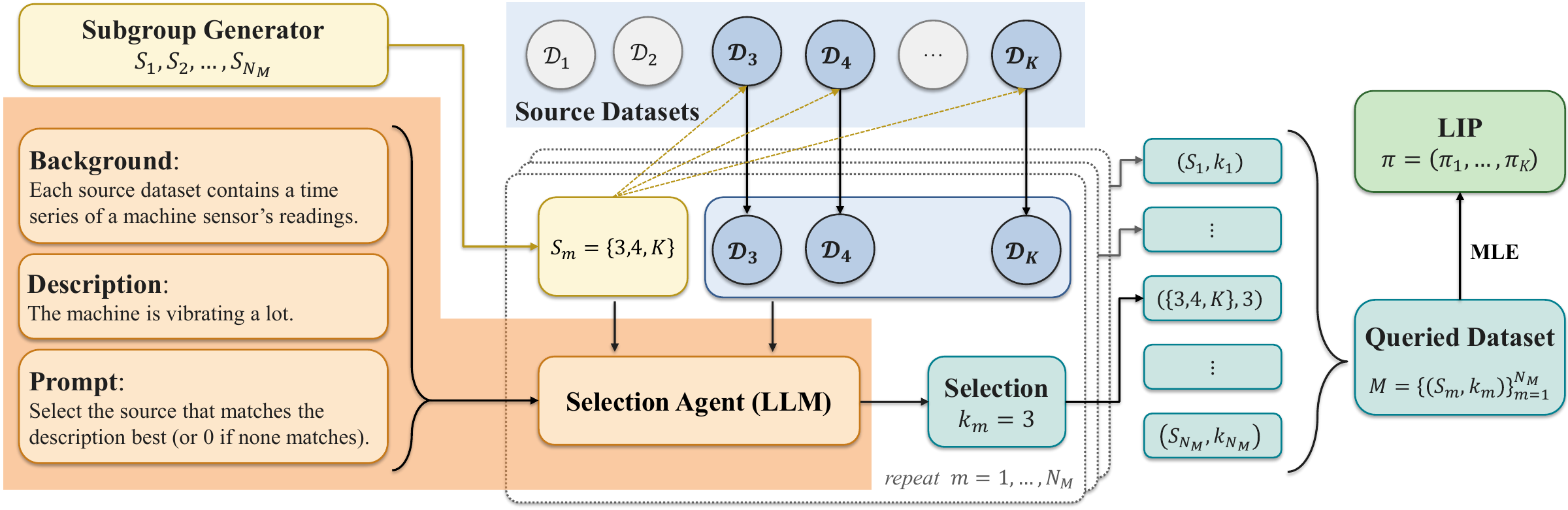}
    \caption{LIP pipeline. For sampled subgroups $S_m$ of source indices, the LLM sees $h$ and the raw data $\{\mathcal{D}_k\}_{k\in S_m}$ and returns the best match $k_m \in S_m$ or $k_m = 0$. Then \eqref{eq:LIP_opt} fits $\pi$ to these choices.}
    \label{fig:lip_pipeline}
\end{figure}

\paragraph{Subgroup Preference Elicitation.}
For each query $m = 1, \dots, N_M$, we sample a subgroup $S_m \subseteq \{1, \ldots, K\}$ of source indices and ask the LLM a multiple choice question: select the most relevant source $k_m \in S_m$ given the contextual knowledge $h$ (target description and task background) and the raw data $\{\mathcal{D}_k\}_{k \in S_m}$, or return $k_m = 0$ if no candidate is deemed relevant. We record the responses as $M = \{(S_m, k_m)\}_{m=1}^{N_M}$.

\paragraph{Empirical Bayes for LIP.}
LIP is parameterized as $\pi_k = \sigma(\alpha_k)$, where $\sigma$ is the sigmoid function and $\alpha_k \in \mathbb{R}$ measures how strongly the LLM favors source $k$. Because the LLM is restricted to selecting at most one candidate per query, we model its choice as a \emph{conditional logit with an outside option} \citep{mcfadden1974conditional,luce1959individual}. This outside option is the null option $k_m = 0$, a virtual alternative with $\alpha_0$ as the selection threshold. Concretely, the probability that the LLM returns choice $k_m \in S_m \cup \{0\}$ given subgroup $S_m$ is 
\begin{equation}
\label{eq:lip_choice}
    p(k_m \mid S_m) = \frac{e^{\alpha_{k_m}}}{e^{\alpha_0} + \sum_{j \in S_m} e^{\alpha_j}},
\end{equation}
where $\alpha_{k_m} = \alpha_0$ when $k_m = 0$ (the null is selected) and $\alpha_{k_m} = \alpha_k$ when a source $k$ is selected.

To find the most probable LIP, we minimize the regularized negative log-likelihood:
\begin{equation}
\label{eq:LIP_opt}
    \min_{\alpha_0, \alpha_1, \dots, \alpha_K}\quad -\sum_{m=1}^{N_M} \log p(k_m \mid S_m) + \lambda_{\mathrm{LIP}} \sum_{k=1}^K \left(\alpha_k - \log\frac{p_0}{1-p_0}\right)^2,
\end{equation}
where $\lambda_{\mathrm{LIP}}$ is an $L_2$ regularization coefficient and $p_0$ is the default probability for a uniform prior.

\subsection{Optimization with LIP-EM}
\label{sec:em_summary}
Recall that the goal in our problem setup is to estimate an accurate $\theta_0$. At this stage, we have both the LIP $\pi$ and the observed data $\mathbf{O} \triangleq \{\mathcal{D}_0, \dots, \mathcal{D}_K\}$. Under a probabilistic model with latent parameters, a natural choice to infer the memberships is EM. We denote the latent memberships as $\mathbf{c} \triangleq \{c_1, \dots, c_K\}$ and infer them jointly with $\theta_0$ from the complete data likelihood
\begin{equation}
    p(\mathbf{O}, \mathbf{c} \mid \theta) = \mathcal{L}(\theta; \mathcal{D}_0)\prod_{k=1}^K p(\mathcal{D}_k\mid c_k=1,\theta)^{c_k}\,p(\mathcal{D}_k\mid c_k=0)^{1-c_k}\,\pi_k^{c_k}(1-\pi_k)^{1-c_k}.
\end{equation}
EM alternates between an E-step, which updates the posterior of $\mathbf{c}$ given the current iterate, and an M-step, which maximizes the expected complete data log-likelihood under this posterior:
\begin{equation}
\label{eq:raw_opt}
    \theta^{(t+1)} \gets \argmax_{\theta} \mathbb{E}_{\mathbf{c}\mid\mathbf{O},\theta^{(t)}}\left[\log p(\mathbf{O}, \mathbf{c} \mid \theta)\right].
\end{equation}

\subsubsection{The E-Step (Expectation)}
\label{sec:E-step}
To optimize \eqref{eq:raw_opt}, we need to first compute the objective. Using the Bernoulli form of $c_k$ ($\mathbb{P}(c_k) = \pi_k^{c_k}(1-\pi_k)^{1-c_k}$) and dropping $\theta$-independent terms (see Appendix~\ref{sec:E-step-derive}), we can reduce \eqref{eq:raw_opt} to
\begin{equation}
\label{eq:opt}
    \max_{\theta}~ \log\mathcal{L}(\theta;\mathcal{D}_0) + \sum_{k=1}^K w_k^{(t)}\log p(\mathcal{D}_k\mid c_k=1, \theta),
\end{equation}
i.e., a weighted sum of log-likelihoods with weights $w_k^{(t)} \triangleq \mathbb{E}[c_k\mid\mathcal{D}_k,\theta^{(t)}] = \mathbb{P}(c_k=1\mid\mathcal{D}_k,\theta^{(t)})$, which Bayes' rule gives in a numerically stable sigmoid form:
\begin{equation}
\label{eq:E-step}
    w_k^{(t)} = \sigma\left( \log\frac{p(\mathcal{D}_k \mid c_k=1, \theta^{(t)})}{p(\mathcal{D}_k \mid c_k=0)} + \log\frac{\pi_k}{1-\pi_k}\right).
\end{equation}
In a nutshell, the log-likelihood ratio (log-LR) $\log [p(\mathcal{D}_k\mid c_k=1,\theta^{(t)})/p(\mathcal{D}_k \mid c_k=0)]$ corrects the prior $\pi_k$ by the empirical evidence. When the relevant likelihood $p(\mathcal{D}_k\mid c_k=1,\theta^{(t)})$ is larger than the null likelihood $p(\mathcal{D}_k \mid c_k=0)$, the log-LR is positive and raises the weight and otherwise lowers it. Now what remains is computing these two likelihoods.

\paragraph{Relevant Likelihood.} Given $\theta^{(t)}$, $p(\mathcal{D}_k\mid c_k=1, \theta^{(t)})$ is obtained by marginalizing out $\theta_k$:
\begin{equation*}
    p(\mathcal{D}_k\mid c_k=1, \theta^{(t)}) = \int \mathcal{L}(\theta_k;\mathcal{D}_k)\,\mathcal{N}(\theta_k\mid\theta^{(t)},\tau^2 I)\,d\theta_k,
\end{equation*}
which is generally intractable. Monte Carlo sampling of $\theta_k$ is efficient only in low dimensions, so for large $d$ we approximate the likelihood with a second-order Taylor expansion of $\log\mathcal{L}(\,\cdot\,;\mathcal{D}_k)$ at $\theta^{(t)}$:
\begin{equation}
\label{prox:taylor}
    \log\mathcal{L}(\theta;\mathcal{D}_k)\approx \log\mathcal{L}(\theta^{(t)};\mathcal{D}_k) + {g_k^{(t)}}^\intercal (\theta-\theta^{(t)}) - \tfrac{1}{2}(\theta-\theta^{(t)})^\intercal H_k^{(t)} (\theta-\theta^{(t)}),
\end{equation}
where $g_k^{(t)} \triangleq \nabla_\theta\log\mathcal{L}(\theta;\mathcal{D}_k)|_{\theta=\theta^{(t)}}$ and $H_k^{(t)} \triangleq -\nabla^2_\theta\log\mathcal{L}(\theta;\mathcal{D}_k)|_{\theta=\theta^{(t)}}$. Under this approximation, the relevant log-likelihood takes the closed form \eqref{eq:rel_prox} (Proposition~\ref{prop:marginal_likelihood_thetat}, proved in Appendix~\ref{sec:proof_marginal_likelihood}).

\begin{proposition}
\label{prop:marginal_likelihood_thetat}
Fix $\tau > 0$ and $\theta^{(t)} \in \mathbb{R}^d$, and assume $I + \tau^2 H_k^{(t)} \succ 0$. Approximation~\eqref{prox:taylor} yields
\begin{equation}
\label{eq:rel_prox}
    \log p(\mathcal{D}_k \mid c_k=1, \theta^{(t)}) \approx \log\mathcal{L}(\theta^{(t)};\mathcal{D}_k) + \tfrac{\tau^2}{2}\,{g_k^{(t)}}^\intercal(I+\tau^2 H_k^{(t)})^{-1}g_k^{(t)} - \tfrac{1}{2}\log\det(I+\tau^2 H_k^{(t)}).
\end{equation}
\end{proposition}

\begin{remark}[Zero-shift limit]
\label{rmk:zero_shift_limit}
At $\tau = 0$, both sides of \eqref{eq:rel_prox} reduce to $\log \mathcal{L}(\theta^{(t)}; \mathcal{D}_k)$, so the approximation holds with equality.
\end{remark}

\begin{remark}[Exactness for Gaussian likelihood]
\label{rmk:exact_gaussian_marginal_general}
Suppose $p(x \mid \theta_k) = \mathcal{N}(x;\, \theta_k, \Sigma_k)$ with known $\Sigma_k \succ 0$. Then \eqref{prox:taylor} holds with equality at every $\theta^{(t)}$, and consequently \eqref{eq:rel_prox} holds with equality.
\end{remark}

Remarks~\ref{rmk:zero_shift_limit} and~\ref{rmk:exact_gaussian_marginal_general} are proved in Appendices~\ref{sec:proof_zero_shift} and~\ref{sec:proof_gaussian_exact}. When $\log\mathcal{L}(\theta;\mathcal{D}_k)$ is not quadratic in $\theta$, \eqref{eq:rel_prox} is only approximate. It is accurate when the remainder of the Taylor expansion \eqref{prox:taylor} is negligible, which holds for small $\tau$. Note that a larger $\tau$ relaxes the relevance criterion and includes more sources, but at the cost of additional estimation error (see the discussion in Sec.~\ref{sec:theory}). The curvature condition $I+\tau^2H_k^{(t)}\succ0$ of Proposition~\ref{prop:marginal_likelihood_thetat} is weaker than $H_k^{(t)}\succeq0$ and matters only for $\tau>0$. Algorithm~\ref{alg:lip_em} (Appendix~\ref{sec:appendix_algorithms}) satisfies it for every $\tau$, because the Hessian it uses, taken at the maximum likelihood estimate of each source, is positive semidefinite. The variant for neural networks sets $\tau=0$ and so needs no Hessian (Sec.~\ref{sec:M-step}).

\paragraph{Null Likelihood.} $p(\mathcal{D}_k\mid c_k=0)$
is determined by the null component $\phi_{\mathrm{null}}$, which acts as a background model: a source $k$ is deemed relevant only if the current model explains it better than the null does. The null is a density, not a scalar threshold, and it need not be the true distribution of the irrelevant sources. A null is useful if the model with parameter $\theta_0$ explains each relevant source better than the null does and each irrelevant source worse. With such a null and $\tau=0$, the E-step weight at $\theta^{(t)}=\theta_0$ converges to the true membership $c_k$ as $N_k$ grows (Theorem~\ref{thm:asymptotic_relevance}). However, a null likelihood set too low fails to suppress irrelevant sources and triggers negative transfer. It is therefore safer to err on the side of a larger null likelihood. Appendix~\ref{sec:null-heuristic} gives two practical constructions built from source model fits, which Sec.~\ref{sec:experiments} uses without any tuning on target data.

\subsubsection{The M-Step (Maximization)}
\label{sec:M-step}
Having computed the relevance weights $w_k^{(t)}$ in the E-step, we can now solve the maximization problem in \eqref{eq:opt}. We denote its objective by $Q(\theta\mid\theta^{(t)})$.
Replacing the relevant log-likelihood with \eqref{eq:rel_prox} gives
\begin{align}
\max_{\theta^{(t+1)}}~ &\log\mathcal{L}(\theta^{(t+1)};\mathcal{D}_0) + \sum_k w_k^{(t)}\log\mathcal{L}(\theta^{(t+1)};\mathcal{D}_k) \notag \\
&+ \sum_k w_k^{(t)}\left[\,\frac{\tau^2}{2}{g_k^{(t+1)}}^\intercal\left(I+\tau^2 H_k^{(t+1)}\right)^{-1}g_k^{(t+1)} - \frac{1}{2}\log\det\left(I +\tau^2H_k^{(t+1)}\right)\right],\label{eq:final_q_function}
\end{align}
which can be interpreted as maximizing the weighted log-likelihood regularized by a $\tau^2$-dependent term. Yet, solving \eqref{eq:final_q_function} exactly requires computing the gradient $g_k^{(t+1)}$ and inverting a costly $d\times d$ matrix, which is not practical for high-dimensional $\theta_0$. As such, we introduce two approaches for fast computations. 

First, we re-expand each $\log\mathcal{L}(\,\cdot\,;\mathcal{D}_k)$, $k=0,\dots,K$, around the local Maximum Likelihood Estimate (MLE) $\hat\theta_k \triangleq \arg\max_{\theta_k}\mathcal{L}(\theta_k;\mathcal{D}_k)$ rather than around the current iterate $\theta^{(t)}$. This shift of the expansion center makes the gradient term vanish (since the gradient is $0$ at the MLE) and replaces the iteration-dependent Hessian $H_k^{(t)}$ with the iteration-\emph{independent} Hessian at the MLE:
\begin{equation*}
    \log\mathcal{L}(\theta;\mathcal{D}_k) \approx \mathrm{const} - \tfrac{1}{2}(\theta-\hat\theta_k)^\intercal H_k(\theta-\hat\theta_k),
\end{equation*}
where $H_k \triangleq -\nabla^2_\theta\log\mathcal{L}(\theta;\mathcal{D}_k)|_{\theta=\hat\theta_k}$. Marginalizing over $\theta_k$ as in Proposition~\ref{prop:marginal_likelihood_thetat} gives
\begin{equation*}
    \log p(\mathcal{D}_k\mid c_k=1, \theta) \approx \mathrm{const} - \tfrac{1}{2}(\theta-\hat\theta_k)^\intercal(I+\tau^2 H_k)^{-1} H_k(\theta-\hat\theta_k).
\end{equation*}
This expansion makes the objective quadratic, which admits a closed-form update
\begin{equation}
\label{eq:m_step_update}
    \theta^{(t+1)} = \Big(I + \sum_{k=1}^K \Lambda_k^{(t)}\Big)^{-1}\Big(\hat\theta_0 + \sum_{k=1}^K \Lambda_k^{(t)}\hat\theta_k\Big), \quad \Lambda_k^{(t)} \triangleq w_k^{(t)} H_0^{-1}(I+\tau^2 H_k)^{-1} H_k.
\end{equation}
The update is a precision-weighted average of the MLEs and needs no gradient descent inner loop. Because $\hat\theta_k$ and $H_k$ depend only on $\mathcal{D}_k$, they are computed \emph{only once and reused} across iterations.

Second, we approximate $H_k \approx N_k\mathcal{I}$ and $H_0 \approx N_0\mathcal{I}$ using the expected per-sample Fisher information $\mathcal{I}$. When $\tau^2 N_k \lambda_{\max}(\mathcal{I}) \ll 1$, the factor $(I+\tau^2 H_k)^{-1}$ is close to $I$, and \eqref{eq:m_step_update} reduces to a clean weighted average of the target and source MLEs (Appendix~\ref{sec:appendix_practical_em} and Algorithm~\ref{alg:lip_em_approx}):
\begin{equation}
\label{eq:prox_update}
    \theta^{(t+1)} = \frac{N_0\hat\theta_0 + \sum_{k=1}^K w_k^{(t)} N_k\, \hat\theta_k}{N_0 + \sum_{k=1}^K w_k^{(t)} N_k}.
\end{equation}

For over-parameterized models such as neural networks, where weight space averaging is meaningless and Hessians are impractical, we set $\tau = 0$, so that $p(\mathcal{D}_k\mid c_k=1,\theta) = \mathcal{L}(\theta;\mathcal{D}_k)$ by Remark~\ref{rmk:zero_shift_limit}, and run gradient ascent on \eqref{eq:opt} (Algorithm~\ref{alg:lip_em_nn} in Appendix~\ref{sec:appendix_algorithms}, which lists all variants).

\paragraph{Bayesian Tempering in EM.} The E-step weight \eqref{eq:E-step} can be unstable in early iterations when $\theta^{(t)}$ is a poor estimate of $\theta_0$, because the log-likelihood ratio in \eqref{eq:E-step} can have high variance and overwhelm the effect of the LIP (see Appendix~\ref{sec:bayesian_tempering}). To mitigate this, we introduce a tempering schedule that suppresses the log-likelihood ratio in early iterations and gradually increases it as the EM iterations approach $\theta_0$. This is achieved by multiplying the log-likelihood ratio by a tempering parameter $\beta_k^{(t)} = (1 - e^{-\nu t})/\varepsilon_k$ with rate $\nu>0$, which grows with $t$ and shrinks with $\varepsilon_k$, the standard deviation of the first-order error that the noisy target estimate causes in the log-likelihood ratio of source $k$. As a result, this tempered EM relies more on the LIP when the estimate is noisy and delegates to data otherwise. The effect of tempering will be discussed in the theoretical analysis in Sec.~\ref{sec:theory}.

\section{Theoretical Results}
\label{sec:theory}

We give two complementary guarantees for LIP-EM. First, for Gaussian mean estimation with $\tau=0$, a well-calibrated LIP makes the mean squared error (MSE) of the tempered EM limit match an oracle cold-start MSE up to an exponentially small residual (Sec.~\ref{sec:fs_correct_lip}). Second, for any fixed LIP, the M-step estimate is consistent for $\theta_0$ as $N_0\to\infty$, and the E-step weight at $\theta_0$ converges to $c_k$ as $N_k\to\infty$ (Sec.~\ref{sec:asymptotic_consistency}). The first formalizes the role of the LIP in cold start, and the second is a safety net.

\subsection{Finite sample cold-start MSE under a correct LIP}
\label{sec:fs_correct_lip}

We consider Gaussian mean estimation. The target has $N_0$ i.i.d.\ samples from $\mathcal{N}(\theta_0, \sigma^2 I)$, and each source $k$ has $N_k = N$ i.i.d.\ samples from $p_k \triangleq \mathcal{N}(\theta_k, \sigma^2 I)$. In its E-step, the algorithm uses a fixed null density $q(x) \triangleq p(x \mid c_k = 0)$ through the null likelihood $p(\mathcal{D}_k\mid c_k=0)=\prod_i q(x_i)$. Let $R \triangleq \{k : c_k = 1\}$ and $\bar R \triangleq \{k : c_k = 0\}$ denote the relevant and irrelevant sets. Define the per-sample log-LR of the relevant model against the null, $\ell(x;\theta) \triangleq \log[p(x \mid \theta)/q(x)]$, and its expectation under $p_k$,
\begin{equation}
\label{eq:fs_rho}
    \rho_k(\theta^{(t)}) \triangleq \mathbb{E}_{x \sim p_k}\left[\ell(x;\theta^{(t)})\right] = \mathrm{KL}(p_k \,\|\, q) - \mathrm{KL}(p_k \,\|\, p(\cdot \mid \theta^{(t)})),
\end{equation}
where $\mathrm{KL}$ is the Kullback--Leibler divergence. The cold-start oracle identifies all relevant sources:
\begin{equation}
\label{eq:fs_oracle}
    \theta^\star \triangleq \frac{N_0\hat\theta_0 + N\sum_{k\in R}\hat\theta_k}{N_0 + N|R|}.
\end{equation}
Under our model in Fig.~\ref{fig:dag}, Proposition~\ref{prop:fs_one_step} in Appendix~\ref{sec:appendix_finite_sample} characterizes the oracle MSE as
\begin{equation}
\label{eq:fs_oracle_mse_body}
    \mathrm{MSE}_\star(\tau) \triangleq \mathbb{E}\|\theta^\star - \theta_0\|^2 = \frac{d\sigma^2}{N_0 + N|R|} + \frac{d\tau^2\, N^2|R|}{(N_0+N|R|)^2},
\end{equation}
which is much smaller than the MSE $d\sigma^2/N_0$ of the target data alone when $N|R| \gg N_0$ and $\tau^2\ll\sigma^2|R|/N_0$. The second term is the irreducible error from the spread $\tau$ of the relevant sources around $\theta_0$. It vanishes when $\tau=0$, the setting of Theorem~\ref{thm:fs_main}.

Theorem~\ref{thm:fs_main} below rests on five assumptions, deferred to Appendix~\ref{sec:appendix_finite_sample}. Assumption~\ref{ass:gen_model} sets $\tau=0$, so relevant sources have $\theta_k=\theta_0$, and fixes the parameters $\theta_k$ of the irrelevant sources (Corollary~\ref{cor:fs_random_config} lets them be random). Assumption~\ref{ass:regularity} imposes regularity conditions for the fixed null density $q$, including a sub-Gaussian per-sample log-LR and an $L$-Lipschitz expected log-LR $\rho_k(\theta)$. Assumption~\ref{ass:margin} posits an identifiability margin $\Delta^*>0$ at $\theta_0$: $\rho_k(\theta_0)\ge\Delta^*$ for relevant sources and $\rho_k(\theta_0)\le-\Delta^*$ for irrelevant ones. Assumption~\ref{ass:lip_calibration} requires the prior mass on irrelevant sources, weighted by their distance from $\theta_0$, to be small relative to the prior mass $P_R\triangleq\sum_{k\in R}\pi_k$ on relevant sources, so that the first EM update enters the basin $\mathcal{B}\triangleq\{\theta:\|\theta-\theta_0\|\le\Delta^*/(4L)\}$. Finally, Assumption~\ref{ass:sample_size} requires $N$ and $NP_R$ to be large enough.

\begin{theorem}[Cold-start MSE under a correct LIP]
\label{thm:fs_main}
Under Assumptions~\ref{ass:gen_model}--\ref{ass:sample_size}, the tempered EM iterates of Algorithms~\ref{alg:lip_em} and~\ref{alg:lip_em_approx}, which both reduce to \eqref{eq:prox_update} in this setting, converge with probability at least $1-O(Ke^{-\Omega(N)})$, and any limit point $\theta^{(\infty)}$ satisfies
\begin{equation}
\label{eq:fs_main}
    \mathbb{E}\bigl\|\theta^{(\infty)}-\theta_0\bigr\|^2
    \leq
    \underbrace{\frac{d\sigma^2}{N_0+N|R|}}_{\text{oracle rate}}
    + \underbrace{O\!\left(e^{-\Omega(\Delta^*\sqrt{NN_0/d})}\right)}_{\substack{\text{weight error}\\\text{residual}}}
    + \underbrace{O\!\left(e^{-\Omega(N)}\right)}_{\substack{\text{concentration}\\\text{failure}}}.
\end{equation}
\end{theorem}

Note that Theorem~\ref{thm:fs_main} does not require an exact LIP $\pi=c$. For sufficiently large source datasets, Assumption~\ref{ass:lip_calibration} only requires the probability mass put on irrelevant sources to be small enough that
\begin{equation}
\label{eq:fs_calib_main}
    \frac{\sum_{k\in\bar R}\pi_k\,\|\theta_k-\theta_0\|}{\sum_{k\in R}\pi_k} \;\leq\; \frac{\Delta^*}{8L}.
\end{equation}
The proof (Appendix~\ref{sec:appendix_finite_sample}) works on an event on which the log-LR and the sample means concentrate. Its complement has probability $O(Ke^{-\Omega(N)})$ and gives the third term of \eqref{eq:fs_main}. Under the Bayesian tempering schedule of Appendix~\ref{sec:bayesian_tempering}, the first E-step weights equal the prior, so on this event the calibration places the first update in the basin $\mathcal{B}$ (Lemma~\ref{lem:fs_basin_entry}). Inside $\mathcal{B}$, the identifiability margin drives each weight toward $c_k$ (Theorem~\ref{thm:fs_weight}). The iterates therefore stay in $\mathcal{B}$ (Lemma~\ref{lem:fs_invariance}) and converge to a fixed point that is unique in $\mathcal{B}$ (Lemma~\ref{lem:fs_contraction}). Lemma~\ref{lem:fs_decomp} splits the error of this fixed point into the error of the oracle \eqref{eq:fs_oracle}, which gives the leading term $\mathrm{MSE}_\star(0)$ of \eqref{eq:fs_oracle_mse_body}, and a remainder driven by the weight errors $w_k-c_k$. These decay exponentially in $\Delta^*\sqrt{NN_0/d}$ and give the second term. The hidden constants are given in Appendix~\ref{sec:appendix_finite_sample}.

\subsection{Asymptotic consistency under any LIP}
\label{sec:asymptotic_consistency}

If the LIP is misspecified so that \eqref{eq:fs_calib_main} fails, the first update need not enter $\mathcal{B}$ and Theorem~\ref{thm:fs_main} does not apply. Two asymptotic guarantees still hold for any fixed LIP.

\begin{theorem}[Consistency as \texorpdfstring{$N_0 \to \infty$}{N0 to infinity}]
\label{thm:asymptotic_consistency}
Fix any iterate $t$, source sizes $\{N_k\}$, and prior $\pi\in(0,1)^K$. Under standard regularity conditions (Appendix~\ref{sec:proof_asymptotic_consistency}), both M-step updates \eqref{eq:m_step_update} and \eqref{eq:prox_update} satisfy $\theta^{(t+1)} \xrightarrow{p} \theta_0$ as $N_0 \to \infty$.
\end{theorem}

\begin{theorem}[Asymptotic relevance at the target parameter]
\label{thm:asymptotic_relevance}
Let $\tau=0$ and fix the null density $q$. For a source $k$ with prior $\pi_k\in(0,1)$, let $w_k(\theta_0)$ be its E-step weight \eqref{eq:E-step} at $\theta^{(t)}=\theta_0$. If $\ell(\cdot\,;\theta_0)$ is integrable under $p_k$ and the identifiability margin of Assumption~\ref{ass:margin} holds, then $w_k(\theta_0)\xrightarrow{p}c_k$ as $N_k\to\infty$.
\end{theorem}

Theorem~\ref{thm:asymptotic_relevance} shows that, at the target parameter, the E-step with $\tau=0$ recovers the true membership of every source as $N_k\to\infty$, provided the model at $\theta_0$ fits each relevant source better than the null does and each irrelevant source worse. The same holds for the tempered weights \eqref{eq:tempering} at $t\ge1$ under standard regularity conditions (Appendix~\ref{sec:proof_asymptotic_relevance}).

\section{Empirical Results}
\label{sec:experiments}
We evaluate our method on three tasks: descriptive, predictive, and prescriptive.

\paragraph{Compared Methods.} LIP-C runs LIP-EM with the LIP elicited from Claude Opus 4.7 as a local Claude Code agent (Appendix~\ref{sec:appendix_agent}), and LIP-G with the LIP elicited from Gemini 3 Flash through its API. The Gaussian study uses LIP-G only and labels it LIP-EM. Uniform EM is the same tempered EM with $\pi_k = p_0$ for every source, i.e., without text. Target Only uses only the target data, and Pooled indiscriminately uses all the data from the target and all sources. For the Gaussians, both are the MLE of the mean and (co)variance. For C-MAPSS, they are ridge regressions of the physical core speed on a spline basis of the cycle, fit on the first $N_0$ target cycles alone (Target Only) or together with the $99$ full source trajectories (Pooled). For Hopper, they are implicit Q-learning (IQL) \citep{kostrikov2022offline} policies. Appendix~\ref{sec:appendix_hyperparameter} gives all hyperparameters.

\subsection{1-d and 2-d Gaussian Estimation}

Our first experiment estimates 1-d (Fig.~\ref{fig:1dgaus}) and 2-d (Fig.~\ref{fig:2dgaus}) Gaussian distributions.  The target samples are marked by black crosses. For both experiments, we pick three source domains (plotted in green, red, and blue), while only the green one matches the target.  To simulate the cold-start regime, we take $N_k=200$ samples from each source and $4$ from the target. As a proof of concept, we provide unambiguous descriptions to the LLM, so that any reasonable LLM can identify the relevant source. The descriptions are ``\textit{The target domain has a larger mean}'' in 1-d and ``\textit{The target domain has a larger x mean}'' in 2-d.

The results show the important role of the LIP. LIP-EM (purple curve) effectively identifies the relevant domain with the help of the language prior compared to Uniform EM (orange curve), which here gives every source a weight near zero and so coincides with Target Only.
As a result, the LIP-EM fit almost recovers the ground truth (black dashed curve). In comparison, Target Only (gray curve), which relies solely on the scarce target data, suffers from high variance and deviates far from the ground truth. Meanwhile, indiscriminately pooling all the data (Pooled, brown curve) suffers from negative transfer and is biased toward the irrelevant sources. 
\begin{figure}[hbt]
    \centering
    \begin{subfigure}[t]{0.49\linewidth}
        \centering
        \includegraphics[width=\linewidth]{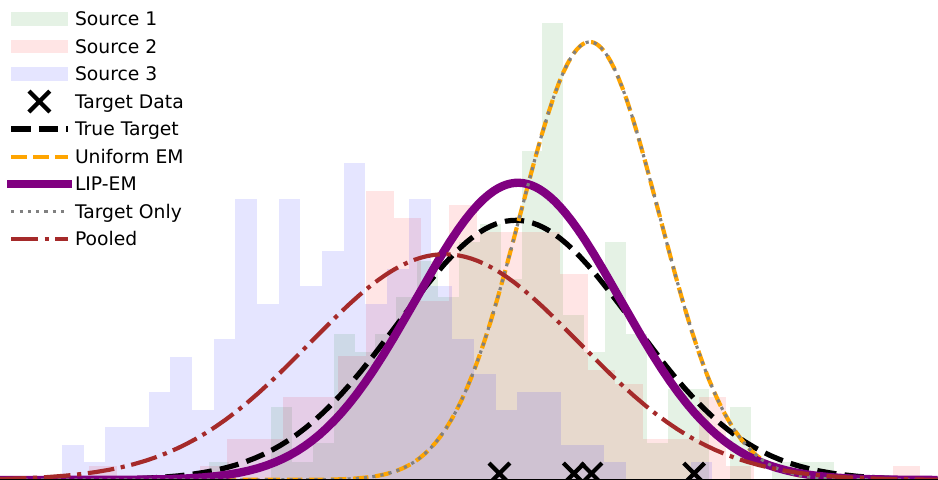}
        \caption{1-d Gaussian estimation.}
        \label{fig:1dgaus}
    \end{subfigure}
    \hfill
    \begin{subfigure}[t]{0.49\linewidth}
        \centering
        \includegraphics[width=\linewidth]{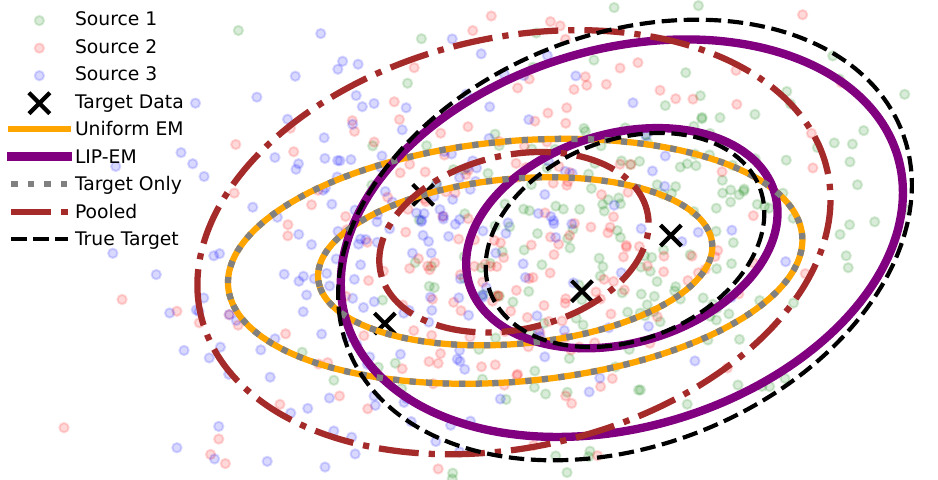}
        \caption{2-d Gaussian estimation.}
        \label{fig:2dgaus}
    \end{subfigure}
    \caption{Gaussian estimation.  LIP-EM (in purple) best fits the target (in black dashed curve).}
    \label{fig:gaus}
\end{figure}
\subsection{Case Study on C-MAPSS Dataset}

To evaluate our method on a physical system, we apply the LIP-EM framework to the C-MAPSS dataset FD001 \citep{saxena2008damage}. It contains $100$ simulated run-to-failure trajectories of turbofan engines. Each trajectory ends when the engine's Remaining Useful Life (RUL) reaches zero. The task is prognostics \citep{liu2013datalevel}, i.e., extrapolation. The response is the physical core speed (sensor 9) of the High-Pressure Compressor (HPC). Given the first $N_0$ cycles of a target engine whose physical core speed grows fast, we predict its future physical core speed. Each engine is one domain, since the goal is a personalized prediction for one engine. The input (the cycle) is shared, but the response at a given cycle differs across engines. This is a concept shift. The other $99$ engines, observed until failure, are the sources. Which of them degrade like the target is not known a priori, and LIP-EM infers it from the LIP and the data. The LLM receives a simulated expert description that hints at fast growth of the physical core speed: \textit{This aircraft is operating in a desert environment. Routine visual inspections confirm severe aerodynamic degradation of the high-pressure compressor, likely due to continuously ingesting fine, dry abrasives at high velocities.} The dataset paper \citep{saxena2008damage} is given to the LLM alongside the description for physical context. Appendix~\ref{sec:appendix_prompt} varies the wording and content of this description, and Appendix~\ref{sec:appendix_ds03} evaluates direct RUL prediction on N-CMAPSS DS03 \citep{ariaschao2021aircraft} with its native flight class label as a real target description.

The $10$ target engines, one per replication, are shown in red in Fig.~\ref{fig:selected_sources}. For each target, we observe only the first part of the trajectory and report the Root Mean Square Error (RMSE) of the predicted future physical core speed. Cutoffs are labeled by the observed share of the target life, e.g., at $30\%$ observed life, the remaining $70\%$ of the trajectory is predicted. A Generalized Linear Model (GLM), namely a natural cubic spline, is used as the backbone statistical model for all compared methods \citep{yue2024federated}. EM uses Algorithm~\ref{alg:lip_em} (Appendix~\ref{sec:appendix_algorithms}). 

The experimental results are summarized in Table~\ref{tab:cmapss} (all nine cutoffs and significance tests in Table~\ref{tab:cmapss_full}). Again, LIP significantly reduces the prediction error during the cold-start phase ($10\%$, $30\%$ and $50\%$ observed life). As an example, we plot the predictions of engine 80 at $30\%$ observed life in Fig.~\ref{fig:engine80}. Early in life, its degradation trend is statistically hidden by the noise. Target Only (gray) mistakes the engine for a healthy one and predicts a stable physical core speed. Pooled (brown), which weights all sources equally, and Uniform EM (orange), which gives all sources the same prior, also miss the acceleration. In contrast, LIP-G (purple), whose prior favors sources with a rising physical core speed, captures the acceleration. This indicates the importance of leveraging language as a modality in DA. 

\begin{figure}[hbt]
    \centering
    \begin{minipage}[t]{0.49\linewidth}
        \centering
        \includegraphics[width=\linewidth]{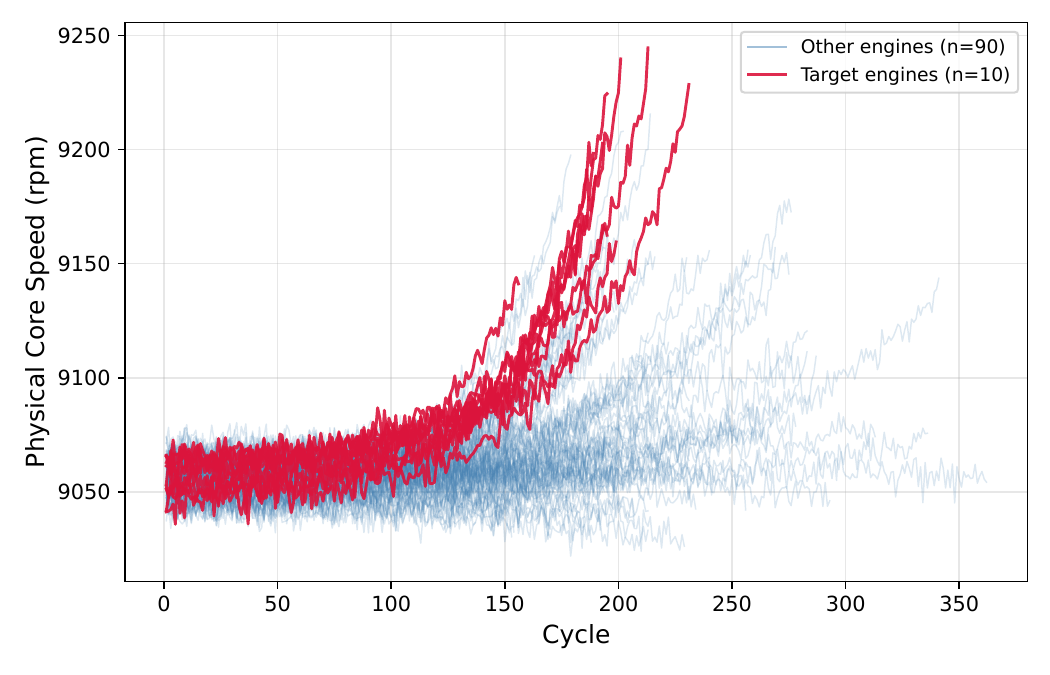}
        \captionof{figure}{Physical core speed of the $100$ FD001 engines, with the $10$ targets in red.}
        \label{fig:selected_sources}
    \end{minipage}
    \hfill
    \begin{minipage}[t]{0.49\linewidth}
        \centering
        \includegraphics[width=\linewidth]{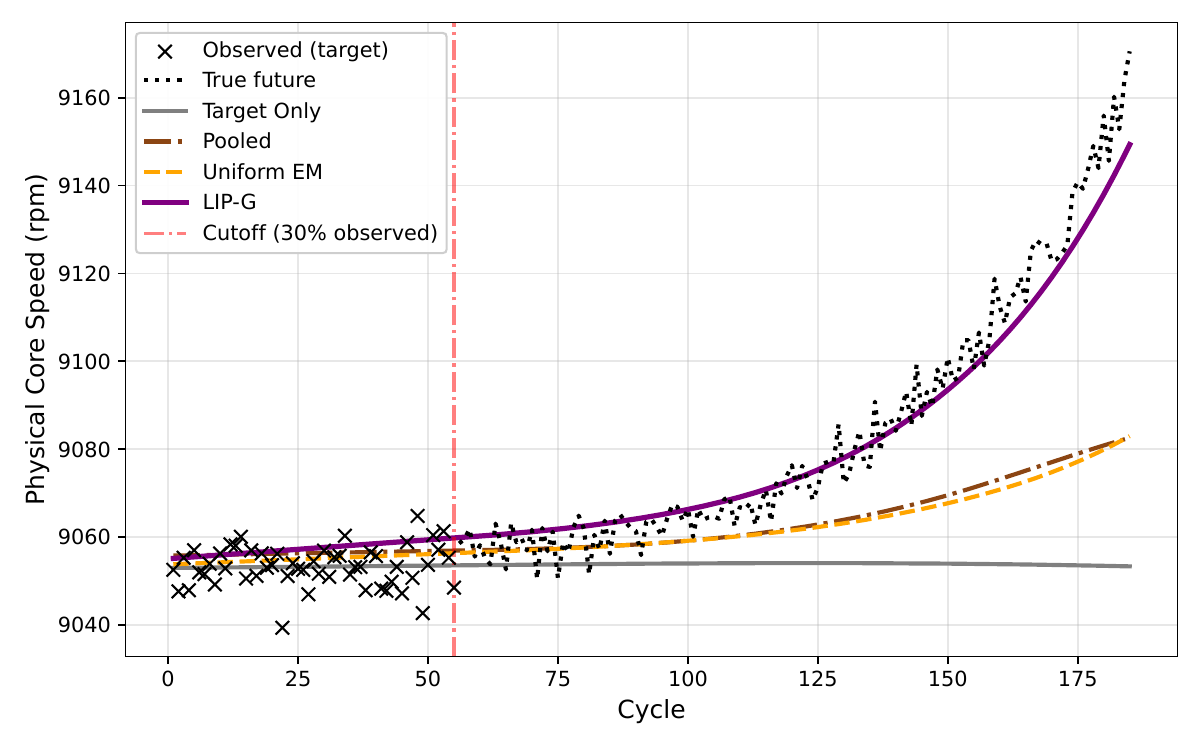}
        \captionof{figure}{Physical core speed predictions for engine 80 at 30\% observed life.}
        \label{fig:engine80}
    \end{minipage}
\end{figure}

\begin{table}[htb]
  \begin{minipage}[b]{0.49\textwidth}
    \centering
    \scriptsize
    \setlength{\tabcolsep}{2.5pt}
    \captionof{table}{C-MAPSS physical core speed RMSE: mean (standard error, SE) over the $10$ targets.}
    \begin{tabular}{lccccc}
    \toprule
    Observed & & & Uniform & & Target\\
    life & LIP-C & LIP-G & EM & Pooled & Only\\
    \midrule
    10\% & 14.3\,(2.1) & 15.0\,(3.5) & 21.4\,(3.2) & 33.9\,(2.8) & 43.7\,(2.9) \\
    30\% & 15.8\,(2.1) & 15.2\,(3.8) & 35.0\,(3.2) & 38.1\,(3.1) & 58.7\,(11.4) \\
    50\% & 17.7\,(3.7) & 17.9\,(6.5) & 24.8\,(4.6) & 44.7\,(3.7) & 36.6\,(5.3) \\
    70\% & 20.1\,(3.4) & 16.4\,(3.4) & 15.2\,(2.3) & 56.5\,(4.6) & 27.0\,(6.1) \\
    90\% & 11.5\,(1.7) & 9.0\,(1.4) & 7.3\,(0.8) & 82.9\,(6.9) & 10.1\,(2.8) \\
    \bottomrule
    \end{tabular}
    \label{tab:cmapss}
  \end{minipage}
  \hfill
  \begin{minipage}[b]{0.49\textwidth}
    \centering
    \scriptsize
    \setlength{\tabcolsep}{2.5pt}
    \caption{Hopper return: mean (SE) over $10$ seeds.}
    \begin{tabular}{rccccc}
    \toprule
    $N_0$ & LIP-C & LIP-G & Uniform EM & Pooled & Target Only\\
    \midrule
    64   & 2686\,(59) & 1702\,(89) & 1190\,(46) & 1209\,(27) & 64\,(14) \\
    128  & 2677\,(40) & 1588\,(53) & 1393\,(93) & 1211\,(37) & 60\,(20) \\
    256  & 2584\,(34) & 1754\,(64) & 1723\,(214) & 1214\,(37) & 105\,(25) \\
    512  & 2650\,(43) & 1767\,(64) & 2646\,(35) & 1256\,(43) & 107\,(20) \\
    1024 & 2653\,(39) & 2687\,(41) & 2692\,(42) & 1221\,(46) & 143\,(24) \\
    2048 & 2623\,(41) & 2602\,(47) & 2644\,(51) & 1217\,(27) & 219\,(32) \\
    4096 & 2694\,(29) & 2696\,(25) & 2693\,(25) & 1146\,(51) & 342\,(33) \\
    \bottomrule
    \end{tabular}
    \label{tab:hopper}
  \end{minipage}
\end{table}

\subsection{Case Study on MuJoCo Hopper Dynamics}
\label{sec:hopper}

To validate our framework in deep neural networks, we apply the LIP-EM algorithm to a deep reinforcement learning task in the MuJoCo Hopper environment \citep{todorov2012mujoco}. MuJoCo allows users to change the simulation environment of a hopper (a three joint jumping robot). The agent in this task is deployed to a new ``planet'' (the target environment). It has only a small dataset $\mathcal{D}_0$ collected in the target environment and much larger source datasets $\{\mathcal{D}_k\}_{k=1}^{10}$ collected under $10$ different gravities $g\in \{1,2,\dots, 10\} ~\mathrm{m/s^2}$. Each source dataset is a \emph{replay} dataset of one million transitions, collected by running soft actor-critic (SAC) \citep{haarnoja2018soft} from scratch and saving the entire training replay buffer. The small target dataset $\mathcal{D}_0$ is also of poor quality (most early SAC episodes end with the hopper falling), which makes DA more critical. The target environment has the gravity of Venus ($g = 8.87~\mathrm{m/s^2}$). We give the target description: \textit{The hopper is deployed to a planet with Venus-like gravity}. We then follow Sec.~\ref{sec:lip} and ask the LLM to select, from each sampled subgroup, the source whose physics is most consistent with the description.

Since IQL, a model-free offline reinforcement learning (RL) algorithm, provides no likelihood $\mathcal{L}$, we follow model-based RL \citep{yu2020mopo} and model the dynamics as a Gaussian whose mean and variance are given by a multi-layer perceptron fitted by MLE. We run EM with this dynamics model and then train IQL policies on the $w_k$-weighted source data and the unweighted target data. EM uses Algorithm~\ref{alg:lip_em_nn} in Appendix~\ref{sec:appendix_algorithms}, which sets $\tau=0$ (Remark~\ref{rmk:zero_shift_limit}) and needs no Hessian. Each EM run completes within a few minutes on a single GPU (Appendix~\ref{sec:appendix_hardware}).

Again, LIP-C achieves a substantially better return in the cold-start phase (Table~\ref{tab:hopper}, $N_0\le 256$). Interestingly, LIP-G uses a false LIP caused by reasoning errors of Gemini 3 Flash (see failure modes in Appendix~\ref{sec:appendix_failure_modes}). This false LIP prefers $g = 7~\mathrm{m/s^2}$ over the closer matches $g \in \{8, 9, 10\}~\mathrm{m/s^2}$ to Venus's gravity. This false prior beats the uniform prior at $N_0=64$ but slows the recovery: at $N_0=512$, Uniform EM already selects $g=9~\mathrm{m/s^2}$ while LIP-G still selects $g=7~\mathrm{m/s^2}$. EM overrides the false LIP from $N_0=1024$. This emphasizes the importance of prior quality, but also showcases the auto-correcting mechanism of EM, echoing our theoretical analysis.

\section{Discussion}
\label{sec:discussion}
We have shown that textual descriptions of a target domain carry information that statistical methods cannot recover from scarce target data. A pretrained LLM can translate this information into a Bayesian prior (LIP) over source relevance. Theoretically, a correct LIP yields a finite sample MSE that matches the oracle rate, while any LIP, even a misspecified one, still yields a consistent estimate as $N_0\to\infty$. Our experiments agree with these claims, with gains concentrated in the cold start. Nonetheless, the quality of the LIP is crucial for cold-start performance (see Appendix~\ref{sec:appendix_compare_LLM}). Future work includes exploring more sophisticated elicitation methods and extending the framework to non-parametric models and to models without explicit likelihood functions.

\begin{ack}
This work was supported by the National Science Foundation under CMMI Grant Number 2328010. The authors declare no competing interests.
\end{ack}

\bibliographystyle{plainnat-lipem}
\bibliography{verified-references}

\begin{thebibliography}{38}
\providecommand{\natexlab}[1]{#1}
\providecommand{\url}[1]{\texttt{#1}}
\expandafter\ifx\csname urlstyle\endcsname\relax
  \providecommand{\doi}[1]{doi: #1}\else
  \providecommand{\doi}{doi: \begingroup \urlstyle{rm}\Url}\fi

\bibitem[Arias~Chao et~al.(2021)Arias~Chao, Kulkarni, Goebel, and Fink]{ariaschao2021aircraft}
Manuel Arias~Chao, Chetan Kulkarni, Kai Goebel, and Olga Fink.
\newblock Aircraft Engine Run-to-Failure Dataset under Real Flight Conditions for Prognostics and Diagnostics.
\newblock \emph{Data}, 6\penalty0 (1):\penalty0 5, 2021.

\bibitem[Balakrishnan et~al.(2017)Balakrishnan, Wainwright, and Yu]{balakrishnan2017statistical}
Sivaraman Balakrishnan, Martin~J. Wainwright, and Bin Yu.
\newblock {Statistical guarantees for the EM algorithm: From population to sample-based analysis}.
\newblock \emph{The Annals of Statistics}, 45\penalty0 (1):\penalty0 77--120, 2017.

\bibitem[Boucheron et~al.(2013)Boucheron, Lugosi, and Massart]{boucheron2013concentration}
St{\'e}phane Boucheron, G{\'a}bor Lugosi, and Pascal Massart.
\newblock \emph{Concentration Inequalities: A Nonasymptotic Theory of Independence}.
\newblock Oxford University Press, 2013.

\bibitem[Brown et~al.(2020)Brown, Mann, Ryder, Subbiah, Kaplan, Dhariwal, Neelakantan, Shyam, Sastry, Askell, Agarwal, Herbert-Voss, Krueger, Henighan, Child, Ramesh, Ziegler, Wu, Winter, Hesse, Chen, Sigler, Litwin, Gray, Chess, Clark, Berner, McCandlish, Radford, Sutskever, and Amodei]{brown2020language}
Tom Brown, Benjamin Mann, Nick Ryder, Melanie Subbiah, Jared~D Kaplan, Prafulla Dhariwal, Arvind Neelakantan, Pranav Shyam, Girish Sastry, Amanda Askell, Sandhini Agarwal, Ariel Herbert-Voss, Gretchen Krueger, Tom Henighan, Rewon Child, Aditya Ramesh, Daniel Ziegler, Jeffrey Wu, Clemens Winter, Chris Hesse, Mark Chen, Eric Sigler, Mateusz Litwin, Scott Gray, Benjamin Chess, Jack Clark, Christopher Berner, Sam McCandlish, Alec Radford, Ilya Sutskever, and Dario Amodei.
\newblock Language Models are Few-Shot Learners.
\newblock In H.~Larochelle, M.~Ranzato, R.~Hadsell, M.F. Balcan, and H.~Lin, editors, \emph{Advances in Neural Information Processing Systems}, volume~33, pages 1877--1901. Curran Associates, Inc., 2020.

\bibitem[Capstick et~al.(2025)Capstick, Krishnan, and Barnaghi]{capstick2025autoelicit}
Alexander Capstick, Rahul Krishnan, and Payam Barnaghi.
\newblock {A}uto{E}licit: Using Large Language Models for Expert Prior Elicitation in Predictive Modelling.
\newblock In Aarti Singh, Maryam Fazel, Daniel Hsu, Simon Lacoste-Julien, Felix Berkenkamp, Tegan Maharaj, Kiri Wagstaff, and Jerry Zhu, editors, \emph{Proceedings of the 42nd International Conference on Machine Learning}, volume 267 of \emph{Proceedings of Machine Learning Research}, pages 6746--6777. PMLR, 13--19 Jul 2025.

\bibitem[Garthwaite et~al.(2005)Garthwaite, Kadane, and O'Hagan]{garthwaite2005statistical}
Paul~H Garthwaite, Joseph~B Kadane, and Anthony O'Hagan.
\newblock Statistical Methods for Eliciting Probability Distributions.
\newblock \emph{Journal of the American Statistical Association}, 100\penalty0 (470):\penalty0 680--701, 2005.

\bibitem[Gu et~al.(2024)Gu, Desai, Lin, and Yang]{gu2024probabilistic}
Bowen Gu, Rishi~J Desai, Kueiyu~Joshua Lin, and Jie Yang.
\newblock Probabilistic medical predictions of large language models.
\newblock \emph{npj Digital Medicine}, 7\penalty0 (1):\penalty0 367, 2024.

\bibitem[Haarnoja et~al.(2018)Haarnoja, Zhou, Abbeel, and Levine]{haarnoja2018soft}
Tuomas Haarnoja, Aurick Zhou, Pieter Abbeel, and Sergey Levine.
\newblock Soft Actor-Critic: Off-Policy Maximum Entropy Deep Reinforcement Learning with a Stochastic Actor.
\newblock In Jennifer Dy and Andreas Krause, editors, \emph{Proceedings of the 35th International Conference on Machine Learning}, volume~80 of \emph{Proceedings of Machine Learning Research}, pages 1861--1870. PMLR, 10--15 Jul 2018.

\bibitem[Hoffman et~al.(2012)Hoffman, Kulis, Darrell, and Saenko]{hoffman2012discovering}
Judy Hoffman, Brian Kulis, Trevor Darrell, and Kate Saenko.
\newblock Discovering Latent Domains for Multisource Domain Adaptation.
\newblock In Andrew Fitzgibbon, Svetlana Lazebnik, Pietro Perona, Yoichi Sato, and Cordelia Schmid, editors, \emph{Computer Vision -- ECCV 2012}, pages 702--715, Berlin, Heidelberg, 2012. Springer Berlin Heidelberg.

\bibitem[Kojima et~al.(2022)Kojima, Gu, Reid, Matsuo, and Iwasawa]{kojima2022large}
Takeshi Kojima, Shixiang~(Shane) Gu, Machel Reid, Yutaka Matsuo, and Yusuke Iwasawa.
\newblock Large Language Models are Zero-Shot Reasoners.
\newblock In S.~Koyejo, S.~Mohamed, A.~Agarwal, D.~Belgrave, K.~Cho, and A.~Oh, editors, \emph{Advances in Neural Information Processing Systems}, volume~35, pages 22199--22213. Curran Associates, Inc., 2022.

\bibitem[Kontar et~al.(2021)Kontar, Raskutti, and Zhou]{kontar2020minimizing}
Raed Kontar, Garvesh Raskutti, and Shiyu Zhou.
\newblock Minimizing Negative Transfer of Knowledge in Multivariate Gaussian Processes: A Scalable and Regularized Approach.
\newblock \emph{IEEE Transactions on Pattern Analysis and Machine Intelligence}, 43\penalty0 (10):\penalty0 3508--3522, 2021.

\bibitem[Kostrikov et~al.(2022)Kostrikov, Nair, and Levine]{kostrikov2022offline}
Ilya Kostrikov, Ashvin Nair, and Sergey Levine.
\newblock Offline Reinforcement Learning with Implicit Q-Learning.
\newblock In \emph{International Conference on Learning Representations}, 2022.

\bibitem[Liu et~al.(2013)Liu, Gebraeel, and Shi]{liu2013datalevel}
Kaibo Liu, Nagi~Z. Gebraeel, and Jianjun Shi.
\newblock A Data-Level Fusion Model for Developing Composite Health Indices for Degradation Modeling and Prognostic Analysis.
\newblock \emph{IEEE Transactions on Automation Science and Engineering}, 10\penalty0 (3):\penalty0 652--664, 2013.

\bibitem[Luce(1959)]{luce1959individual}
R.~Duncan Luce.
\newblock \emph{Individual Choice Behavior: A Theoretical Analysis}.
\newblock John Wiley \& Sons, New York, 1959.

\bibitem[Mansour et~al.(2008)Mansour, Mohri, and Rostamizadeh]{mansour2008domain}
Yishay Mansour, Mehryar Mohri, and Afshin Rostamizadeh.
\newblock Domain Adaptation with Multiple Sources.
\newblock In D.~Koller, D.~Schuurmans, Y.~Bengio, and L.~Bottou, editors, \emph{Advances in Neural Information Processing Systems}, volume~21. Curran Associates, Inc., 2008.

\bibitem[Mata et~al.(2025)Mata, Ranasinghe, and Ryoo]{mata2024copt}
Cristina Mata, Kanchana Ranasinghe, and Michael~S. Ryoo.
\newblock CoPT: Unsupervised Domain Adaptive Segmentation Using Domain-Agnostic Text Embeddings.
\newblock In Ale{\v{s}} Leonardis, Elisa Ricci, Stefan Roth, Olga Russakovsky, Torsten Sattler, and G{\"u}l Varol, editors, \emph{Computer Vision -- ECCV 2024}, pages 424--440, Cham, 2025. Springer Nature Switzerland.

\bibitem[McFadden(1974)]{mcfadden1974conditional}
Daniel McFadden.
\newblock Conditional logit analysis of qualitative choice behavior.
\newblock In \emph{Frontiers in Econometrics}, pages 105--142. Academic Press, New York, 1974.

\bibitem[O'Hagan et~al.(2006)O'Hagan, Buck, Daneshkhah, Eiser, Garthwaite, Jenkinson, Oakley, and Rakow]{ohagan2006uncertain}
Anthony O'Hagan, Caitlin~E Buck, Alireza Daneshkhah, J~Richard Eiser, Paul~H Garthwaite, David~J Jenkinson, Jeremy~E Oakley, and Tim Rakow.
\newblock \emph{Uncertain Judgements: Eliciting Experts' Probabilities}.
\newblock John Wiley \& Sons, 2006.

\bibitem[Pan and Yang(2010)]{pan2009survey}
Sinno~Jialin Pan and Qiang Yang.
\newblock A Survey on Transfer Learning.
\newblock \emph{IEEE Transactions on Knowledge and Data Engineering}, 22\penalty0 (10):\penalty0 1345--1359, 2010.

\bibitem[Peng et~al.(2019)Peng, Bai, Xia, Huang, Saenko, and Wang]{peng2019moment}
Xingchao Peng, Qinxun Bai, Xide Xia, Zijun Huang, Kate Saenko, and Bo~Wang.
\newblock Moment Matching for Multi-Source Domain Adaptation.
\newblock In \emph{2019 IEEE/CVF International Conference on Computer Vision (ICCV)}, pages 1406--1415, 2019.

\bibitem[Petroni et~al.(2019)Petroni, Rockt{\"a}schel, Riedel, Lewis, Bakhtin, Wu, and Miller]{petroni2019language}
Fabio Petroni, Tim Rockt{\"a}schel, Sebastian Riedel, Patrick Lewis, Anton Bakhtin, Yuxiang Wu, and Alexander Miller.
\newblock Language Models as Knowledge Bases?
\newblock In Kentaro Inui, Jing Jiang, Vincent Ng, and Xiaojun Wan, editors, \emph{Proceedings of the 2019 Conference on Empirical Methods in Natural Language Processing and the 9th International Joint Conference on Natural Language Processing ({EMNLP}-{IJCNLP})}, pages 2463--2473, Hong Kong, China, November 2019. Association for Computational Linguistics.

\bibitem[Radford et~al.(2021)Radford, Kim, Hallacy, Ramesh, Goh, Agarwal, Sastry, Askell, Mishkin, Clark, Krueger, and Sutskever]{radford2021learning}
Alec Radford, Jong~Wook Kim, Chris Hallacy, Aditya Ramesh, Gabriel Goh, Sandhini Agarwal, Girish Sastry, Amanda Askell, Pamela Mishkin, Jack Clark, Gretchen Krueger, and Ilya Sutskever.
\newblock Learning Transferable Visual Models From Natural Language Supervision.
\newblock In Marina Meila and Tong Zhang, editors, \emph{Proceedings of the 38th International Conference on Machine Learning}, volume 139 of \emph{Proceedings of Machine Learning Research}, pages 8748--8763. PMLR, 18--24 Jul 2021.

\bibitem[Saxena et~al.(2008)Saxena, Goebel, Simon, and Eklund]{saxena2008damage}
Abhinav Saxena, Kai Goebel, Don Simon, and Neil Eklund.
\newblock Damage propagation modeling for aircraft engine run-to-failure simulation.
\newblock In \emph{2008 International Conference on Prognostics and Health Management}, pages 1--9, 2008.

\bibitem[Sekulovski et~al.(2026)Sekulovski, Waaijers, and Arena]{sekulovski2026llm}
Nikola Sekulovski, Meike Waaijers, and Giuseppe Arena.
\newblock LLM-based prior elicitation for Bayesian graphical modeling.
\newblock \emph{British Journal of Mathematical and Statistical Psychology}, 2026.

\bibitem[Sun et~al.(2011)Sun, Chattopadhyay, Panchanathan, and Ye]{sun2011two}
Qian Sun, Rita Chattopadhyay, Sethuraman Panchanathan, and Jieping Ye.
\newblock A Two-Stage Weighting Framework for Multi-Source Domain Adaptation.
\newblock In J.~Shawe-Taylor, R.~Zemel, P.~Bartlett, F.~Pereira, and K.~Weinberger, editors, \emph{Advances in Neural Information Processing Systems}, volume~24. Curran Associates, Inc., 2011.

\bibitem[Todorov et~al.(2012)Todorov, Erez, and Tassa]{todorov2012mujoco}
Emanuel Todorov, Tom Erez, and Yuval Tassa.
\newblock MuJoCo: A physics engine for model-based control.
\newblock In \emph{2012 IEEE/RSJ International Conference on Intelligent Robots and Systems}, pages 5026--5033, 2012.

\bibitem[Ueda and Nakano(1998)]{ueda1998deterministic}
Naonori Ueda and Ryohei Nakano.
\newblock Deterministic annealing EM algorithm.
\newblock \emph{Neural Networks}, 11\penalty0 (2):\penalty0 271--282, 1998.

\bibitem[Wang et~al.(2025)Wang, Zhang, Wang, and Zhu]{wang2025landa}
Zhenbin Wang, Lei Zhang, Lituan Wang, and Minjuan Zhu.
\newblock LanDA: Language-Guided Multi-Source Domain Adaptation.
\newblock \emph{IEEE Transactions on Artificial Intelligence}, pages 1--12, 2025.

\bibitem[Wang et~al.(2019)Wang, Dai, Póczos, and Carbonell]{wang2019characterizing}
Zirui Wang, Zihang Dai, Barnabás Póczos, and Jaime Carbonell.
\newblock Characterizing and Avoiding Negative Transfer.
\newblock In \emph{2019 IEEE/CVF Conference on Computer Vision and Pattern Recognition (CVPR)}, pages 11285--11294, 2019.

\bibitem[Wen et~al.(2020)Wen, Greiner, and Schuurmans]{wen2020domain}
Junfeng Wen, Russell Greiner, and Dale Schuurmans.
\newblock Domain Aggregation Networks for Multi-Source Domain Adaptation.
\newblock In Hal~Daumé III and Aarti Singh, editors, \emph{Proceedings of the 37th International Conference on Machine Learning}, volume 119 of \emph{Proceedings of Machine Learning Research}, pages 10214--10224. PMLR, 13--18 Jul 2020.

\bibitem[Xiong et~al.(2014)Xiong, McCloskey, Hsieh, and Corso]{xiong2014latent}
Caiming Xiong, Scott McCloskey, Shao-Hang Hsieh, and Jason Corso.
\newblock Latent Domains Modeling for Visual Domain Adaptation.
\newblock \emph{Proceedings of the AAAI Conference on Artificial Intelligence}, 28\penalty0 (1), 2014.

\bibitem[Yao et~al.(2024)Yao, Yang, Pan, Liu, Koh, and Finn]{yao2024improving}
Huaxiu Yao, Xinyu Yang, Xinyi Pan, Shengchao Liu, Pang~Wei Koh, and Chelsea Finn.
\newblock Improving Domain Generalization with Domain Relations.
\newblock In B.~Kim, Y.~Yue, S.~Chaudhuri, K.~Fragkiadaki, M.~Khan, and Y.~Sun, editors, \emph{International Conference on Learning Representations}, volume 2024, pages 42624--42646, 2024.

\bibitem[Yu et~al.(2020)Yu, Thomas, Yu, Ermon, Zou, Levine, Finn, and Ma]{yu2020mopo}
Tianhe Yu, Garrett Thomas, Lantao Yu, Stefano Ermon, James Zou, Sergey Levine, Chelsea Finn, and Tengyu Ma.
\newblock MOPO: Model-based Offline Policy Optimization.
\newblock In H.~Larochelle, M.~Ranzato, R.~Hadsell, M.F. Balcan, and H.~Lin, editors, \emph{Advances in Neural Information Processing Systems}, volume~33, pages 14129--14142. Curran Associates, Inc., 2020.

\bibitem[Yue et~al.(2024)Yue, Kontar, and Gómez]{yue2024federated}
Xubo Yue, Raed~Al Kontar, and Ana María~Estrada Gómez.
\newblock Federated data analytics: A study on linear models.
\newblock \emph{IISE Transactions}, 56\penalty0 (1):\penalty0 16--28, 2024.

\bibitem[Zhang et~al.(2023)Zhang, Deng, Zhang, and Wu]{zhang2022survey}
Wen Zhang, Lingfei Deng, Lei Zhang, and Dongrui Wu.
\newblock A Survey on Negative Transfer.
\newblock \emph{IEEE/CAA Journal of Automatica Sinica}, 10\penalty0 (2):\penalty0 305--329, 2023.

\bibitem[Zhao et~al.(2018)Zhao, Zhang, Wu, Moura, Costeira, and Gordon]{zhao2018adversarial}
Han Zhao, Shanghang Zhang, Guanhang Wu, Jos\'{e} M.~F. Moura, Joao~P Costeira, and Geoffrey~J Gordon.
\newblock Adversarial Multiple Source Domain Adaptation.
\newblock In S.~Bengio, H.~Wallach, H.~Larochelle, K.~Grauman, N.~Cesa-Bianchi, and R.~Garnett, editors, \emph{Advances in Neural Information Processing Systems}, volume~31. Curran Associates, Inc., 2018.

\bibitem[Zhao et~al.(2024)Zhao, Chen, Huang, Xu, and Ding]{zhao2024more}
Sicheng Zhao, Hui Chen, Hu~Huang, Pengfei Xu, and Guiguang Ding.
\newblock More is better: deep domain adaptation with multiple sources.
\newblock In \emph{Proceedings of the Thirty-Third International Joint Conference on Artificial Intelligence}, IJCAI '24, 2024.

\bibitem[Zhu et~al.(2026)Zhu, Deng, Shi, Gangrade, and Scott]{zhu2026domain}
Yilun Zhu, Naihao Deng, Naichen Shi, Aditya Gangrade, and Clayton Scott.
\newblock Domain Generalization Under Posterior Drift.
\newblock \emph{arXiv preprint arXiv:2510.04441}, 2026.

\end{thebibliography}

\appendix

\section{Extended Related Work}
\label{sec:appendix_related_work}

\paragraph{Domain Generalization with Metadata.}
Closely related to our motivation of using contextual information is Domain Generalization (DG) with metadata. Unlike standard DG, which seeks a purely data-driven domain-invariant representation, metadata-aware DG leverages supplementary domain attributes to capture inter-domain relationships. \citet{zhu2026domain} demonstrate the necessity of metadata when the physics of the domain changes, showing that empirical risk minimization on pooled data is provably suboptimal under posterior drift, and incorporate domain metadata as an additional input during training and inference. Other methods such as $\mathrm{D^3G}$ \citep{yao2024improving} establish domain similarity matrices from structured metadata to reweight domain-specific models. These works share a structural assumption that we relax: metadata must be consistently available across \emph{all} training domains. In engineering practice, such records exist only for the present target system, while historical source datasets predate the metadata convention. Our framework operates on unstructured natural language descriptions available \emph{only} for the target and delegates the relevance reasoning to the LLM rather than to a similarity computation over fixed metadata fields.

\paragraph{Language-Guided Domain Adaptation.}
The emergence of Vision-Language Models (VLMs) like CLIP \citep{radford2021learning} has enabled adaptation methods to use natural language as a bridge between domains. Recent works such as LanDA \citep{wang2025landa} and CoPT \citep{mata2024copt} demonstrate the efficacy of using text descriptions to guide visual alignment by treating text as a feature embedding aligned with image features in a shared space. Two structural commitments distinguish these methods from ours. First, like domain generalization with metadata, they also require text descriptions for source domains as well as for the target, which is impractical when the source datasets predate the natural language pipeline. Second, this field is built upon massive VLM architectures, making the methods task-specific and restricted to modalities that the VLM was pretrained on. We instead use the LLM as a \emph{reasoning agent} producing discrete preferences, require text only for the target, and remain model-agnostic, so our method applies to general parameter estimation tasks beyond vision-language alignment.

\section{LIP Construction via Subgroup Preference Elicitation}
\label{sec:appendix_lip}
\subsection{Subgroup Preference Elicitation}
Directly prompting an LLM for numerical probabilities $\pi_k$ can sometimes be unreliable due to the context window constraint, and it is not an easy task for human experts either. To keep the cognitive task as simple and reliable as possible, we constrain the LLM to select \emph{at most one} candidate per query: which source domain, among a set of candidates, is most likely to satisfy the contextual information of the target domain? Effectively, we are asking a prompted LLM to act as a judge in selecting the relevant source domains. It either returns the single most relevant source from the presented subgroup, or it returns nothing if no candidate satisfies the contextual description. To perform this selection, the LLM is provided with contextual knowledge $h$, which includes the text that describes the target domain and background information about the task, such as relevant technical reports.

For each query $m = 1, \dots, N_M$, we sample a subgroup $S_m \subseteq \{1, \ldots, K\}$ of source domain indices and present the corresponding datasets $\{\mathcal{D}_k\}_{k \in S_m}$ to the LLM. The LLM sees the source data themselves (an excerpt when a dataset is too large for its context window). No source descriptions are provided, and $S_m$ only holds the indices of the presented sources. The LLM returns a \emph{choice} $k_m \in S_m \cup \{0\}$, where $k_m \in S_m$ identifies the selected source and $k_m = 0$ denotes the null option (no candidate is considered relevant). We collect these responses into a dataset $M = \{(S_m, k_m)\}_{m=1}^{N_M}$. The choice of subgroup size $|S_m|$ and composition can be altered per application and is independent of our proposed method. In practice, $|S_m|$ is bounded only by the LLM's context window and its ability to reliably compare many candidates at once.

\paragraph{Query complexity.}
\label{sec:appendix_query_complexity}
Although the number of possible subgroups is combinatorial in $K$, the method never enumerates them: \eqref{eq:LIP_opt} has only $K+1$ parameters, $\alpha_0, \alpha_1, \ldots, \alpha_K$, and each query contributes one choice (possibly the null option). If subgroups of mean size $s$ are sampled uniformly, source $k$ appears in $n_k \sim \mathrm{Bin}(N_M, s/K)$ queries with mean $\mu = N_M s/K$, so seeing every source a constant number of times on average requires only $N_M = O(K/s)$ queries. The count $n_k$ concentrates around $\mu$: the multiplicative Chernoff bound gives $\Pr(|n_k-\mu|\geq\epsilon\mu)\leq 2e^{-\epsilon^2\mu/3}$ for $0<\epsilon\leq 1$, and a union bound over the $K$ sources gives $|n_k-\mu|<\epsilon\mu$ for every $k$ with probability at least $1-2Ke^{-\epsilon^2\mu/3}$. This probability is at least $1-\delta$ once $\mu\geq 3\log(2K/\delta)/\epsilon^2$, i.e., once $N_M\geq 3K\log(2K/\delta)/(s\epsilon^2)$, which is $O\bigl(K\log(K/\delta)/(s\epsilon^2)\bigr)$.

To see how many queries stabilize the estimates of $\alpha_k$ as $K$ grows, we simulate the choice model \eqref{eq:lip_choice} with $\alpha_0=0$ and ground truth values of $\alpha_k$ drawn with replacement from the $99$ fitted values of $\alpha_k-\alpha_0$ in the C-MAPSS LIP (Claude responses, target engine 80). Subgroups of size $s=5$ are drawn uniformly, choices are sampled from the model, and \eqref{eq:LIP_opt} is refitted with the same optimizer as in the experiments (L-BFGS with strong Wolfe line search, Appendix~\ref{sec:appendix_hyperparameter}). The regularization coefficient $\lambda_{\mathrm{LIP}}$ is scaled in proportion to $N_M/K$, so that the balance between the regularizer and the likelihood stays fixed ($\lambda_{\mathrm{LIP}}=0.1$ at $N_M/K\approx 2$, as in C-MAPSS). Table~\ref{tab:lip_scaling} reports the Spearman correlation between the refitted and the true values of $\alpha_k$. At matched $N_M/K$, the results for $K=100$ and $K=1000$ agree to within noise, so accuracy depends on $K$ only through $N_M/K$, and the number of queries needed for a given accuracy grows linearly, not combinatorially, in $K$. The results for $K=10$ are noisier because ten sources give only a coarse ranking. Accuracy rises steeply up to $N_M\approx 2K$, i.e., about $10$ appearances per source at $s=5$, and then flattens, as the remaining error is set by the sampling noise of the choice model. In C-MAPSS, $189$--$199$ of the $200$ queries remain for each target after those that contain it are dropped (Appendix~\ref{sec:appendix_hyperparameter}), so $N_M/K\approx 2$ with $K=99$ and mean subgroup size $4$ (about $8$ appearances per source), and the same setting at $K=1000$ needs about $2{,}000$ queries of that size ($\approx 1{,}600$ at $s=5$).

\begin{table}[htb]
    \centering
    \small
    \caption{Query scaling simulation: Spearman correlation between the refitted and the true values of $\alpha_k$ (subgroup size $s=5$, mean over $5$ seeds, standard error at most $0.015$ for $K\geq 100$ and at most $0.12$ for $K=10$).}
    \label{tab:lip_scaling}
    \begin{tabular}{rccc}
    \toprule
    $N_M/K$ & $K=10$ & $K=100$ & $K=1000$ \\
    \midrule
    1  & 0.646 & 0.675 & 0.697 \\
    2  & 0.821 & 0.793 & 0.799 \\
    5  & 0.816 & 0.869 & 0.885 \\
    10 & 0.903 & 0.918 & 0.921 \\
    20 & 0.939 & 0.942 & 0.943 \\
    \bottomrule
    \end{tabular}
\end{table}

\paragraph{Ranking view and very large $K$.}
If the LLM answered without noise and $|S_m|=2$, a comparison sort such as merge sort would recover its full ordering of the sources with $O(K\log K)$ pairwise queries. Sources with $\alpha_k$ below $\alpha_0$ tie at the bottom of this ordering, because any pair of them returns the null option. If no pair returned the null option, only the lowest-ranked source can lie below $\alpha_0$, and one extra query with $S_m$ containing only that source locates the null. Since this ordering, together with the position of the null, determines the answer to every subgroup query, $O(K\log K)$ noise-free queries suffice to predict every such answer. In practice, the subgroups are sampled in advance rather than chosen adaptively as in a sort, so the queries are independent and can be issued in parallel, which reduces wall-clock time. When even the $O(K)$ queries of Table~\ref{tab:lip_scaling} are unaffordable, the LLM can act as a coding agent that writes and runs analysis code on the source files once, a fixed effort that does not grow with the number of queries, and then answers subgroup queries from the computed summaries. This is how LIP-C was produced for C-MAPSS and Hopper (Appendix~\ref{sec:appendix_agent}). For DS03, LIP-C instead uses independent calls (Appendix~\ref{sec:appendix_ds03}).

\subsection{Empirical Bayes for LIP}
Empirical Bayes is a category of methods that defines a prior based on data. With these collected LLM preferences $M$, we now estimate the prior probability distribution by maximizing its likelihood. To avoid numerical issues in the downstream Bayesian model, we parameterize $\pi_k = \sigma(\alpha_k)$ so that $\pi_k \in (0,1)$, where $\alpha_k \in \mathbb{R}$ measures how strongly the LLM favors source $k$.

Because the LLM is restricted to selecting at most one candidate per query, we model its choice as a \emph{conditional logit with an outside option} \citep{mcfadden1974conditional,luce1959individual}. The null option is a virtual alternative with $\alpha_0$ as the threshold above which a candidate is considered relevant: a source $k$ is the most probable choice when $\alpha_k$ exceeds both $\alpha_0$ and $\alpha_j$ for every other candidate $j$ in $S_m$. Concretely, the probability that the LLM returns choice $k_m \in S_m \cup \{0\}$ given the presented subgroup $S_m$ is the single softmax \eqref{eq:lip_choice} reproduced here:
\begin{equation*}
    p(k_m \mid S_m) = \frac{e^{\alpha_{k_m}}}{e^{\alpha_0} + \sum_{j \in S_m} e^{\alpha_j}},
\end{equation*}
where $\alpha_{k_m} = \alpha_0$ when $k_m = 0$ (the null is selected) and $\alpha_{k_m} = \alpha_k$ when a particular source $k$ is selected.

To find the most probable LIP, we minimize the regularized negative log-likelihood \eqref{eq:LIP_opt}:
\begin{equation*}
    \min_{\alpha_0, \alpha_1, \dots, \alpha_K}\quad -\sum_{m=1}^{N_M} \log p(k_m \mid S_m) + \lambda_{\mathrm{LIP}} \sum_{k=1}^K \left(\alpha_k - \log\frac{p_0}{1-p_0}\right)^2,
\end{equation*}
where $\lambda_{\mathrm{LIP}}>0$ is an $L_2$ regularization coefficient and $p_0$ is the default probability for a uniform prior. The objective is the sum of a convex conditional logit negative log-likelihood and a strictly convex quadratic regularizer in $\alpha_1,\ldots,\alpha_K$. The Hessian of each query's term is the covariance matrix of the one-hot choice vector under the model. Its diagonal entry in the $\alpha_0$ direction is $p(0\mid S_m)\bigl(1-p(0\mid S_m)\bigr)$, which is strictly positive because $S_m$ is nonempty, so every softmax probability lies in $(0,1)$. For a nonzero direction $v=(v_0,\ldots,v_K)$, the full Hessian $H$ satisfies $v^\intercal H v\geq 2\lambda_{\mathrm{LIP}}\sum_{k=1}^K v_k^2>0$ if $v_k\neq 0$ for some $k\geq 1$, and $v^\intercal H v=v_0^2\sum_{m=1}^{N_M}p(0\mid S_m)\bigl(1-p(0\mid S_m)\bigr)>0$ otherwise. The full Hessian is therefore positive definite everywhere when $M$ is nonempty, and the objective is strictly convex. Whenever $M$ contains at least one null choice and at least one source choice, the objective is moreover coercive. For bounded $\alpha_1,\ldots,\alpha_K$, a term in which the null option is chosen grows without bound as $\alpha_0\to-\infty$, and a term in which a source is chosen does so as $\alpha_0\to+\infty$, while the regularizer grows quadratically in $\alpha_1,\ldots,\alpha_K$ and the negative log-likelihood is nonnegative. A strictly convex, coercive objective has a unique global minimizer. If $M$ is nonempty but contains no null choice, $\alpha_0$ enters only the denominators of the source choices, so every term of the negative log-likelihood increases strictly in $\alpha_0$. The objective therefore decreases strictly as $\alpha_0 \to -\infty$ and has no finite minimizer in $\alpha_0$. The problem then reduces to the conditional logit without an outside option, whose minimizer in $\alpha_1, \ldots, \alpha_K$ exists and is unique because the regularizer is strictly convex and grows quadratically in these parameters. Symmetrically, if $M$ is nonempty and contains only null choices, the negative log-likelihood decreases strictly to $0$ as $\alpha_0\to+\infty$, and the fitted $\alpha_1,\ldots,\alpha_K$ tend to $\log[p_0/(1-p_0)]$, which gives the uniform prior $p_0$.

The regularization term serves three functions. First, it provides an anchor for the LIP. The conditional logit likelihood is invariant under a common additive shift of all $\alpha_k$ (including the null $\alpha_0$), so $\pi_k = \sigma(\alpha_k)$ is, on its own, gauge-dependent --- only the differences $\alpha_k - \alpha_0$ enter the choice probabilities. The quadratic regularizer fixes the gauge by anchoring each $\alpha_k$ near $\log[p_0/(1-p_0)]$, which makes $\pi_k = \sigma(\alpha_k)$ a meaningful probability that recovers the uniform prior $p_0$ when $M$ is empty. Second, the regularization prevents numerical instability in the optimization. If a source $k$ wins every subgroup in which it appears, the unregularized problem would drive $\alpha_k$ to infinity. Third, the regularization absorbs occasional noise in the dataset $M$ caused by LLM hallucinations. The null option parameter $\alpha_0$ is left unregularized and serves as the implicit reference level against which the $\alpha_k$ are measured.

\section{Practical Implementation of the EM Algorithm}
\label{sec:appendix_practical_em}

\subsection{Hessian Approximation}
\label{sec:hessian_prox}
One natural concern about the described method is its computational complexity. Evaluating the posterior weights $w_k^{(t)}$ via \eqref{eq:rel_prox} in the E-step requires computing and inverting the iteration-dependent Hessian $H_k^{(t)}$ of every source domain at every iteration, and the M-step \eqref{eq:m_step_update} also needs the Hessians $H_k$ and $H_0$ at the MLEs and a $d\times d$ inversion at every iteration. As the parameter dimension $d$ grows, this $\mathcal{O}(K d^3)$ operation becomes a significant bottleneck. To circumvent this computational burden, we set $\tau = 0$ in the E-step and show below that, under a Fisher approximation of the Hessians and for small $\tau$, the M-step reduces to a weighted average of the local MLEs (Algorithm~\ref{alg:lip_em_approx}).

We start with the E-step. For $\tau>0$, the two terms of \eqref{eq:rel_prox} that depend on $\tau$ are small when $\tau^2\|g_k^{(t)}\|^2$ and $\tau^2\|H_k^{(t)}\|_*$ (nuclear norm) are much smaller than one. By Remark~\ref{rmk:zero_shift_limit}, at $\tau = 0$ the quadratic form in the gradient and the log-determinant term of \eqref{eq:rel_prox} vanish, and the approximation holds with equality:
\begin{equation}
    \log p(\mathcal{D}_k \mid c_k = 1, \theta^{(t)}) = \log \mathcal{L}(\theta^{(t)}; \mathcal{D}_k).
\end{equation}
Consequently, the relevance weight \eqref{eq:E-step} reduces to a standard likelihood ratio:
\begin{equation}
\label{eq:approx_e_step}
    w_k^{(t)} = \sigma\left( \log \frac{\mathcal{L}(\theta^{(t)}; \mathcal{D}_k)}{p(\mathcal{D}_k \mid c_k = 0)} + \log \frac{\pi_k}{1 - \pi_k} \right).
\end{equation}

In the M-step, we use the approximations $H_k \approx N_k \mathcal{I}$ and $H_0 \approx N_0 \mathcal{I}$, where $\mathcal{I}$ is the expected per-sample Fisher information matrix, assumed invertible and common to the target and all sources. Since $\mathcal{I}$ commutes with $(I + \tau^2 N_k \mathcal{I})^{-1}$, the matrix $\Lambda_k^{(t)}$ of \eqref{eq:m_step_update} simplifies to $\Lambda_k^{(t)} \approx w_k^{(t)} (N_k/N_0)(I + \tau^2 N_k \mathcal{I})^{-1}$.
When $\tau^2 N_k \lambda_{\max}(\mathcal{I}) \ll 1$, the inverse $(I + \tau^2 N_k \mathcal{I})^{-1}$ is close to the identity matrix $I$, so $\Lambda_k^{(t)} \approx w_k^{(t)} (N_k/N_0)\, I$.

Finally, substituting this $\Lambda_k^{(t)}$ into \eqref{eq:m_step_update} reduces the M-step to a weighted average of the local MLEs, which is \eqref{eq:prox_update}:
\begin{equation}
    \theta^{(t+1)} = \frac{N_0\hat{\theta}_0 + \sum_{k=1}^K w_k^{(t)} N_k \hat{\theta}_k}{N_0 + \sum_{k=1}^K w_k^{(t)}N_k}.
\end{equation}
This update costs $\mathcal{O}(Kd)$ per M-step, and the influence of each source domain is governed by its sample size $N_k$ and its current relevance weight $w_k^{(t)}$.

\subsection{Bayesian Tempering}
\label{sec:bayesian_tempering}

Standard EM never decreases the observed data likelihood, but even under additional regularity conditions its limit points are only guaranteed to be stationary points, which may be poor local optima or saddle points. In a cold-start regime, the initial estimate $\hat{\theta}_0$ of the target parameter is based on an extremely small dataset $\mathcal{D}_0$ of size $N_0$. When the relevance weights $w_k^{(t)}$ are small, the iterate $\theta^{(t+1)}$ stays close to $\hat{\theta}_0$, so it depends almost exclusively on this target dataset and suffers from high variance. This volatility may sometimes drive the algorithm into a trivial local optimum where all the relevance weights $w_k^{(t)}$ are close to zero. To avoid being absorbed by this degenerate solution, we apply a Bayesian tempering mechanism. Structurally, this is in the deterministic annealing EM tradition \citep{ueda1998deterministic}, which multiplies the log-likelihood term by an inverse temperature that grows from small to large over iterations. Our specialization is a per-source variance-calibrated schedule $\beta_k^{(t)} \propto 1/\varepsilon_k$, where $\varepsilon_k$ (derived below) is the standard deviation of the first-order error that the noisy target estimate causes in the log-LR of source $k$.

To analyze how the estimation error of $\hat{\theta}_0$ affects the log-LR, which the E-step evaluates at an iterate near $\hat\theta_0$ when the weights are small, we approximate $\log\mathcal{L}(\hat{\theta}_0; \mathcal{D}_k) - \log\mathcal{L}(\theta_0; \mathcal{D}_k)$ by a first-order Taylor expansion around the true parameter $\theta_0$:
\begin{equation}
    \big|\log\mathcal{L}(\hat{\theta}_0; \mathcal{D}_k) - \log\mathcal{L}(\theta_0; \mathcal{D}_k)\big| \approx  \left| (\hat{\theta}_0 - \theta_0)^\intercal \nabla_\theta \log\mathcal{L}(\theta; \mathcal{D}_k)\big|_{\theta=\theta_0}\right|.
\end{equation}

For a source that matches the target ($\theta_k=\theta_0$), this error magnitude is governed by the inner product of two independent random vectors, since $\hat\theta_0$ depends only on $\mathcal{D}_0$ and the score only on $\mathcal{D}_k$. Assuming asymptotic normality of the MLE, the target parameter error $(\hat{\theta}_0 - \theta_0)$ has a covariance matrix approximated by the inverse target Hessian, $H_0^{-1}$. Meanwhile, $\nabla_\theta \log\mathcal{L}(\theta; \mathcal{D}_k)|_{\theta=\theta_0}$ is the score function evaluated at the true parameter, with mean zero and covariance approximated by $H_k$. As such, the inner product has mean zero and variance $\mathbb{E}\big[(\hat{\theta}_0-\theta_0)^\intercal H_k (\hat{\theta}_0-\theta_0)\big] \approx \mathrm{Tr}(H_0^{-1}H_k) \triangleq \varepsilon_k^2$. When $H_0$ is singular (e.g., when the few target samples do not identify every coordinate of $\theta$), Algorithm~\ref{alg:lip_em} uses $\varepsilon_k^2 = \mathrm{Tr}(H_0^{+}H_k)$ with the Moore--Penrose pseudo-inverse $H_0^{+}$.

To evaluate this variance, we apply the Fisher approximation of Appendix~\ref{sec:hessian_prox}, $H_0 \approx N_0 \mathcal{I}$ and $H_k \approx N_k \mathcal{I}$. Then $H_0^{-1}H_k \approx (N_k/N_0)\,I$, and the variance simplifies to $\varepsilon_k^2 \approx dN_k/N_0$, which equals $\mathrm{Tr}(H_0^{-1}H_k)$ exactly in the Gaussian mean model with known covariance $\Sigma$ ($H_0 = N_0\Sigma^{-1}$, $H_k = N_k\Sigma^{-1}$). Algorithms~\ref{alg:lip_em_approx} and~\ref{alg:lip_em_nn} use this scale, which requires no Hessian.

Therefore, the variance of the first-order error in the log-LR grows linearly in $N_k/N_0$. In the cold-start regime, the source domains are typically much richer than the target domain ($N_k \gg N_0$). Consequently, this error in the untempered log-LR can easily overwhelm the effect of the prior. In the same Gaussian mean model with known covariance $\Sigma$, the full error also has mean $-\tfrac12\mathrm{Tr}(H_0^{-1}H_k) = -dN_k/(2N_0)$, which comes from the second-order term $-\tfrac12(\hat{\theta}_0-\theta_0)^\intercal H_k(\hat{\theta}_0-\theta_0)$ that the first-order expansion omits. This negative mean lowers every source's log-LR against the null and can push all relevance weights toward zero.

To damp this error and prevent degenerate behavior, we multiply the log-LR in the E-step weight \eqref{eq:E-step} by a tempering parameter $\beta_k^{(t)}$:
\begin{equation}
\label{eq:tempering}
    w_k^{(t)} = \sigma\left(\beta_k^{(t)} \cdot \log\frac{p(\mathcal{D}_k \mid c_k = 1, \theta^{(t)})}{p(\mathcal{D}_k \mid c_k = 0)} + \log\frac{\pi_k}{1-\pi_k}\right),
\end{equation}
where $\beta_k^{(t)} = (1 - e^{-\nu t})/\varepsilon_k$ and $\nu > 0$ controls how fast $\beta_k^{(t)}$ converges to its asymptote $1/\varepsilon_k$.

This tempering parameter operates via two parallel mechanisms. The scale $1/\varepsilon_k$ brings the variance of the first-order error back to $\mathcal{O}(1)$. When the target dataset is scarce, $1/\varepsilon_k \approx \sqrt{N_0/(dN_k)}$ is small, so the weight relies on the prior more than on the noisy log-LR. Under this approximation, a richer target dataset gives a larger $1/\varepsilon_k$, which shifts the emphasis back to the data.

Meanwhile, the temporal component, $1 - e^{-\nu t}$, ensures that the effect of the likelihood is released gradually. When $t$ is small, $\beta_k^{(t)}$ is also small, so the E-step weight $w_k^{(t)}$ relies more on the prior. This protects the EM algorithm from being absorbed by the trivial solution. As the EM loop iterates, the factor $e^{-\nu t}$ vanishes exponentially fast, so $\beta_k^{(t)}$ approaches its asymptote $1/\varepsilon_k$.

Pseudocode for the Hessian reuse, small-$\tau$ approximate, and neural network variants of the procedure is given in Appendix~\ref{sec:pseudocode}.

\section{LIP-EM Algorithms}
\label{sec:appendix_algorithms}
\subsection{EM Derivation}
\label{sec:E-step-derive}
Using the Bernoulli form $\mathbb{P}(c_k) = \pi_k^{c_k}(1-\pi_k)^{1-c_k}$, the complete data log-likelihood splits as
\begin{align}
\log p(\mathbf{O}, \mathbf{c} \mid \theta)
& = \underbrace{\log \mathcal{L}(\theta; \mathcal{D}_0) + \sum_{k=1}^K c_k \log p(\mathcal{D}_k \mid c_k=1, \theta)}_{\text{terms involving } \theta} \notag\\
&\quad + \underbrace{\sum_{k=1}^K \big[(1-c_k) \log p(\mathcal{D}_k \mid c_k=0) + c_k\log \pi_k+(1-c_k)\log(1-\pi_k)\big]}_{\text{terms not involving } \theta}. \label{eq:L}
\end{align}
Since \eqref{eq:L} is linear in $\mathbf{c}$, taking its expectation under $\mathbf{c}\mid\mathbf{O},\theta^{(t)}$ replaces each $c_k$ by $w_k^{(t)} = \mathbb{P}(c_k=1\mid\mathcal{D}_k,\theta^{(t)})$. This posterior depends only on $\mathcal{D}_k$ because the complete data likelihood factorizes over $k$, and Bayes' rule gives
\begin{equation*}
w_k^{(t)} = \frac{\pi_k\, p(\mathcal{D}_k\mid c_k=1,\theta^{(t)})}{\pi_k\, p(\mathcal{D}_k\mid c_k=1,\theta^{(t)}) + (1-\pi_k)\, p(\mathcal{D}_k\mid c_k=0)},
\end{equation*}
which is \eqref{eq:E-step} after dividing the numerator and the denominator by the numerator. Dropping the $\theta$-independent terms then reduces \eqref{eq:raw_opt} to \eqref{eq:opt} (reproduced below for convenience):
\begin{equation*}
    \max_{\theta}~ \log\mathcal{L}(\theta;\mathcal{D}_0) + \sum_{k=1}^K w_k^{(t)}\log p(\mathcal{D}_k\mid c_k=1, \theta).
\end{equation*}

\subsection{Heuristics for Null Likelihood Selection}
\label{sec:null-heuristic}
The null likelihood discussion in Sec.~\ref{sec:E-step} motivates two practical choices. Neither is selected or tuned on target data. 

\paragraph{Empirical Bayes with Monte Carlo.} When source domains are abundant, we approximate $\phi_{\mathrm{null}}$ by the empirical distribution of the MLEs of the other sources, each weighted by its probability $1-w_j$ of being irrelevant. Since $p(\mathcal{D}_k\mid c_k=0) = \int \mathcal{L}(\theta;\mathcal{D}_k)\,\phi_{\mathrm{null}}(\theta)\,d\theta$, this gives a Monte Carlo estimate of the null likelihood:
\begin{equation}
\label{eq:eb-null}
    p(\mathcal{D}_k\mid c_k=0) \approx \sum_{j\ne k} \frac{1-w_j}{\sum_{i\ne k}(1-w_i)}\,\mathcal{L}(\hat\theta_j; \mathcal{D}_k),
\end{equation}
where $\hat\theta_j$ is the local MLE of source $j$ (Sec.~\ref{sec:M-step}). The weights $w_j$ are those of the previous EM iteration ($w_j = \pi_j$ at $t = 0$). As a Monte Carlo estimator, this option is more fragile when the sources are few or contaminated by relevant ones. 

\paragraph{Parametric null.} When sources are few or potentially contaminated, we fit a parametric pooled model $\theta^{\mathrm{pool}} \triangleq \argmax_\theta \mathcal{L}(\theta;\bigcup_k\mathcal{D}_k)$ on the combined source data and set $p(\mathcal{D}_k\mid c_k=0) \triangleq \mathcal{L}(\theta^{\mathrm{pool}}; \mathcal{D}_k)$. The resulting null is a smoothed version of the source empirical distribution, and the log-LR in \eqref{eq:E-step} then measures how much better $\theta^{(t)}$ explains $\mathcal{D}_k$ than the pooled model does. A source that fails this test adds no value beyond the pooled model, and its weight drops below its prior.

\subsection{Pseudocode}
\label{sec:pseudocode}
We collect three pseudocode versions of the procedure described in Sec.~\ref{sec:em_summary} and Appendix~\ref{sec:appendix_practical_em}: the closed-form Hessian reuse implementation (Algorithm~\ref{alg:lip_em}), the small-$\tau$ approximation that admits a closed-form M-step (Algorithm~\ref{alg:lip_em_approx}), and the neural network variant that replaces the closed form with gradient ascent (Algorithm~\ref{alg:lip_em_nn}).

\paragraph{Hessian reuse across iterations.} For tractability, Algorithm~\ref{alg:lip_em} computes Hessians once at the local MLEs $\hat\theta_k$ and reuses them across iterations: it identifies the iteration-dependent $H_k^{(t)}$ in \eqref{eq:rel_prox} with the frozen $H_k = -\nabla^2_\theta\log\mathcal{L}(\theta;\mathcal{D}_k)\big|_{\theta=\hat\theta_k}$. This identification is exact at $\theta^{(t)} = \hat\theta_k$. If the Hessian of $\log\mathcal{L}(\,\cdot\,;\mathcal{D}_k)$ is locally Lipschitz near $\hat\theta_k$ (e.g., if $\log\mathcal{L}(\,\cdot\,;\mathcal{D}_k)$ is $C^3$ there), its error $\|H_k^{(t)}-H_k\|$ is $\mathcal{O}(\|\theta^{(t)}-\hat\theta_k\|)$ near $\hat\theta_k$. The same identification is built into the Hessian reuse M-step \eqref{eq:m_step_update}, which maximizes the quadratic approximation of the objective $Q$ of \eqref{eq:opt} obtained by expanding each log-likelihood around its MLE, rather than $Q$ itself (the two differ by a constant when every $\log\mathcal{L}(\,\cdot\,;\mathcal{D}_k)$, $k=0,\dots,K$, is quadratic in $\theta$). A variant that recomputed and inverted the Hessians $H_k^{(t)}$ at every iterate would cost $\mathcal{O}(Kd^3)$ per iteration. Freezing $H_k$ at the local MLE also makes it positive semidefinite (by second-order optimality), so the curvature condition $I + \tau^2 H_k \succ 0$ of Proposition~\ref{prop:marginal_likelihood_thetat} holds for every $\tau$. For a Gaussian linear model with design $X_k$, as in the C-MAPSS spline model, $H_k = X_k^\intercal X_k/\hat\sigma_k^2$.

\paragraph{Precision form of the M-step.} Write $C_k \triangleq (I + \tau^2 H_k)^{-1} H_k$, so that $\Lambda_k^{(t)} = w_k^{(t)} H_0^{-1} C_k$ and $I + \sum_k \Lambda_k^{(t)} = H_0^{-1}\big(H_0 + \sum_k w_k^{(t)} C_k\big)$. Hence, for invertible $H_0$, \eqref{eq:m_step_update} equals the precision form
\begin{equation*}
    \theta^{(t+1)} = \Big(H_0 + \sum_{k=1}^K w_k^{(t)} C_k\Big)^{-1}\Big(H_0\hat\theta_0 + \sum_{k=1}^K w_k^{(t)} C_k\hat\theta_k\Big),
\end{equation*}
which does not involve $H_0^{-1}$ and therefore remains usable when $H_0$ is singular. The precision form uses $\hat\theta_0$ only through $H_0\hat\theta_0$. For a Gaussian linear model with target design $X_0$ and responses $y_0$, a singular $H_0$ makes $\hat\theta_0$ non-unique, but the normal equations give $H_0\hat\theta_0 = X_0^\intercal y_0/\hat\sigma_0^2$ for every maximizer $\hat\theta_0$, so the update is well defined. The inverse still requires $A_t \triangleq H_0 + \sum_{k=1}^K w_k^{(t)} C_k$ to be nonsingular. As $H_0$ and the $H_k$ are positive semidefinite (second-order optimality at the MLEs), so is each $C_k$: it has the eigenvectors of $H_k$ and the eigenvalues $\lambda/(1+\tau^2\lambda)\ge0$ for the eigenvalues $\lambda\ge0$ of $H_k$, so $\ker C_k = \ker H_k$. For every $v\in\mathbb{R}^d$, $v^\intercal A_t v = v^\intercal H_0 v + \sum_k w_k^{(t)} v^\intercal C_k v$ is a sum of nonnegative terms, so it vanishes exactly when $v\in\ker H_0$ and $v\in\ker H_k$ for every $k$ with $w_k^{(t)}>0$. Hence $A_t\succ0$ if and only if $\ker H_0\cap\bigcap_{k:\,w_k^{(t)}>0}\ker H_k=\{0\}$. When every weight is positive (as for $\pi\in(0,1)^K$ and finite log-LRs), this is the iteration-independent condition $H_0+\sum_{k=1}^K H_k\succ0$.

\begin{algorithm}[H]
\caption{LIP-EM with Bayesian Tempering (Hessian reuse implementation)}
\label{alg:lip_em}
\begin{algorithmic}[1]
\REQUIRE Target data $\mathcal{D}_0$, source data $\{\mathcal{D}_k\}_{k=1}^K$, LIP $\pi$, null component $\phi_{\mathrm{null}}$, prior standard \mbox{deviation} $\tau$, tempering rate $\nu$
\STATE Compute target MLE $\hat\theta_0 = \arg\max_\theta \mathcal{L}(\theta; \mathcal{D}_0)$ and Hessian $H_0 = -\nabla^2_\theta \log\mathcal{L}(\theta;\mathcal{D}_0)\big|_{\theta=\hat\theta_0}$
\FOR{$k = 1, \dots, K$}
    \STATE Compute source MLE $\hat\theta_k = \arg\max_\theta \mathcal{L}(\theta; \mathcal{D}_k)$
    \STATE Compute source Hessian $H_k = -\nabla^2_\theta \log\mathcal{L}(\theta;\mathcal{D}_k)\big|_{\theta=\hat\theta_k}$ \quad // \emph{computed once}
    \STATE Compute tempering scale $\varepsilon_k^2 = \mathrm{Tr}(H_0^{+} H_k)$ \quad // \emph{pseudo-inverse (Appendix~\ref{sec:bayesian_tempering})}
\ENDFOR
\STATE Initialize $\theta^{(0)} \gets \vec{0}$, \quad $t \gets 0$ \quad // \emph{at $t=0$, $\beta_k^{(0)}=0$ gives $w_k^{(0)} = \pi_k$ regardless of $\theta^{(0)}$}
\WHILE{not converged}
    \STATE Tempering schedule: $\beta_k^{(t)} \gets (1 - e^{-\nu t})/\varepsilon_k$ for each $k$
    \STATE $\triangleright$ \emph{E-step}
    \FOR{$k = 1, \dots, K$}
        \STATE Compute relevant likelihood $p(\mathcal{D}_k \mid c_k=1, \theta^{(t)})$ via \eqref{eq:rel_prox} with $H_k$ in place of $H_k^{(t)}$
        \STATE Compute null likelihood $p(\mathcal{D}_k \mid c_k=0)$ via $\phi_{\mathrm{null}}$
        \STATE Compute relevance weight $w_k^{(t)}$ via \eqref{eq:tempering}
    \ENDFOR
    \STATE $\triangleright$ \emph{M-step}
    \STATE Update $\theta^{(t+1)}$ via the precision form of \eqref{eq:m_step_update} \quad // \emph{requires $A_t\succ0$ (see above)}
    \STATE $t \gets t + 1$
\ENDWHILE
\RETURN $\theta^{(t)}$
\end{algorithmic}
\end{algorithm}

\begin{algorithm}[H]
\caption{LIP-EM with Bayesian Tempering and small-$\tau$ approximation}
\label{alg:lip_em_approx}
\begin{algorithmic}[1]
\REQUIRE Target data $\mathcal{D}_0$, source data $\{\mathcal{D}_k\}_{k=1}^K$, LIP $\pi$, null density $p(\cdot \mid c_k=0)$, tempering rate $\nu$
\STATE Compute target MLE $\hat\theta_0 = \arg\max_\theta \mathcal{L}(\theta; \mathcal{D}_0)$
\FOR{$k = 1, \dots, K$}
    \STATE Compute source MLE $\hat\theta_k = \arg\max_\theta \mathcal{L}(\theta; \mathcal{D}_k)$
    \STATE Set tempering scale $\varepsilon_k^2 = d\, N_k / N_0$ \quad // \emph{assumes $H_k\approx N_k\mathcal{I}$ and $H_0\approx N_0\mathcal{I}$}
\ENDFOR
\STATE Initialize $\theta^{(0)} \gets \vec{0}$, \quad $t \gets 0$
\WHILE{not converged}
    \STATE Tempering schedule: $\beta_k^{(t)} \gets (1 - e^{-\nu t})/\varepsilon_k$ for each $k$
    \STATE $\triangleright$ \emph{E-step}
    \FOR{$k = 1, \dots, K$}
        \STATE $w_k^{(t)} \gets \sigma\Big(\beta_k^{(t)}\, \log\frac{\mathcal{L}(\theta^{(t)};\mathcal{D}_k)}{p(\mathcal{D}_k\mid c_k=0)} + \log\frac{\pi_k}{1-\pi_k}\Big)$
    \ENDFOR
    \STATE $\triangleright$ \emph{M-step}
    \STATE $\displaystyle\theta^{(t+1)} \gets \frac{N_0\,\hat\theta_0 + \sum_{k=1}^K w_k^{(t)} N_k\,\hat\theta_k}{N_0 + \sum_{k=1}^K w_k^{(t)} N_k}$ \quad \eqref{eq:prox_update}
    \STATE $t \gets t + 1$
\ENDWHILE
\RETURN $\theta^{(t)}$
\end{algorithmic}
\end{algorithm}

\begin{algorithm}[H]
\caption{LIP-EM for Neural Networks ($\tau = 0$, gradient-based M-step)}
\label{alg:lip_em_nn}
\begin{algorithmic}[1]
\REQUIRE Target data $\mathcal{D}_0$, source data $\{\mathcal{D}_k\}_{k=1}^K$, LIP $\pi$, tempering rate $\nu$, gradient-based M-step optimizer $\mathrm{Opt}$ (e.g., AdamW)
\STATE Train pooled model $\theta^{\mathrm{pool}} \gets \arg\max_\theta \mathcal{L}\left(\theta;\, \bigcup_{k=1}^K \mathcal{D}_k\right)$ on the combined source data
\FOR{$k = 1, \dots, K$}
    \STATE Pre-compute null likelihood $p(\mathcal{D}_k \mid c_k=0) \gets \mathcal{L}(\theta^{\mathrm{pool}}; \mathcal{D}_k)$
    \STATE Set tempering scale $\varepsilon_k^2 = d\, N_k / N_0$ for $d$ trainable parameters \quad // \emph{heuristic (Appendix~\ref{sec:appendix_hyperparameter})}
\ENDFOR
\STATE Initialize $\theta^{(0)} \gets \theta^{\mathrm{pool}}$, \quad $t \gets 0$
\WHILE{not converged}
    \STATE Tempering schedule: $\beta_k^{(t)} \gets (1 - e^{-\nu t})/\varepsilon_k$ for each $k$
    \STATE $\triangleright$ \emph{E-step}
    \FOR{$k = 1, \dots, K$}
        \STATE $w_k^{(t)} \gets \sigma\Big(\beta_k^{(t)}\, \log\frac{\mathcal{L}(\theta^{(t)};\mathcal{D}_k)}{p(\mathcal{D}_k\mid c_k=0)} + \log\frac{\pi_k}{1-\pi_k}\Big)$
    \ENDFOR
    \STATE $\triangleright$ \emph{M-step (gradient ascent on the weighted log-likelihood)}
    \STATE $\theta^{(t+1)} \gets \mathrm{Opt}\left(\theta^{(t)};\; \log\mathcal{L}(\theta;\mathcal{D}_0) + \sum_{k=1}^K w_k^{(t)} \log\mathcal{L}(\theta;\mathcal{D}_k)\right)$
    \STATE $t \gets t + 1$
\ENDWHILE
\RETURN $\theta^{(t)}$
\end{algorithmic}
\end{algorithm}

\section{Proofs}
\label{sec:appendix_proofs}
\subsection{Proof of Proposition~\ref{prop:marginal_likelihood_thetat}}
\label{sec:proof_marginal_likelihood}
\begin{proof}
The marginal likelihood integral of Sec.~\ref{sec:E-step}, with $p(\mathcal{D}_k \mid \theta_k) = \mathcal{L}(\theta_k;\mathcal{D}_k)$, is
\begin{equation}
\label{eq:prop1_integral}
    p(\mathcal{D}_k \mid c_k = 1, \theta^{(t)}) = \int p(\mathcal{D}_k \mid \theta_k)\, \mathcal{N}(\theta_k \mid \theta^{(t)}, \tau^2 I)\, d\theta_k.
\end{equation}
Exponentiating \eqref{prox:taylor} at $\theta = \theta_k$ gives
\begin{equation*}
    p(\mathcal{D}_k \mid \theta_k) \approx \mathcal{L}(\theta^{(t)}; \mathcal{D}_k) \exp\left( {g_k^{(t)}}^\intercal (\theta_k - \theta^{(t)}) - \tfrac{1}{2}(\theta_k - \theta^{(t)})^\intercal H_k^{(t)} (\theta_k - \theta^{(t)}) \right).
\end{equation*}
Substituting this into \eqref{eq:prop1_integral} and pulling the factor $\mathcal{L}(\theta^{(t)};\mathcal{D}_k)$, which does not depend on $\theta_k$, outside the integral gives
\begin{equation}
\label{eq:prop1_after_taylor}
\begin{split}
    p(\mathcal{D}_k \mid c_k = 1, \theta^{(t)}) \approx{}& \mathcal{L}(\theta^{(t)}; \mathcal{D}_k) \int \exp\Bigl( {g_k^{(t)}}^\intercal (\theta_k - \theta^{(t)}) - \tfrac{1}{2}(\theta_k - \theta^{(t)})^\intercal H_k^{(t)} (\theta_k - \theta^{(t)}) \Bigr)\\
    &\times\mathcal{N}(\theta_k \mid \theta^{(t)}, \tau^2 I)\, d\theta_k.
\end{split}
\end{equation}

Let $u \triangleq \theta_k - \theta^{(t)}$. This change of variables is a translation, so $d\theta_k = du$. Since $\tau > 0$, the prior has the density $\mathcal{N}(\theta_k \mid \theta^{(t)}, \tau^2 I) = (2\pi\tau^2)^{-d/2} \exp(-\tfrac{1}{2\tau^2} u^\intercal u)$, so the integrand of \eqref{eq:prop1_after_taylor} becomes
\begin{equation*}
    (2\pi\tau^2)^{-d/2} \exp\left( {g_k^{(t)}}^\intercal u - \tfrac{1}{2} u^\intercal H_k^{(t)} u - \tfrac{1}{2\tau^2} u^\intercal u \right).
\end{equation*}
The two quadratic terms combine into a single quadratic form via the algebraic identity $u^\intercal H_k^{(t)} u + \tau^{-2} u^\intercal u = u^\intercal A\, u$ with
\begin{equation*}
    A \triangleq H_k^{(t)} + \tau^{-2} I.
\end{equation*}
Since $A = \tau^{-2}(I + \tau^2 H_k^{(t)})$ and $I + \tau^2 H_k^{(t)} \succ 0$ by hypothesis, we have $A \succ 0$, so the Gaussian integral below converges. Equation~\eqref{eq:prop1_after_taylor} therefore becomes
\begin{equation}
\label{eq:prop1_combined}
    p(\mathcal{D}_k \mid c_k=1, \theta^{(t)}) \approx \mathcal{L}(\theta^{(t)}; \mathcal{D}_k)\, (2\pi\tau^2)^{-d/2} \int_{\mathbb{R}^d} \exp\left( {g_k^{(t)}}^\intercal u - \tfrac{1}{2} u^\intercal A\, u \right) du.
\end{equation}

For any $A \succ 0$ and $b \in \mathbb{R}^d$, completing the square gives $b^\intercal u - \tfrac{1}{2} u^\intercal A u = -\tfrac{1}{2}(u - A^{-1}b)^\intercal A\,(u - A^{-1}b) + \tfrac{1}{2} b^\intercal A^{-1} b$, and the change of variables $v = u - A^{-1}b$ then gives
\begin{equation}
\label{eq:prop1_gauss_id}
    \int_{\mathbb{R}^d} \exp\left( b^\intercal u - \tfrac{1}{2} u^\intercal A u \right) du = (2\pi)^{d/2} (\det A)^{-1/2} \exp\left( \tfrac{1}{2} b^\intercal A^{-1} b \right),
\end{equation}
where the prefactor $(2\pi)^{d/2}(\det A)^{-1/2}$ is the normalizing constant of the density $\mathcal{N}(A^{-1} b, A^{-1})$. Applying \eqref{eq:prop1_gauss_id} with $b = g_k^{(t)}$ to \eqref{eq:prop1_combined},
\begin{equation}
\label{eq:prop1_after_id}
    p(\mathcal{D}_k \mid c_k=1, \theta^{(t)}) \approx \mathcal{L}(\theta^{(t)}; \mathcal{D}_k)\, (\tau^2)^{-d/2} (\det A)^{-1/2} \exp\left( \tfrac{1}{2} {g_k^{(t)}}^\intercal A^{-1} g_k^{(t)} \right),
\end{equation}
where the factor $(2\pi)^{-d/2}$ in \eqref{eq:prop1_combined} cancels the factor $(2\pi)^{d/2}$ in \eqref{eq:prop1_gauss_id}, leaving $(\tau^2)^{-d/2}$.

Using $\det(\alpha M) = \alpha^d \det M$ for a scalar $\alpha$ and a $d\times d$ matrix $M$,
\begin{equation}
\label{eq:prop1_det}
    (\tau^2)^{-d/2}\, (\det A)^{-1/2} = \det\big(\tau^2 A\big)^{-1/2} = \det\big(I + \tau^2 H_k^{(t)}\big)^{-1/2},
\end{equation}
where the last equality uses $\tau^2 A = I + \tau^2 H_k^{(t)}$.

As shown above, $A = \tau^{-2}(I + \tau^2 H_k^{(t)})$, and $I + \tau^2 H_k^{(t)} \succ 0$ is invertible, so $(\alpha M)^{-1} = \alpha^{-1} M^{-1}$ with $\alpha = \tau^{-2}$ and $M = I + \tau^2 H_k^{(t)}$ gives
\begin{equation}
\label{eq:prop1_inverse}
    A^{-1} = \tau^2\, \big(I + \tau^2 H_k^{(t)}\big)^{-1}.
\end{equation}

Inserting \eqref{eq:prop1_det} and \eqref{eq:prop1_inverse} into \eqref{eq:prop1_after_id},
\begin{equation*}
    p(\mathcal{D}_k \mid c_k=1, \theta^{(t)}) \approx \mathcal{L}(\theta^{(t)}; \mathcal{D}_k)\, \det\big(I + \tau^2 H_k^{(t)}\big)^{-1/2} \exp\left( \tfrac{\tau^2}{2}\, {g_k^{(t)}}^\intercal \big(I + \tau^2 H_k^{(t)}\big)^{-1} g_k^{(t)} \right).
\end{equation*}
Taking the natural logarithm of both sides yields \eqref{eq:rel_prox}.

The only approximation enters through \eqref{prox:taylor}. If that holds with equality, so does \eqref{eq:rel_prox}.
\end{proof}

\subsection{Proof of Remark~\ref{rmk:zero_shift_limit}}
\label{sec:proof_zero_shift}
\begin{proof}
We evaluate both sides of \eqref{eq:rel_prox} at $\tau = 0$.

The prior $\mathcal{N}(\theta^{(t)}, \tau^2 I)$ at $\tau = 0$ is the Dirac measure $\delta_{\theta^{(t)}}$. Substituting it into the marginal likelihood integral \eqref{eq:prop1_integral} and using the defining property of the Dirac measure,
\begin{equation*}
    p(\mathcal{D}_k \mid c_k = 1, \theta^{(t)}) = \int p(\mathcal{D}_k \mid \theta_k)\, \delta_{\theta^{(t)}}(\theta_k)\, d\theta_k = p(\mathcal{D}_k \mid \theta^{(t)}) = \mathcal{L}(\theta^{(t)}; \mathcal{D}_k).
\end{equation*}
Taking the log shows that the left-hand side of \eqref{eq:rel_prox} equals $\log \mathcal{L}(\theta^{(t)}; \mathcal{D}_k)$.

The two correction terms in \eqref{eq:rel_prox} both vanish at $\tau = 0$:
\begin{itemize}[leftmargin=*]
    \item The quadratic term carries a multiplicative factor $\tau^2$: substituting $\tau = 0$ algebraically,
    \begin{equation*}
        \tfrac{\tau^2}{2}\, {g_k^{(t)}}^\intercal \big(I + \tau^2 H_k^{(t)}\big)^{-1} g_k^{(t)}\,\Big|_{\tau = 0} = 0\cdot {g_k^{(t)}}^\intercal I^{-1} g_k^{(t)} = 0.
    \end{equation*}
    \item The log-determinant evaluates to
    \begin{equation*}
        -\tfrac{1}{2}\log\det\big(I + \tau^2 H_k^{(t)}\big)\,\Big|_{\tau = 0} = -\tfrac{1}{2}\log\det(I) = -\tfrac{1}{2}\log 1 = 0.
    \end{equation*}
\end{itemize}
The right-hand side therefore reduces to $\log \mathcal{L}(\theta^{(t)}; \mathcal{D}_k)$, which equals the left-hand side. Hence, \eqref{eq:rel_prox} holds with equality.
\end{proof}
\subsection{Proof of Remark~\ref{rmk:exact_gaussian_marginal_general}}
\label{sec:proof_gaussian_exact}
\begin{proof}
Under the Gaussian likelihood model $p(x \mid \theta_k) = \mathcal{N}(x;\,\theta_k, \Sigma_k)$ with known $\Sigma_k \succ 0$, the log-density of each sample $x_i \in \mathcal{D}_k$ is
\begin{equation*}
    \log p(x_i \mid \theta_k) = -\tfrac{1}{2}(x_i - \theta_k)^\intercal \Sigma_k^{-1} (x_i - \theta_k) - \tfrac{1}{2}\log\det(2\pi\Sigma_k).
\end{equation*}
The first term is a quadratic polynomial in $\theta_k$ and the second is constant in $\theta_k$. Since the samples are i.i.d., $\log\mathcal{L}(\theta_k;\,\mathcal{D}_k) = \sum_{i=1}^{N_k}\log p(x_i\mid\theta_k)$, and summing over the $N_k$ samples preserves the quadratic structure:
\begin{equation*}
    \log\mathcal{L}(\theta_k;\,\mathcal{D}_k) = -\tfrac{1}{2}\sum_{i=1}^{N_k}(x_i - \theta_k)^\intercal \Sigma_k^{-1}(x_i - \theta_k) - \tfrac{N_k}{2}\log\det(2\pi\Sigma_k),
\end{equation*}
where the last term does not depend on $\theta_k$.

The function $\theta_k \mapsto \log\mathcal{L}(\theta_k;\,\mathcal{D}_k)$ is a polynomial of total degree at most $2$. Its second-order Taylor expansion at any point $\theta^{(t)}$ is the right-hand side of \eqref{prox:taylor}. Since $\log\mathcal{L}(\,\cdot\,;\mathcal{D}_k)$ is $C^\infty$, Taylor's theorem with the Lagrange remainder gives the remainder of this expansion as
\begin{equation*}
    R_2(\theta_k) = \sum_{|\alpha|=3} \tfrac{1}{\alpha!}\,\partial^\alpha \log\mathcal{L}(\xi;\,\mathcal{D}_k)\, (\theta_k - \theta^{(t)})^\alpha
\end{equation*}
for some $\xi$ on the segment between $\theta^{(t)}$ and $\theta_k$, with $\alpha!=\alpha_1!\cdots\alpha_d!$. Since $\log\mathcal{L}(\theta_k;\,\mathcal{D}_k)$ is a polynomial of degree $\leq 2$, all third- and higher-order partial derivatives vanish identically, so $R_2 \equiv 0$. The expansion \eqref{prox:taylor} therefore holds with equality at every $\theta^{(t)}$. Moreover, $H_k^{(t)} = N_k\Sigma_k^{-1} \succ 0$, so the curvature condition of Proposition~\ref{prop:marginal_likelihood_thetat} holds for every $\tau > 0$, and the last step of its proof (Appendix~\ref{sec:proof_marginal_likelihood}) shows that \eqref{eq:rel_prox} then holds with equality. At $\tau = 0$, \eqref{eq:rel_prox} holds with equality by Remark~\ref{rmk:zero_shift_limit}.
\end{proof}

\subsection{Setup, Statements, and Proofs for the Finite Sample Analysis}
\label{sec:appendix_finite_sample}

This appendix collects the setup, assumptions, lemmas, and proofs for Theorem~\ref{thm:fs_main}, in which the locations $\theta_k$ of the irrelevant sources are fixed. We call these locations the configuration. The target data $\mathcal{D}_0$ consists of $N_0\ge1$ i.i.d.\ samples from $\mathcal{N}(\theta_0,\sigma^2I)$. Each source $k$ has equal sample size $N_k = N$ and data $\mathcal{D}_k = \{x_i^{(k)}\}_{i=1}^N$ drawn i.i.d.\ from $p_k \triangleq \mathcal{N}(\theta_k, \sigma^2 I)$. The null density $q(x)=p(x\mid c_k=0)$ is held fixed in the analysis, and the algorithm uses it through the product null likelihood $p(\mathcal{D}_k\mid c_k=0)=\prod_i q(x_i^{(k)})$ (Assumption~\ref{ass:regularity}). Write $R=\{k:c_k=1\}$, $\bar R=\{k:c_k=0\}$, and $\Delta_{\max}=\max_{k\in\bar R}\|\theta_k-\theta_0\|$ with $\Delta_{\max}=0$ if $\bar R=\emptyset$. We write $\mathbf{1}_R(k)$ for the indicator of $k\in R$, $a\wedge b\triangleq\min(a,b)$, $a\vee b\triangleq\max(a,b)$, and $(a)_+\triangleq\max(a,0)$ for the positive part. We assume $|R|\ge1$, since otherwise the transfer problem reduces to estimation from the target data alone.

\paragraph{Scope.} The analysis covers Algorithm~\ref{alg:lip_em_approx}, that is, the tempered E-step at $\tau=0$ with $\varepsilon_k^2=dN_k/N_0$ followed by the M-step \eqref{eq:prox_update}. Algorithm~\ref{alg:lip_em} with $\tau=0$ coincides with it in this model. Its Hessians are exactly $H_0=N_0\sigma^{-2}I$ and $H_k=N\sigma^{-2}I$, so its M-step \eqref{eq:m_step_update}, which is exact here since the log-likelihood is quadratic, reduces to \eqref{eq:prox_update}, its tempering scale is $\mathrm{Tr}(H_0^{-1}H_k)=dN/N_0$, and its E-step at $\tau=0$ uses $\mathcal{L}(\theta^{(t)};\mathcal{D}_k)$ by Remark~\ref{rmk:zero_shift_limit}. Theorem~\ref{thm:fs_main} bounds the MSE of any limit point $\theta^{(\infty)}$ of these iterates. On the good event $\mathcal{G}$ of \eqref{eq:fs_good_event}, which has $\mathbb{P}(\mathcal{G}^c)\le\delta_{\mathrm{conc}}$ by \eqref{eq:fs_good_prob}, the iterates approach $\theta^{(\infty)}$ geometrically (Lemma~\ref{lem:fs_contraction}).

\subsubsection{Assumptions}

\begin{assumption}[Generative model with $\tau = 0$, configuration fixed]
\label{ass:gen_model}
Relevant sources satisfy $\theta_k=\theta_0$ for $k\in R$, and irrelevant sources have fixed locations $\theta_k$ for $k\in\bar R$. All probabilities and expectations in Theorem~\ref{thm:fs_main} and its proof are over the data, conditional on these locations.
\end{assumption}

\begin{assumption}[Regularity for the fixed null]
\label{ass:regularity}
The null likelihood used by the algorithm is $p(\mathcal{D}_k\mid c_k=0)=\prod_{i=1}^{N}q(x_i^{(k)})$ for a fixed density $q$. The Fisher information $\mathcal{I}(\theta_0)$ is positive definite. On a neighborhood containing the basin used below, the log-LR $\ell(x;\theta)=\log[p(x\mid\theta)/q(x)]$ satisfies, for finite constants $V,V_L,L,\bar L_{\mathrm{env}}$ with $L\le \bar L_{\mathrm{env}}$:
\begin{enumerate}[label=(\roman*),leftmargin=*]
\item for every $k$, $\ell(x;\theta)-\rho_k(\theta)$ is sub-Gaussian under $p_k$ with parameter $V^2$, where $\rho_k(\theta)=\mathbb{E}_{p_k}\ell(x;\theta)$.
\item $\rho_k(\theta)$ is $L$-Lipschitz in $\theta$ for every $k$.
\item there is a function $L_\ell(x)$ with $|\ell(x;\theta)-\ell(x;\theta')|\le L_\ell(x)\|\theta-\theta'\|$, $\max_k\mathbb{E}_{p_k}L_\ell\le\bar L_{\mathrm{env}}$, and $L_\ell-\mathbb{E}_{p_k}L_\ell$ sub-Gaussian under every $p_k$ with parameter $V_L^2$.
\end{enumerate}
\end{assumption}

In the Gaussian model, $\ell(x;\theta)$ equals $x^\intercal\theta/\sigma^2-\log q(x)-\|x\|^2/(2\sigma^2)$ up to an $x$-independent term. By Gaussian concentration of Lipschitz functions \citep[Thm.~5.6]{boucheron2013concentration}, clause (i) therefore holds whenever $\log q(x)+\|x\|^2/(2\sigma^2)$ is Lipschitz in $x$. This is the case, for example, for $q=\mathcal{N}(c,\sigma^2 I)$ and for any mixture of such Gaussians whose centers lie in a bounded set. Clauses (ii) and (iii) hold on any bounded neighborhood for every such $q$, since $\ell(x;\theta)-\ell(x;\theta')=x^\intercal(\theta-\theta')/\sigma^2-(\|\theta\|^2-\|\theta'\|^2)/(2\sigma^2)$ does not depend on $q$.

\begin{assumption}[Identifiability margin]
\label{ass:margin}
There is $\Delta^*>0$ such that, for every source $k$, the expected per-sample log-LR $\rho_k(\theta_0)$ has the correct sign with margin $\Delta^*$:
\begin{equation}
\label{eq:fs_margin_lower}
    \bigl(2\mathbf{1}_R(k) - 1\bigr)\,\rho_k(\theta_0) \;\geq\; \Delta^*.
\end{equation}
Thus relevant sources have $\rho_k(\theta_0)\ge\Delta^*$ and irrelevant ones have $\rho_k(\theta_0)\le-\Delta^*$.
\end{assumption}

\begin{assumption}[LIP calibration]
\label{ass:lip_calibration}
The prior $\pi\in(0,1)^K$ puts little mass on the irrelevant sources, weighted by their distance from $\theta_0$, relative to its mass on the relevant sources:
\begin{equation}
\label{eq:lip_calibration}
    \frac{\sum_{k\in\bar R}\pi_k\,\|\theta_k-\theta_0\|}{\sum_{k\in R}\pi_k} \;\leq\; \frac{c_1\Delta^*}{2L},
\end{equation}
with $c_1=1/4$.
\end{assumption}

\begin{assumption}[Sufficient source sample size]
\label{ass:sample_size}
With $P_R=\sum_{k\in R}\pi_k$, assume
\begin{equation}
\label{eq:fs_sample_size}
\begin{aligned}
    N &\ge C_{\mathrm{unif}}\,\frac{V^2 d}{(\Delta^*)^2}\,\log\!\Bigl(1+\tfrac{12\bar L_{\mathrm{env}}}{L}\Bigr),\\
    N \wedge NP_R &\ge 512\,\frac{\sigma^2 L^2 d}{(\Delta^*)^2},
    \qquad
    N \ge N^\dagger,
\end{aligned}
\end{equation}
where $C_{\mathrm{unif}}\ge256$ is a universal constant and $N^\dagger=N^\dagger(\Delta_{\max})$ is the contraction threshold of Lemma~\ref{lem:fs_contraction}. The three clauses respectively give uniform log-LR concentration, basin entry, and contraction. No lower bound is placed on $N_0$.
\end{assumption}

\subsubsection{Notation for the analysis}

Define the per-sample log-likelihood ratio and its empirical and population averages:
\begin{equation*}
    \ell(x; \theta) \triangleq \log \frac{p(x \mid \theta)}{q(x)}, \qquad \bar{s}_k(\theta) \triangleq \frac{1}{N}\sum_{i=1}^N \ell(x_i^{(k)}; \theta), \qquad \rho_k(\theta) \triangleq \mathbb{E}_{x \sim p_k}[\ell(x; \theta)].
\end{equation*}
We abuse notation and set $\rho_k \triangleq \rho_k(\theta_0)$, the population per-sample log-LR at the truth, which matches \eqref{eq:fs_rho}. Under Assumption~\ref{ass:gen_model} ($\tau=0$), the relevant likelihood $p(\mathcal{D}_k\mid c_k=1,\theta)$ equals $\mathcal{L}(\theta;\mathcal{D}_k)=\prod_i p(x_i^{(k)}\mid\theta)$ exactly by Remark~\ref{rmk:zero_shift_limit}. Throughout the analysis, $\mathcal{B} \triangleq \{\theta : \|\theta - \theta_0\| \leq c_1\Delta^*/L\}$ denotes the \emph{basin}, with $c_1 = 1/4$ as in Assumption~\ref{ass:lip_calibration}. Write $b_k \triangleq (2\mathbf{1}_R(k)-1)\log[\pi_k/(1-\pi_k)]$ for the \emph{signed} prior log-odds of source $k$, positive when the prior favors the correct relevance label ($\pi_k>1/2$ for $k\in R$, $\pi_k<1/2$ for $k\in\bar R$), and $B_- \triangleq \max_k\,(-b_k)_+$ for the largest prior log-odds in favor of a wrong label. The weight bounds below depend on the prior only through $b_k$ and $B_-$, so only prior confidence in a wrong label weakens them. In particular, $B_-=0$ exactly when no prior favors a wrong label, and $B_- \leq \max_k|\log[\pi_k/(1-\pi_k)]|$ always.

\paragraph{Tempering schedule.} The analysis follows the tempering schedule $\beta_k^{(t)} = (1-e^{-\nu t})/\varepsilon_k$ of Appendix~\ref{sec:bayesian_tempering} at every iteration, rather than replacing it by its limit, because the schedule gives both basin entry and convergence. Two features matter. (i) At $t = 0$, $\beta_k^{(0)} = 0$, so the weights equal the prior, $w_k^{(0)} = \sigma(\log\tfrac{\pi_k}{1-\pi_k}) = \pi_k$, \emph{independently of $\theta^{(0)}$}. The first M-step is therefore the prior-weighted average of the MLEs, which places $\theta^{(1)}$ in the basin on the good event $\mathcal{G}$ (Lemma~\ref{lem:fs_basin_entry}). (ii) As $t\to\infty$, $\beta_k^{(t)}\uparrow\beta^\infty \triangleq 1/\varepsilon_k$. With $\varepsilon_k^2 = dN_k/N_0 = dN/N_0$ (Appendix~\ref{sec:bayesian_tempering}), this limit is the same for every source, and
\begin{equation}
\label{eq:fs_commit_scale}
    \beta^\infty = \sqrt{N_0/(dN)}, \qquad \mu \;\triangleq\; \beta^\infty N \Delta^* \;=\; \Delta^*\sqrt{N N_0/d},
\end{equation}
where $\mu$ controls the weight errors once $\beta_k^{(t)}$ is close to $\beta^\infty$: at $\beta=\beta^\infty$, Theorem~\ref{thm:fs_weight}(i) bounds each weight error by $e^{-c_2\mu+B_-}$. We write $t_0 \triangleq \lceil\nu^{-1}\log 2\rceil \vee 1$ for the first iteration from which $\beta_k^{(t)}\ge\tfrac12\beta^\infty$.

The total log-LR is $s_k(\theta) \triangleq N \bar s_k(\theta) = \log[p(\mathcal{D}_k\mid\theta)/q(\mathcal{D}_k)]$, with $q(\mathcal{D}_k)\triangleq\prod_i q(x_i^{(k)})$ the product null likelihood of Assumption~\ref{ass:regularity}. Under the Gaussian relevant likelihood, $s_k(\theta) = -\frac{N}{2\sigma^2}\|\hat\theta_k - \theta\|^2 + \xi_k$ for some $\theta$-independent $\xi_k$ depending on $\mathcal{D}_k$ and $q$. In this model, the tempered E-step \eqref{eq:tempering} at $\tau=0$ and the M-step \eqref{eq:prox_update} specialize to
\begin{align}
    w_k^{(t)} &= \sigma\Big(\beta_k^{(t)}\, s_k(\theta^{(t)}) + \log\tfrac{\pi_k}{1-\pi_k}\Big), \label{eq:fs_estep}\\
    \theta^{(t+1)} &= \frac{N_0\hat{\theta}_0 + N\sum_k w_k^{(t)}\hat{\theta}_k}{T^{(t)}}, \qquad T^{(t)} \triangleq N_0 + N\sum_k w_k^{(t)}, \label{eq:fs_mstep}
\end{align}
and the oracle iterate $\theta^\star$ is defined in \eqref{eq:fs_oracle}.

\subsubsection{Oracle MSE \texorpdfstring{$\mathrm{MSE}_\star$}{MSE-star}: a finite sample bias--variance identity}

\begin{proposition}[Conditional MSE of the M-step with fixed weights]
\label{prop:fs_one_step}
For deterministic weights $w \in [0,1]^K$ and any source parameters $\{\theta_k\}_{k=1}^K$, the M-step \eqref{eq:fs_mstep} satisfies
\begin{equation}
\label{eq:fs_one_step}
    \mathbb{E}\left[\|\theta^{(1)} - \theta_0\|^2 \;\big|\; \{\theta_k\}\right] = \frac{d\sigma^2\, N_{\mathrm{eff}}}{T^2} + \Bigl\|\frac{\sum_{k=1}^K w_k\, N\,(\theta_k - \theta_0)}{T}\Bigr\|^2,
\end{equation}
where $T \triangleq N_0 + N\sum_k w_k$ and $N_{\mathrm{eff}} \triangleq N_0 + N\sum_k w_k^2$.
\end{proposition}

\begin{proof}
Decompose $\theta^{(1)} - \theta_0$ as
\begin{equation*}
    \theta^{(1)} - \theta_0 = \frac{1}{T}\Bigl[N_0(\hat\theta_0 - \theta_0) + \sum_{k=1}^K w_k\, N\,(\hat\theta_k - \theta_0)\Bigr].
\end{equation*}
Conditional on $\theta_k$, $\hat\theta_k - \theta_0 = (\hat\theta_k - \theta_k) + (\theta_k - \theta_0)$ has mean $\theta_k - \theta_0$ and covariance $\sigma^2 I/N$. The target error $\hat\theta_0 - \theta_0$ has mean zero and covariance $\sigma^2 I/N_0$. Independence across sources (and target) yields
\begin{equation*}
    \mathbb{E}\left[\theta^{(1)} - \theta_0 \mid \{\theta_k\}\right] = \frac{1}{T}\sum_{k=1}^K w_k\, N\,(\theta_k - \theta_0),
\end{equation*}
\begin{equation*}
    \mathrm{Cov}\left(\theta^{(1)} \mid \{\theta_k\}\right) = \frac{1}{T^2}\Bigl[N_0^2\cdot \tfrac{\sigma^2 I}{N_0} + \sum_k (w_k N)^2\cdot \tfrac{\sigma^2 I}{N}\Bigr] = \frac{\sigma^2\, N_{\mathrm{eff}}}{T^2}\,I.
\end{equation*}
The conditional MSE is the trace of this covariance, $d\sigma^2N_{\mathrm{eff}}/T^2$, plus the squared norm of the conditional mean of $\theta^{(1)}-\theta_0$, which gives \eqref{eq:fs_one_step}. 
\end{proof}

Specializing to oracle weights $w_k = \mathbf{1}_R(k)$ gives $N_{\mathrm{eff}} = T = N_0 + N|R|$. Under Assumption~\ref{ass:gen_model} ($\tau = 0$), $\theta_k = \theta_0$ for $k\in R$, so the bias term vanishes and the oracle MSE collapses to the standard precision-weighted rate:
\begin{equation}
\label{eq:fs_oracle_mse}
    \mathrm{MSE}_\star(0) = \frac{d\sigma^2}{N_0+N|R|}.
\end{equation}
For $\tau>0$, the relevant locations are i.i.d.\ $\mathcal{N}(\theta_0,\tau^2I)$, so the bias term of \eqref{eq:fs_one_step} has expectation $N^2\,\mathbb{E}\|\sum_{k\in R}(\theta_k-\theta_0)\|^2/(N_0+N|R|)^2 = d\tau^2N^2|R|/(N_0+N|R|)^2$, which, added to the variance term, recovers $\mathrm{MSE}_\star(\tau)$ in \eqref{eq:fs_oracle_mse_body}. Below, $\mathrm{MSE}_\star$ abbreviates $\mathrm{MSE}_\star(0)$.

\subsubsection{Per-step decomposition with random weights}

Let $\gamma_k^{(t)} \triangleq w_k^{(t)} - \mathbf{1}_R(k)$ denote the weight errors, that is, the deviations of the EM weights from the oracle weights $\mathbf{1}_R(k)=c_k$. The following identity reduces the analysis of the EM iterate to controlling the weight errors $\gamma_k^{(t)}$.

\begin{lemma}[Oracle decomposition]
\label{lem:fs_decomp}
For every $t \geq 0$,
\begin{equation}
\label{eq:fs_decomp}
    \theta^{(t+1)} - \theta^\star = \frac{N}{T^{(t)}}\sum_{k=1}^K \gamma_k^{(t)}\,\bigl(\hat{\theta}_k - \theta^\star\bigr).
\end{equation}
On any event where $T^{(t)} \geq N_0 + NP_R$, the bound $N/T^{(t)} \leq 1/P_R$ yields the pointwise bound
\begin{equation}
\label{eq:fs_perstep_sharp}
    \bigl\|\theta^{(t+1)} - \theta^\star\bigr\|^2 \;\leq\; \frac{K}{P_R^{2}}\sum_{k=1}^K (\gamma_k^{(t)})^2\,\|\hat{\theta}_k - \theta^\star\|^2.
\end{equation}
\end{lemma}

\begin{proof}
Let $\gamma_k \triangleq w_k^{(t)} - \mathbf{1}_R(k)$ and $T_R \triangleq N_0 + N|R|$, so $T^{(t)} = T_R + N\sum_k \gamma_k$. Multiply \eqref{eq:fs_mstep} by $T^{(t)}$ and \eqref{eq:fs_oracle} by $T_R$:
\begin{equation*}
    T^{(t)}\theta^{(t+1)} = N_0\hat\theta_0 + N\sum_k w_k^{(t)}\hat\theta_k, \qquad T_R\theta^\star = N_0\hat\theta_0 + N\sum_{k\in R}\hat\theta_k.
\end{equation*}
Subtract:
\begin{equation*}
    T^{(t)}\theta^{(t+1)} - T_R\theta^\star = N\sum_k \gamma_k\,\hat\theta_k.
\end{equation*}
Since $T^{(t)}\theta^\star = T_R\theta^\star + N(\sum_k\gamma_k)\theta^\star$, this gives
\begin{equation}
\label{eq:fs_decomp_proof}
    T^{(t)}(\theta^{(t+1)} - \theta^\star) = N\sum_k\gamma_k(\hat\theta_k - \theta^\star),
\end{equation}
and dividing by $T^{(t)}\ge N_0>0$ gives \eqref{eq:fs_decomp}. On an event where $T^{(t)} \geq N_0 + NP_R \geq NP_R$, we have $N/T^{(t)} \leq 1/P_R$, so \eqref{eq:fs_decomp_proof}, the triangle inequality and Cauchy--Schwarz give
\begin{equation*}
    \|\theta^{(t+1)}-\theta^\star\|^2 \leq \frac{1}{P_R^{2}}\Bigl(\sum_k|\gamma_k|\,\|\hat\theta_k-\theta^\star\|\Bigr)^{2} \leq \frac{K}{P_R^{2}}\sum_k\gamma_k^2\,\|\hat\theta_k-\theta^\star\|^2,
\end{equation*}
which is \eqref{eq:fs_perstep_sharp}.
\end{proof}

\subsubsection{Identifiability margin and weight concentration}

The next theorem shows that, in the basin, the identifiability margin drives each weight toward its true label $c_k$. Part (i) is a deterministic bound that holds on a uniform concentration event, and part (ii) bounds the probability that this event fails.

\begin{theorem}[Weight concentration]
\label{thm:fs_weight}
Recall the signed prior log-odds $b_k = (2\mathbf{1}_R(k)-1)\log\tfrac{\pi_k}{1-\pi_k}$ and the largest prior log-odds in favor of a wrong label, $B_- = \max_k\,(-b_k)_+$, and assume $\Delta^* > 0$. With $c_1 = 1/4$ as in Assumption~\ref{ass:lip_calibration}, set $c_2 = 1/2$ and $c_3 = 1/256$. Define the uniform concentration event
\begin{equation}
\label{eq:fs_event_E}
    \mathcal{E}_{\mathrm{unif}} \triangleq \bigcap_{k=1}^K \Bigl\{\sup_{\theta\in\mathcal{B}}\bigl|\bar s_k(\theta) - \rho_k(\theta)\bigr| \leq \Delta^*/4\Bigr\}.
\end{equation}
\begin{enumerate}[leftmargin=*,label=(\roman*)]
\item (\emph{Deterministic bound.}) Under Assumption~\ref{ass:margin} and clause (ii) of Assumption~\ref{ass:regularity}, on $\mathcal{E}_{\mathrm{unif}}$, for every $\theta^{(t)} \in \mathcal{B}$ and every tempering value $\beta\ge0$, the signed sigmoid input in \eqref{eq:fs_estep} with $\beta$ in place of $\beta_k^{(t)}$ obeys
\begin{equation}
\label{eq:fs_sigmoid_input}
    \bigl(2\mathbf{1}_R(k) - 1\bigr)\bigl[\beta s_k(\theta^{(t)}) + \log\tfrac{\pi_k}{1-\pi_k}\bigr] \;\geq\; c_2\,\beta N\,\Delta^* + b_k \;\geq\; c_2\,\beta N\,\Delta^* - B_-,
\end{equation}
and consequently $|\gamma_k^{(t)}| \leq \exp \bigl(-c_2 \beta N \Delta^* - b_k\bigr) \leq \exp \bigl(-c_2 \beta N \Delta^* + B_-\bigr)$ (nontrivial once $c_2\beta N\Delta^* > -b_k$).
\item (\emph{Probability.}) Under Assumption~\ref{ass:regularity} and the first clause of Assumption~\ref{ass:sample_size}, and with $c_L \triangleq \bar L_{\mathrm{env}}^2/(2V_L^2)$,
\begin{equation}
\label{eq:fs_gamma_bound}
    \mathbb{P}(\mathcal{E}_{\mathrm{unif}}^c) \;\leq\; 2K\,\exp\bigl(-c_3 N (\Delta^*)^2/V^2\bigr) + K\,\exp\bigl(-c_L N\bigr),
\end{equation}
where $c_3$ absorbs the covering factor of the supremum over the compact basin $\mathcal{B}$.
\end{enumerate}
\end{theorem}

We prove Theorem~\ref{thm:fs_weight} by combining a Lipschitz bound on $\rho_k(\theta)$ (Lemma~\ref{lem:fs_lipschitz}) with sub-Gaussian concentration of the empirical log-LR (Lemma~\ref{lem:fs_concentration}) and a covering argument over $\mathcal{B}$, following the population-to-sample EM analysis template of \citet{balakrishnan2017statistical}, specialized to our sigmoid-weighted setting with per-source binary latent indicators.

\begin{lemma}[Lipschitz drift of $\rho_k$ in the basin]
\label{lem:fs_lipschitz}
Under clause (ii) of Assumption~\ref{ass:regularity}, for any $\theta\in\mathcal{B}$,
\begin{equation*}
    \bigl|\rho_k(\theta) - \rho_k\bigr| \leq L\|\theta - \theta_0\| \leq \Delta^*/4.
\end{equation*}
\end{lemma}

\begin{proof}
Clause (ii) of Assumption~\ref{ass:regularity} gives the first inequality, and $\|\theta-\theta_0\|\le c_1\Delta^*/L=\Delta^*/(4L)$ on $\mathcal{B}$ gives the second. 
\end{proof}

\begin{lemma}[Sub-Gaussian concentration of the empirical log-LR]
\label{lem:fs_concentration}
Under clause (i) of Assumption~\ref{ass:regularity}, for every fixed $\theta\in\mathcal{B}$ and every $\epsilon > 0$,
\begin{equation*}
    \mathbb{P}_{\mathcal{D}_k\sim p_k^{\otimes N}}\bigl(\,|\bar s_k(\theta) - \rho_k(\theta)| > \epsilon\,\bigr) \leq 2\exp\bigl(-N \epsilon^2/(2 V^2)\bigr).
\end{equation*}
\end{lemma}

\begin{proof}
By clause (i), the variables $\ell(x_i^{(k)};\theta)-\rho_k(\theta)$, $i=1,\dots,N$, are i.i.d., centered and sub-Gaussian with parameter $V^2$. Their mean $\bar s_k(\theta)-\rho_k(\theta)$ is therefore sub-Gaussian with parameter $V^2/N$, and the two-sided sub-Gaussian tail bound gives the claim. 
\end{proof}

\begin{proof}[Proof of Theorem~\ref{thm:fs_weight}]
\emph{(i)} Work on $\mathcal{E}_{\mathrm{unif}}$ and fix any $\theta^{(t)} \in \mathcal{B}$. Decompose
\begin{equation*}
    \bar s_k(\theta^{(t)}) = \rho_k + \bigl[\rho_k(\theta^{(t)}) - \rho_k\bigr] + \bigl[\bar s_k(\theta^{(t)}) - \rho_k(\theta^{(t)})\bigr].
\end{equation*}
By Lemma~\ref{lem:fs_lipschitz} the first bracket is at most $\Delta^*/4$ in absolute value, and the definition \eqref{eq:fs_event_E} of $\mathcal{E}_{\mathrm{unif}}$ bounds the second bracket by $\Delta^*/4$ \emph{simultaneously for all} $\theta\in\mathcal{B}$. Hence
\begin{equation*}
    \bigl|\bar s_k(\theta^{(t)}) - \rho_k\bigr| \leq \Delta^*/2,
\end{equation*}
and since $|2\mathbf{1}_R(k)-1|=1$ and the identifiability margin \eqref{eq:fs_margin_lower} gives $(2\mathbf{1}_R(k)-1)\rho_k\ge\Delta^*$, we get $(2\mathbf{1}_R(k)-1)\bar s_k(\theta^{(t)})\ge(2\mathbf{1}_R(k)-1)\rho_k-|\bar s_k(\theta^{(t)})-\rho_k|\ge\Delta^*/2$. Multiplying by $N$,
\begin{equation*}
    \bigl(2\mathbf{1}_R(k) - 1\bigr)\, s_k(\theta^{(t)}) \geq N\Delta^*/2 = c_2 N\Delta^*.
\end{equation*}
Multiplying by $\beta\ge0$ and adding the signed prior log-odds $(2\mathbf{1}_R(k)-1)\log[\pi_k/(1-\pi_k)] = b_k \geq -B_-$ gives \eqref{eq:fs_sigmoid_input}. Write $z\triangleq\beta s_k(\theta^{(t)})+\log\tfrac{\pi_k}{1-\pi_k}$. The sigmoid satisfies $\sigma(-y)=1/(1+e^{y})\le e^{-y}$ for every real $y$, so for $k\in R$ (where $\mathbf{1}_R(k)=1$ and $z\ge c_2\beta N\Delta^*+b_k$ by \eqref{eq:fs_sigmoid_input}),
\begin{equation*}
    |\gamma_k^{(t)}| = 1 - w_k^{(t)} = \sigma\bigl(-(\beta s_k(\theta^{(t)}) + \log\tfrac{\pi_k}{1-\pi_k})\bigr) \leq \exp\bigl(-c_2\beta N\Delta^* - b_k\bigr) \leq \exp\bigl(-c_2\beta N\Delta^* + B_-\bigr),
\end{equation*}
and, for $k\in\bar R$, $|\gamma_k^{(t)}|=w_k^{(t)}=\sigma(z)\le e^{z}$ with $-z\ge c_2\beta N\Delta^*+b_k$ by \eqref{eq:fs_sigmoid_input}, which gives the same bound.

\emph{(ii)} For a single fixed $\theta$, Lemma~\ref{lem:fs_concentration} with $\epsilon = \Delta^*/8$ gives the tail $2\exp(-N(\Delta^*)^2/(128V^2))$. The supremum over $\mathcal{B}$ is controlled by a net argument, which clause (iii) of Assumption~\ref{ass:regularity} makes quantitative: $r_N(\theta) \triangleq \bar s_k(\theta) - \rho_k(\theta)$ is $(\bar L_N + L)$-Lipschitz on $\mathcal{B}$ with $\bar L_N \triangleq N^{-1}\sum_i L_\ell(x_i^{(k)})$, and a $\delta$-net $\mathcal{C}_\delta$ of the ball $\mathcal{B}$, of radius $c_1\Delta^*/L$, with $\delta \triangleq \Delta^*/(24\bar L_{\mathrm{env}})$ has cardinality $|\mathcal{C}_\delta| \leq (1 + 12\bar L_{\mathrm{env}}/L)^d$. On the event $\{\bar L_N \leq 2\bar L_{\mathrm{env}}\}$, which has probability at least $1 - e^{-c_L N}$ by the sub-Gaussian part of clause (iii), every $\theta\in\mathcal{B}$ lies within $\delta$ of a net point, and $\delta(\bar L_N + L)\le3\delta\bar L_{\mathrm{env}}=\Delta^*/8$ (using $L \leq \bar L_{\mathrm{env}}$). Hence $\sup_{\mathcal{B}}|r_N| \leq \max_{\mathcal{C}_\delta}|r_N| + \Delta^*/8$, so Lemma~\ref{lem:fs_concentration} at level $\Delta^*/8$ on the net points and a union bound give
\begin{equation*}
    \mathbb{P}\Bigl(\sup_{\theta\in\mathcal{B}}|r_N(\theta)| > \Delta^*/4\Bigr) \;\leq\; 2\,(1 + 12\bar L_{\mathrm{env}}/L)^d\, e^{-N(\Delta^*)^2/(128V^2)} + e^{-c_L N}.
\end{equation*}
With $C_{\mathrm{unif}}\ge256$ in the first clause of Assumption~\ref{ass:sample_size}, $d\log(1+12\bar L_{\mathrm{env}}/L)\le N(\Delta^*)^2/(256V^2)$, so $(1+12\bar L_{\mathrm{env}}/L)^d\,e^{-N(\Delta^*)^2/(128V^2)}\le e^{-N(\Delta^*)^2/(256V^2)}$. That is, the covering factor and the pointwise rate are absorbed into $c_3=1/256$. A union bound over the $K$ sources gives $\mathbb{P}(\mathcal{E}_{\mathrm{unif}}^c) \leq 2K\exp(-c_3 N(\Delta^*)^2/V^2) + K\exp(-c_L N)$.
\end{proof}

\subsubsection{EM dynamics: basin entry, invariance, and convergence}
\label{sec:fs_dynamics}

Throughout this subsection, Assumptions~\ref{ass:gen_model}--\ref{ass:sample_size} are in force. We show that the tempered iterates enter the basin $\mathcal{B}$ at $t=1$, stay in $\mathcal{B}$, and converge to a fixed point in $\mathcal{B}$. Define the target error radius $\zeta_0$ and the tail level $u_{\mathrm{tail}}$ by
\begin{equation*}
    \zeta_0 \triangleq 2\sigma\sqrt{(d+u_{\mathrm{tail}})/N_0}, \qquad u_{\mathrm{tail}} \triangleq \frac{(\Delta^*)^2\, N P_R}{512\, L^2\sigma^2}.
\end{equation*}
The lemmas below hold on the good event
\begin{equation}
\label{eq:fs_good_event}
    \mathcal{G} \triangleq \mathcal{E}_{\mathrm{unif}} \cap \bigl\{\|\hat\theta_0-\theta_0\|\le \zeta_0\bigr\} \cap \bigl\{\textstyle\max_k\|\hat\theta_k-\theta_k\|\le \tfrac{c_1\Delta^*}{4L}\bigr\},
\end{equation}
where $\mathcal{E}_{\mathrm{unif}}$ is the uniform concentration event \eqref{eq:fs_event_E}. We bound $\mathbb{P}(\mathcal{G}^c)$ with three tools: Theorem~\ref{thm:fs_weight}(ii) for $\mathcal{E}_{\mathrm{unif}}^c$, the $\chi^2$ Chernoff bound $\mathbb{P}_{z\sim\mathcal{N}(0,I_d)}\bigl(\|z\|^2 \ge 4(d+u)\bigr) \le e^{-c_4(d+u)}$, which holds for every $u\ge0$ with $c_4=(3-\log 4)/2>4/5$,\footnote{From the standard tail $\mathbb{P}(\chi^2_d \ge y) \le e^{-\frac12(y - d - d\log(y/d))}$ at $y = 4a$ with $a \triangleq d+u$: since $d\log\tfrac{4a}{d} \le d\log 4 + a - d$ for $a \ge d$, the exponent is at least $\tfrac12(3a - d\log 4) \ge \tfrac12(3-\log 4)\,a = c_4\,a$, where the second inequality uses $d\le a$.} applied to $\hat\theta_0-\theta_0$, and Gaussian tails of the source MLE errors. Together they give universal constants $C_g,c_g>0$ such that
\begin{equation}
\label{eq:fs_good_prob}
    \mathbb{P}(\mathcal{G}^c) \le \delta_{\mathrm{conc}} \triangleq 2K e^{-c_3 N(\Delta^*)^2/V^2} + K e^{-c_L N} + C_g K e^{-c_g N(\Delta^*)^2/(L^2\sigma^2)} + e^{-c_4(d+u_{\mathrm{tail}})},
\end{equation}
valid under the sample size Assumption~\ref{ass:sample_size}. Its uniform concentration clause is the condition of Theorem~\ref{thm:fs_weight}(ii), which gives the first two terms. Its basin entry clause gives $N(c_1\Delta^*/(4L))^2/\sigma^2\ge2d$, so the threshold for the $\chi^2_d$ variable $N\|\hat\theta_k-\theta_k\|^2/\sigma^2$ is at least twice its mean, which gives the third term its exponential form. The fourth (target) term is the $\chi^2$ bound at level $u_{\mathrm{tail}}$. Each term decays exponentially in $N$.

\begin{lemma}[Basin entry at $t=1$]
\label{lem:fs_basin_entry}
On $\mathcal{G}$, under Assumption~\ref{ass:lip_calibration}, $\theta^{(1)}\in \mathcal{B}$.
\end{lemma}

\begin{proof}
Since $\beta_k^{(0)}=0$, the weights are exactly the prior, $w_k^{(0)}=\pi_k$. With $T_\pi \triangleq N_0+N\sum_k\pi_k$, splitting $\hat\theta_k-\theta_0=(\hat\theta_k-\theta_k)+(\theta_k-\theta_0)$ and using $\theta_k=\theta_0$ for $k\in R$ (Assumption~\ref{ass:gen_model}) gives
\begin{equation*}
    \theta^{(1)}-\theta_0 = \underbrace{\frac{N\sum_{k\in\bar R}\pi_k(\theta_k-\theta_0)}{T_\pi}}_{\text{bias}} + \underbrace{\frac{N_0(\hat\theta_0-\theta_0)+N\sum_k\pi_k(\hat\theta_k-\theta_k)}{T_\pi}}_{\text{noise}}.
\end{equation*}
By the triangle inequality, the noise has norm at most $\frac{N_0\|\hat\theta_0-\theta_0\|}{T_\pi}+\frac{N\sum_k\pi_k\|\hat\theta_k-\theta_k\|}{T_\pi}$. On $\mathcal{G}$ the source part is at most $\max_k\|\hat\theta_k-\theta_k\|\le c_1\Delta^*/(4L)$, because its coefficients $N\pi_k/T_\pi$ are nonnegative and sum to at most $1$. For the target part, $\|\hat\theta_0-\theta_0\|\le\zeta_0$ and $T_\pi\ge N_0+NP_R\ge 2\sqrt{N_0\,NP_R}$ (AM--GM) give
\begin{equation*}
    \frac{N_0\zeta_0}{T_\pi} \le \frac{2\sigma\sqrt{(d+u_{\mathrm{tail}}) N_0}}{2\sqrt{N_0\,NP_R}} = \sigma\sqrt{\tfrac{d+u_{\mathrm{tail}}}{NP_R}} \le \frac{c_1\Delta^*}{4L},
\end{equation*}
where the last inequality holds because $\sigma^2 u_{\mathrm{tail}}/(NP_R) = (\Delta^*)^2/(512L^2) = \tfrac12\bigl(\tfrac{c_1\Delta^*}{4L}\bigr)^2$ by the definition of $u_{\mathrm{tail}}$, and $\sigma^2 d/(NP_R) \le \tfrac12\bigl(\tfrac{c_1\Delta^*}{4L}\bigr)^2$ under the basin entry clause $NP_R \ge 512\,\sigma^2 L^2 d/(\Delta^*)^2$ of Assumption~\ref{ass:sample_size}. Adding the bounds on the source and target parts shows that the noise has norm at most $c_1\Delta^*/(2L)$. For the bias, $T_\pi\ge N\sum_{k\in R}\pi_k$, so by Assumption~\ref{ass:lip_calibration},
\begin{equation*}
    \|\text{bias}\| \le \frac{\sum_{k\in\bar R}\pi_k\|\theta_k-\theta_0\|}{\sum_{k\in R}\pi_k} \le \frac{c_1\Delta^*}{2L}.
\end{equation*}
Summing the two bounds gives $\|\theta^{(1)}-\theta_0\|\le c_1\Delta^*/L$, that is, $\theta^{(1)}\in\mathcal{B}$.
\end{proof}

\begin{lemma}[Forward invariance and weight monotonicity]
\label{lem:fs_invariance}
For $t\ge0$, let $F^{(t)}$ be the map that sends $\theta$ to the M-step \eqref{eq:fs_mstep} with the weights \eqref{eq:fs_estep} evaluated at $\theta$ in place of $\theta^{(t)}$, and let the limiting map $F^{(\infty)}$ be the same map with $\beta_k^{(t)}\equiv\beta^\infty$. On $\mathcal{G}$, every $F^{(t)}$ maps $\mathcal{B}$ into itself, and so does $F^{(\infty)}$. Hence, once $\theta^{(1)}\in\mathcal{B}$, the iterates satisfy $\theta^{(t)}\in \mathcal{B}$ for all $t\ge 1$. Moreover, at every $\theta\in\mathcal{B}$ the weights $w_k^{(t)}(\theta)$, and also those at tempering $\beta^\infty$, satisfy $w_k^{(t)}(\theta)\ge\pi_k$ for $k\in R$ and $w_k^{(t)}(\theta)\le\pi_k$ for $k\in\bar R$.
\end{lemma}

\begin{proof}
Work on $\mathcal{G}$. Fix $t\ge0$ and $\theta\in\mathcal{B}$. Write $w_k\triangleq w_k^{(t)}(\theta)$ for the weights of \eqref{eq:fs_estep} with $\theta$ in place of $\theta^{(t)}$, and $T\triangleq T^{(t)}(\theta)=N_0+N\sum_k w_k$. The map $F^{(t)}$ applies the M-step \eqref{eq:fs_mstep} with these weights, so $F^{(t)}(\theta)=\bigl(N_0\hat\theta_0+N\sum_k w_k\hat\theta_k\bigr)/T$.

We first show that the log-LR has the sign of the true label on $\mathcal{B}$, by repeating the argument of Theorem~\ref{thm:fs_weight}(i) at $\theta$. Since $\theta\in\mathcal{B}$, we have $\|\theta-\theta_0\|\le c_1\Delta^*/L=\Delta^*/(4L)$, so Lemma~\ref{lem:fs_lipschitz} gives $|\rho_k(\theta)-\rho_k|\le\Delta^*/4$. Since $\mathcal{G}\subseteq\mathcal{E}_{\mathrm{unif}}$ and $\theta\in\mathcal{B}$, \eqref{eq:fs_event_E} gives $|\bar s_k(\theta)-\rho_k(\theta)|\le\Delta^*/4$. By the triangle inequality,
\begin{equation*}
    \bigl|\bar s_k(\theta)-\rho_k\bigr| \le \frac{\Delta^*}{2}.
\end{equation*}
The identifiability margin \eqref{eq:fs_margin_lower} gives $(2\mathbf{1}_R(k)-1)\rho_k\ge\Delta^*$. As $|2\mathbf{1}_R(k)-1|=1$,
\begin{equation*}
    \bigl(2\mathbf{1}_R(k)-1\bigr)\,\bar s_k(\theta) \ge \bigl(2\mathbf{1}_R(k)-1\bigr)\rho_k-\bigl|\bar s_k(\theta)-\rho_k\bigr| \ge \Delta^*-\frac{\Delta^*}{2} = \frac{\Delta^*}{2}.
\end{equation*}
Multiplying by $N$ and using $s_k(\theta)=N\bar s_k(\theta)$ gives $(2\mathbf{1}_R(k)-1)\,s_k(\theta)\ge N\Delta^*/2>0$. That is,
\begin{equation*}
    s_k(\theta)>0 \quad (k\in R), \qquad s_k(\theta)<0 \quad (k\in\bar R).
\end{equation*}

The tempering is nonnegative, $\beta_k^{(t)}=(1-e^{-\nu t})/\varepsilon_k\ge0$, so the term $\beta_k^{(t)}s_k(\theta)$ in \eqref{eq:fs_estep} is $\ge0$ for $k\in R$ and $\le0$ for $k\in\bar R$. The sigmoid is increasing and takes the value $\pi_k$ at $\log\tfrac{\pi_k}{1-\pi_k}$. Hence
\begin{equation*}
    w_k\ge\pi_k \quad (k\in R), \qquad w_k\le\pi_k \quad (k\in\bar R),
\end{equation*}
which is the monotonicity claim. Since also $w_k\ge0$ for every $k$, and $P_R=\sum_{k\in R}\pi_k>0$ because $|R|\ge1$, the denominator obeys
\begin{equation*}
    T \ge N_0+N\sum_{k\in R}w_k \ge N_0+NP_R \ge NP_R > 0.
\end{equation*}

Next we bound the distance from $F^{(t)}(\theta)$ to $\theta_0$. Since $T\theta_0=N_0\theta_0+N\sum_k w_k\theta_0$, subtracting $\theta_0$ from $F^{(t)}(\theta)$ gives
\begin{equation*}
    F^{(t)}(\theta)-\theta_0 = \frac{N_0(\hat\theta_0-\theta_0)+N\sum_k w_k(\hat\theta_k-\theta_0)}{T}.
\end{equation*}
Split $\hat\theta_k-\theta_0=(\hat\theta_k-\theta_k)+(\theta_k-\theta_0)$ and use $\theta_k=\theta_0$ for $k\in R$ (Assumption~\ref{ass:gen_model}). This gives
\begin{equation*}
    F^{(t)}(\theta)-\theta_0
    = \underbrace{\frac{N_0(\hat\theta_0-\theta_0)+N\sum_k w_k(\hat\theta_k-\theta_k)}{T}}_{\text{noise}}
    + \underbrace{\frac{N\sum_{k\in\bar R}w_k(\theta_k-\theta_0)}{T}}_{\text{bias}},
\end{equation*}
which is the decomposition of Lemma~\ref{lem:fs_basin_entry} with $w_k$ in place of $\pi_k$.

By the triangle inequality, the noise has norm at most
\begin{equation*}
    \frac{N\sum_k w_k\|\hat\theta_k-\theta_k\|}{T}+\frac{N_0\|\hat\theta_0-\theta_0\|}{T}.
\end{equation*}
In the first (source) term, the coefficients $Nw_k/T$ are nonnegative and sum to $N\sum_k w_k/T\le1$. On $\mathcal{G}$, \eqref{eq:fs_good_event} gives $\max_k\|\hat\theta_k-\theta_k\|\le c_1\Delta^*/(4L)$, so
\begin{equation*}
    \frac{N\sum_k w_k\|\hat\theta_k-\theta_k\|}{T} \le \max_k\|\hat\theta_k-\theta_k\| \le \frac{c_1\Delta^*}{4L}.
\end{equation*}
For the second (target) term, $\|\hat\theta_0-\theta_0\|\le\zeta_0$ on $\mathcal{G}$, the lower bound on $T$ above gives $T\ge N_0+NP_R$, and AM--GM gives $N_0+NP_R\ge2\sqrt{N_0\,NP_R}$. With $\zeta_0=2\sigma\sqrt{(d+u_{\mathrm{tail}})/N_0}$,
\begin{equation*}
    \frac{N_0\|\hat\theta_0-\theta_0\|}{T} \le \frac{N_0\zeta_0}{2\sqrt{N_0\,NP_R}} = \sigma\sqrt{\frac{d+u_{\mathrm{tail}}}{NP_R}} \le \frac{c_1\Delta^*}{4L}.
\end{equation*}
As in Lemma~\ref{lem:fs_basin_entry}, the definition of $u_{\mathrm{tail}}$ gives $\sigma^2u_{\mathrm{tail}}/(NP_R)=(\Delta^*)^2/(512L^2)=\tfrac12\bigl(\tfrac{c_1\Delta^*}{4L}\bigr)^2$, and the basin entry clause $NP_R\ge512\,\sigma^2L^2d/(\Delta^*)^2$ of Assumption~\ref{ass:sample_size} gives $\sigma^2d/(NP_R)\le\tfrac12\bigl(\tfrac{c_1\Delta^*}{4L}\bigr)^2$. Adding these two bounds and taking square roots gives the last inequality. Adding the bounds on the source and target terms shows that the noise has norm at most $c_1\Delta^*/(2L)$.

For the bias, applying, in order, the triangle inequality, then $0\le w_k\le\pi_k$ for $k\in\bar R$ together with $T\ge NP_R>0$ (both shown above), and then Assumption~\ref{ass:lip_calibration} gives
\begin{equation*}
    \Bigl\|\frac{N\sum_{k\in\bar R}w_k(\theta_k-\theta_0)}{T}\Bigr\| \le \frac{N\sum_{k\in\bar R}w_k\|\theta_k-\theta_0\|}{T} \le \frac{N\sum_{k\in\bar R}\pi_k\|\theta_k-\theta_0\|}{NP_R} \le \frac{c_1\Delta^*}{2L}.
\end{equation*}

Combining the noise and bias bounds by the triangle inequality,
\begin{equation*}
    \|F^{(t)}(\theta)-\theta_0\| \le \frac{c_1\Delta^*}{2L}+\frac{c_1\Delta^*}{2L} = \frac{c_1\Delta^*}{L},
\end{equation*}
so $F^{(t)}(\theta)\in\mathcal{B}$. As $\theta\in\mathcal{B}$ and $t\ge0$ were arbitrary, $F^{(t)}(\mathcal{B})\subseteq\mathcal{B}$ for every $t$. The argument above uses the tempering only through $\beta_k^{(t)}\ge0$. It therefore applies verbatim to the limiting map $F^{(\infty)}$, whose tempering $\beta_k\equiv\beta^\infty=1/\varepsilon_k$ is also nonnegative. Hence $F^{(\infty)}(\mathcal{B})\subseteq\mathcal{B}$, and at every $\theta\in\mathcal{B}$ the weights at tempering $\beta^\infty$ are monotone in the same way and their denominator satisfies $T^{(\infty)}(\theta)\ge N_0+NP_R$. For the iterates, $\theta^{(1)}\in\mathcal{B}$ by Lemma~\ref{lem:fs_basin_entry}, and $\theta^{(t)}\in\mathcal{B}$ implies $\theta^{(t+1)}=F^{(t)}(\theta^{(t)})\in\mathcal{B}$. Induction on $t\ge1$ completes the proof.
\end{proof}

\begin{lemma}[Contraction and convergence]
\label{lem:fs_contraction}
Let $\bar D_{\mathcal{B}} \triangleq \tfrac{5c_1\Delta^*}{4L}+\Delta_{\max}$ and
\begin{equation*}
    N^\dagger \;\triangleq\; \frac{64\,d}{(\Delta^*)^2 N_0}\,\log^2\Bigl(e+\frac{1400\,K \bar D_{\mathcal{B}}^2 e^{B_-} d}{\sigma^2 (\Delta^*)^3 N_0^2}\Bigr),
\end{equation*}
which is polynomial in $d/((\Delta^*)^2 N_0)$ and polylogarithmic in the remaining quantities. Then, on $\mathcal{G}$, for $N\ge N^\dagger$ and every $t\ge t_0$, the map $F^{(t)}$ of Lemma~\ref{lem:fs_invariance} obeys
\begin{equation}
\label{eq:fs_jacobian}
    \sup_{\theta\in \mathcal{B}}\bigl\|DF^{(t)}(\theta)\bigr\|_{\mathrm{op}} \le \kappa \triangleq \frac{N^{3/2} K \bar D_{\mathcal{B}}^2 e^{B_-}}{\sigma^2\sqrt{d N_0}}\, e^{-c_2\mu/2} \le \tfrac12,
\end{equation}
with $\mu=\Delta^*\sqrt{NN_0/d}$ as in \eqref{eq:fs_commit_scale}. Consequently the iterates converge geometrically on $\mathcal{G}$ to $\theta^{(\infty)}$, the unique fixed point in $\mathcal{B}$ of the limiting map $F^{(\infty)}$ (with $\beta_k^{(t)}\equiv\beta^\infty$).
\end{lemma}

\begin{proof}
Work on $\mathcal{G}$ and fix $t\ge t_0$. Write $F\triangleq F^{(t)}$ and, at a point $\theta$, $w_k\triangleq w_k^{(t)}(\theta)$, $T\triangleq T^{(t)}(\theta)$ and $z_k\triangleq\beta_k^{(t)}s_k(\theta)+\log\tfrac{\pi_k}{1-\pi_k}$, so that \eqref{eq:fs_estep} reads $w_k=1/(1+e^{-z_k})$. All sums run over the sources $k=1,\dots,K$.

We first compute the Jacobian of $F$. Write \eqref{eq:fs_mstep} as $F=A/T$ with $A\triangleq N_0\hat\theta_0+N\sum_k w_k\hat\theta_k$. The target enters $A$ only through the constant $N_0\hat\theta_0$ and enters $T$ only through the constant $N_0$, so
\begin{equation*}
    DA = N\sum_k\hat\theta_k\,\nabla_\theta w_k^\intercal, \qquad DT = N\sum_k\nabla_\theta w_k^\intercal.
\end{equation*}
The quotient rule $D(A/T)=T^{-1}DA-T^{-2}A\,DT=T^{-1}(DA-F\,DT)$ then gives
\begin{equation*}
    DF(\theta) = \frac{N}{T}\sum_k\bigl(\hat\theta_k-F(\theta)\bigr)\,\nabla_\theta w_k^\intercal.
\end{equation*}
The $\theta$-dependence of the target weight $N_0/T$ is carried entirely by the term $-F\,DT$, so the target contributes to $DF$ only through $F$. The tempering $\beta_k^{(t)}$ and the prior log-odds do not depend on $\theta$, and $s_k(\theta)=-\tfrac{N}{2\sigma^2}\|\hat\theta_k-\theta\|^2+\xi_k$ gives $\nabla_\theta s_k(\theta)=\tfrac{N}{\sigma^2}(\hat\theta_k-\theta)$. The chain rule with $\partial w_k/\partial z_k=w_k(1-w_k)$ gives $\nabla_\theta w_k=w_k(1-w_k)\,\beta_k^{(t)}\tfrac{N}{\sigma^2}(\hat\theta_k-\theta)$. Hence
\begin{equation*}
    DF(\theta) = \frac{N^2}{T\sigma^2}\sum_k\beta_k^{(t)}\,w_k(1-w_k)\,\bigl(\hat\theta_k-F(\theta)\bigr)(\hat\theta_k-\theta)^\intercal.
\end{equation*}

Next we bound the factor $w_k(1-w_k)$. Fix $\theta\in\mathcal{B}$ and let $\tilde z_k\triangleq(2\mathbf{1}_R(k)-1)z_k$ be the signed sigmoid input. Since $\mathcal{G}\subseteq\mathcal{E}_{\mathrm{unif}}$, Theorem~\ref{thm:fs_weight}(i) with $\beta=\beta_k^{(t)}$ gives
\begin{equation*}
    \tilde z_k \ge c_2\beta_k^{(t)}N\Delta^*+b_k \ge c_2\beta_k^{(t)}N\Delta^*-B_-.
\end{equation*}
For $k\in R$ we have $\tilde z_k=z_k$ and $1-w_k=1/(1+e^{z_k})=1/(1+e^{\tilde z_k})$. For $k\in\bar R$ we have $\tilde z_k=-z_k$ and $w_k=1/(1+e^{-z_k})=1/(1+e^{\tilde z_k})$. As $w_k(1-w_k)\le\min(w_k,1-w_k)$ and $1/(1+e^{y})\le e^{-y}$ for every real $y$,
\begin{equation*}
    w_k(1-w_k) \le \frac{1}{1+e^{\tilde z_k}} \le e^{-\tilde z_k} \le e^{B_-}\,e^{-c_2\beta_k^{(t)}N\Delta^*}.
\end{equation*}

The distances in $DF(\theta)$ are at most $\bar D_{\mathcal{B}}$. To see this, fix $\theta\in\mathcal{B}$. Every $\vartheta\in\mathcal{B}$ satisfies $\|\vartheta-\theta_0\|\le c_1\Delta^*/L$. This covers $\vartheta=\theta$, and also $\vartheta=F(\theta)$ because Lemma~\ref{lem:fs_invariance} gives $F(\theta)\in\mathcal{B}$. On $\mathcal{G}$, \eqref{eq:fs_good_event} gives $\|\hat\theta_k-\theta_k\|\le c_1\Delta^*/(4L)$. Finally $\|\theta_k-\theta_0\|\le\Delta_{\max}$ for every $k$, as it is $0$ for $k\in R$ (Assumption~\ref{ass:gen_model}) and at most $\Delta_{\max}$ for $k\in\bar R$ by definition. The triangle inequality gives
\begin{equation*}
    \|\hat\theta_k-\vartheta\| \le \|\hat\theta_k-\theta_k\|+\|\theta_k-\theta_0\|+\|\theta_0-\vartheta\| \le \frac{c_1\Delta^*}{4L}+\Delta_{\max}+\frac{c_1\Delta^*}{L} = \bar D_{\mathcal{B}}.
\end{equation*}

We now bound the operator norm. Fix $\theta\in\mathcal{B}$. Each summand of $DF(\theta)$ above is rank one, with $\|(\hat\theta_k-F(\theta))(\hat\theta_k-\theta)^\intercal\|_{\mathrm{op}}=\|\hat\theta_k-F(\theta)\|\,\|\hat\theta_k-\theta\|\le\bar D_{\mathcal{B}}^2$ by the distance bound. Further, $T\ge N_0$ because $w_k\ge0$. The tempering schedule gives $\beta_k^{(t)}=(1-e^{-\nu t})\beta^\infty$ with $\beta^\infty=\sqrt{N_0/(dN)}$ by \eqref{eq:fs_commit_scale}, so $\beta_k^{(t)}\le\beta^\infty$. Since $t\ge t_0\ge\nu^{-1}\log2$, we have $e^{-\nu t}\le\tfrac12$, and hence $\beta_k^{(t)}\ge\tfrac12\beta^\infty$. With $\mu=\beta^\infty N\Delta^*$ from \eqref{eq:fs_commit_scale}, this gives $c_2\beta_k^{(t)}N\Delta^*\ge\tfrac12c_2\mu$ for every $k$. Combining these bounds with the bound on $w_k(1-w_k)$ and the triangle inequality for $\|\cdot\|_{\mathrm{op}}$,
\begin{equation*}
    \|DF(\theta)\|_{\mathrm{op}} \le \frac{N^2}{N_0\sigma^2}\,\beta^\infty\,\bar D_{\mathcal{B}}^2\,e^{B_-}\sum_k e^{-c_2\beta_k^{(t)}N\Delta^*} \le \frac{N^2\beta^\infty}{N_0}\cdot\frac{K\bar D_{\mathcal{B}}^2e^{B_-}}{\sigma^2}\,e^{-c_2\mu/2} = \kappa,
\end{equation*}
where the last equality uses $\tfrac{N^2}{N_0}\beta^\infty=\tfrac{N^2}{N_0}\sqrt{N_0/(dN)}=N^{3/2}/\sqrt{dN_0}$. As $\theta\in\mathcal{B}$ was arbitrary, this is the first inequality in \eqref{eq:fs_jacobian}.

It remains to show that $\kappa\le\tfrac12$ when $N\ge N^\dagger$. With $c_2=\tfrac12$ (Theorem~\ref{thm:fs_weight}), put
\begin{equation*}
    x\triangleq\frac{c_2\mu}{2}=\frac{\Delta^*}{4}\sqrt{\frac{NN_0}{d}}, \qquad M\triangleq\frac{K\bar D_{\mathcal{B}}^2e^{B_-}d}{\sigma^2(\Delta^*)^3N_0^2}.
\end{equation*}
Solving the first identity for $N$ gives $N=16x^2d/((\Delta^*)^2N_0)$, so $N^{3/2}=64x^3\bigl(d/((\Delta^*)^2N_0)\bigr)^{3/2}$. Substituting into $\kappa$,
\begin{equation*}
    \kappa = 64x^3\,\frac{d^{3/2}}{(\Delta^*)^3N_0^{3/2}}\cdot\frac{K\bar D_{\mathcal{B}}^2e^{B_-}}{\sigma^2\sqrt{dN_0}}\,e^{-x} = 64\,M\,x^3e^{-x}.
\end{equation*}
On $x\ge0$ the function $x^3e^{-x/2}$ is maximized at $x=6$, so $x^3e^{-x/2}\le(6/e)^3$. As $64(6/e)^3<700$,
\begin{equation*}
    \kappa = 64M\,\bigl(x^3e^{-x/2}\bigr)\,e^{-x/2} \le 700\,M\,e^{-x/2}.
\end{equation*}
The hypothesis $N\ge N^\dagger=\tfrac{64d}{(\Delta^*)^2N_0}\log^2(e+1400M)$ gives
\begin{equation*}
    x = \frac{\Delta^*}{4}\sqrt{\frac{NN_0}{d}} \ge \frac{\Delta^*}{4}\cdot\frac{8}{\Delta^*}\,\log(e+1400M) = 2\log(e+1400M).
\end{equation*}
Hence $e^{-x/2}\le1/(e+1400M)$, and
\begin{equation*}
    \kappa \le \frac{700M}{e+1400M} < \frac12.
\end{equation*}

We now deduce the contraction. The basin $\mathcal{B}$ is a closed ball, hence convex and complete. The map $F^{(t)}$ is continuously differentiable on $\mathbb{R}^d$, because the weights are smooth in $\theta$ and $T\ge N_0>0$. For $\theta,\theta'\in\mathcal{B}$ the segment between them lies in $\mathcal{B}$, so, for every $t\ge t_0$, the integral form of the mean value theorem and \eqref{eq:fs_jacobian} give
\begin{equation*}
    \bigl\|F^{(t)}(\theta)-F^{(t)}(\theta')\bigr\| = \Bigl\|\int_0^1 DF^{(t)}\bigl(\theta'+\lambda(\theta-\theta')\bigr)(\theta-\theta')\,d\lambda\Bigr\| \le \kappa\,\|\theta-\theta'\|.
\end{equation*}
With Lemma~\ref{lem:fs_invariance}, each $F^{(t)}$ with $t\ge t_0$ is thus a $\kappa$-contraction of $\mathcal{B}$ into itself. Lemma~\ref{lem:fs_invariance} also gives $F^{(\infty)}(\mathcal{B})\subseteq\mathcal{B}$. The bound on the Jacobian therefore applies verbatim to $F^{(\infty)}$, because the distance bound needs only this inclusion and the operator norm bound needs only a tempering in $[\tfrac12\beta^\infty,\beta^\infty]$, which $\beta_k\equiv\beta^\infty$ satisfies. So $F^{(\infty)}$ is also a $\kappa$-contraction of $\mathcal{B}$ into itself. By the Banach fixed-point theorem it has a fixed point $\theta^{(\infty)}=F^{(\infty)}(\theta^{(\infty)})\in\mathcal{B}$, unique in $\mathcal{B}$.

Finally, we show that the iterates converge. We first bound $F^{(t)}-F^{(\infty)}$ uniformly on $\mathcal{B}$. Fix $\theta\in\mathcal{B}$. Write $w_k'\triangleq w_k^{(\infty)}(\theta)$ and $T'\triangleq T^{(\infty)}(\theta)=N_0+N\sum_k w_k'$ for the weights and the denominator at tempering $\beta^\infty$, and $F'\triangleq F^{(\infty)}(\theta)$. Multiplying each M-step by its denominator and subtracting gives
\begin{equation*}
    T\,F^{(t)}(\theta)-T'F' = N\sum_k(w_k-w_k')\,\hat\theta_k.
\end{equation*}
Since $T-T'=N\sum_k(w_k-w_k')$, subtracting $(T-T')F'$ from both sides gives the identity
\begin{equation*}
    F^{(t)}(\theta)-F' = \frac{N}{T}\sum_k(w_k-w_k')\bigl(\hat\theta_k-F'\bigr).
\end{equation*}
This is the computation in the proof of Lemma~\ref{lem:fs_decomp}, with $w_k'$, $T'$ and $F'$ in the roles of $\mathbf{1}_R(k)$, $T_R$ and $\theta^\star$. Here $T\ge N_0$, and $\|\hat\theta_k-F'\|\le\bar D_{\mathcal{B}}$ by the distance bound because $F'\in\mathcal{B}$. The weights $w_k$ and $w_k'$ have sigmoid inputs that differ by $(\beta_k^{(t)}-\beta^\infty)s_k(\theta)$, and $\beta_k^{(t)}-\beta^\infty=-e^{-\nu t}\beta^\infty$. The sigmoid is $\tfrac14$-Lipschitz, and the bound $|\bar s_k(\theta)-\rho_k|\le\Delta^*/2$ from the proof of Lemma~\ref{lem:fs_invariance} together with the triangle inequality gives $|s_k(\theta)|=N|\bar s_k(\theta)|\le N(|\rho_k|+\Delta^*/2)$. Hence
\begin{equation*}
    |w_k-w_k'| \le \tfrac14\,\bigl|\beta_k^{(t)}-\beta^\infty\bigr|\,|s_k(\theta)| \le \tfrac14\,e^{-\nu t}\beta^\infty N\bigl(|\rho_k|+\Delta^*/2\bigr).
\end{equation*}
Inserting $T\ge N_0$, $\|\hat\theta_k-F'\|\le\bar D_{\mathcal{B}}$ and this bound into the identity, and using $\tfrac{N^2}{N_0}\beta^\infty=N^{3/2}/\sqrt{dN_0}$,
\begin{equation*}
    \bigl\|F^{(t)}(\theta)-F^{(\infty)}(\theta)\bigr\| \le \frac{N}{N_0}\,\bar D_{\mathcal{B}}\sum_k|w_k-w_k'| \le C\,e^{-\nu t},
    \qquad
    C \triangleq \frac{N^{3/2}\,\bar D_{\mathcal{B}}}{4\sqrt{dN_0}}\sum_k\Bigl(|\rho_k|+\frac{\Delta^*}{2}\Bigr).
\end{equation*}
The constant $C$ is deterministic and depends on neither $\theta$ nor $t$. As $\theta\in\mathcal{B}$ and $t\ge t_0$ were arbitrary, $\sup_{\theta\in\mathcal{B}}\|F^{(t)}(\theta)-F^{(\infty)}(\theta)\|\le C\,e^{-\nu t}$ for every $t\ge t_0$.

Now let $e_t\triangleq\|\theta^{(t)}-\theta^{(\infty)}\|$. For $t\ge t_0\ge1$, Lemma~\ref{lem:fs_invariance} gives $\theta^{(t)}\in\mathcal{B}$. Using $\theta^{(t+1)}=F^{(t)}(\theta^{(t)})$ and $\theta^{(\infty)}=F^{(\infty)}(\theta^{(\infty)})$,
\begin{equation*}
    \theta^{(t+1)}-\theta^{(\infty)} = \bigl[F^{(t)}(\theta^{(t)})-F^{(\infty)}(\theta^{(t)})\bigr]+\bigl[F^{(\infty)}(\theta^{(t)})-F^{(\infty)}(\theta^{(\infty)})\bigr].
\end{equation*}
The uniform bound above controls the first bracket and the contraction of $F^{(\infty)}$ controls the second, so, for every $t\ge t_0$,
\begin{equation*}
    e_{t+1} \le \kappa\,e_t + C\,e^{-\nu t}.
\end{equation*}
Induction on $t$ starting from $t_0$ gives, for $t\ge t_0$,
\begin{equation*}
    e_t \le \kappa^{t-t_0}e_{t_0} + C\sum_{j=t_0}^{t-1}\kappa^{t-1-j}e^{-\nu j}.
\end{equation*}
Let $r\triangleq\max(\kappa,e^{-\nu})$, which is $<1$ because $\kappa\le\tfrac12$ and $\nu>0$, and fix any $r'\in(r,1)$. Since $e^{-\nu}<r'$, each term satisfies $\kappa^{t-1-j}e^{-\nu j}=r'^{t-1}(\kappa/r')^{t-1-j}(e^{-\nu}/r')^j\le r'^{t-1}(\kappa/r')^{t-1-j}$. Summing the geometric series in $\kappa/r'<1$ bounds the sum by $r'^{t-1}/(1-\kappa/r')=r'^t/(r'-\kappa)$. Further, $\kappa^{t-t_0}\le r'^{t-t_0}$, and $e_{t_0}\le 2c_1\Delta^*/L$ because $\theta^{(t_0)}$ and $\theta^{(\infty)}$ both lie in the ball $\mathcal{B}$ of radius $c_1\Delta^*/L$. Hence, for every $t\ge t_0$,
\begin{equation*}
    \|\theta^{(t)}-\theta^{(\infty)}\| \le \Bigl(\frac{2c_1\Delta^*}{L\,r'^{t_0}}+\frac{C}{r'-\kappa}\Bigr)\,r'^t.
\end{equation*}
So $\theta^{(t)}\to\theta^{(\infty)}$ geometrically, at every rate $r'\in(\max(\kappa,e^{-\nu}),1)$.
\end{proof}

\subsubsection{Proof of Theorem~\ref{thm:fs_main}}

\renewcommand{\formalthmname}{Theorem~\ref{thm:fs_main}}
\begin{formalthm}[formal statement]
Under Assumptions~\ref{ass:gen_model}--\ref{ass:sample_size}, the iterates of Algorithms~\ref{alg:lip_em} and~\ref{alg:lip_em_approx}, which both reduce to \eqref{eq:prox_update} in this setting, converge on the good event $\mathcal{G}$ of \eqref{eq:fs_good_event}, which satisfies $\mathbb{P}(\mathcal{G}^c)\le\delta_{\mathrm{conc}}=O(Ke^{-\Omega(N)})$ by \eqref{eq:fs_good_prob}, and any measurable selection $\theta^{(\infty)}$ of the limit points satisfies \eqref{eq:fs_main}.
\end{formalthm}

\begin{proof}[Proof of Theorem~\ref{thm:fs_main}]
We work on the good event $\mathcal{G}$ of \eqref{eq:fs_good_event}, which by \eqref{eq:fs_good_prob} satisfies $\mathbb{P}(\mathcal{G}^c)\le\delta_{\mathrm{conc}}$. On $\mathcal{G}$, Lemmas~\ref{lem:fs_basin_entry}--\ref{lem:fs_contraction} (Lemma~\ref{lem:fs_contraction} under the contraction clause $N\ge N^\dagger$ of Assumption~\ref{ass:sample_size}) show that the iterates enter the basin $\mathcal{B}$ at $t=1$, remain in $\mathcal{B}$, and converge to a fixed point $\theta^{(\infty)}\in \mathcal{B}$. Off $\mathcal{G}$, every M-step iterate from $t\ge1$ is a convex combination of $\{\hat\theta_0,\hat\theta_1,\ldots,\hat\theta_K\}$, so the sequence is bounded and has limit points. By convention, $\theta^{(\infty)}$ denotes any measurable selection of them (one exists, since the set of limit points is a nonempty random closed set). Split
\begin{equation}
\label{eq:fs_main_decomp}
    \mathbb{E}\|\theta^{(\infty)} - \theta_0\|^2
    =
    \mathbb{E}\bigl[\|\theta^{(\infty)} - \theta_0\|^2\mathbf{1}_{\mathcal{G}}\bigr]
    +
    \mathbb{E}\bigl[\|\theta^{(\infty)} - \theta_0\|^2\mathbf{1}_{\mathcal{G}^c}\bigr].
\end{equation}
For the good event term, set $X\triangleq \theta^\star-\theta_0$ and $Y\triangleq \theta^{(\infty)}-\theta^\star$. Expanding the square and applying Cauchy--Schwarz to the cross term gives
\begin{align*}
    \mathbb{E}\bigl[\|\theta^{(\infty)} - \theta_0\|^2\mathbf{1}_{\mathcal{G}}\bigr]
    &=
    \mathbb{E}\bigl[\|X\|^2\mathbf{1}_{\mathcal{G}}\bigr]
    + 2\,\mathbb{E}\bigl[\langle X,Y\rangle\mathbf{1}_{\mathcal{G}}\bigr]
    + \mathbb{E}\bigl[\|Y\|^2\mathbf{1}_{\mathcal{G}}\bigr] \\
    &\le
    \mathrm{MSE}_\star
    + 2\sqrt{\mathrm{MSE}_\star}\,
      \sqrt{\mathbb{E}\bigl[\|Y\|^2\mathbf{1}_{\mathcal{G}}\bigr]}
    + \mathbb{E}\bigl[\|Y\|^2\mathbf{1}_{\mathcal{G}}\bigr].
\end{align*}
It remains to bound $\mathbb{E}[\|Y\|^2\mathbf{1}_{\mathcal{G}}]$. Apply the oracle decomposition at the fixed point: since $\theta^{(\infty)} = F^{(\infty)}(\theta^{(\infty)})$, the algebraic identity \eqref{eq:fs_decomp_proof} holds with $\gamma_k^{(\infty)}\triangleq w_k^{(\infty)}(\theta^{(\infty)})-\mathbf{1}_R(k)$, where $w_k^{(\infty)}$ denotes the weights at tempering $\beta^\infty$ as in the proof of Lemma~\ref{lem:fs_contraction}. On $\mathcal{G}$ the weights of relevant sources are at least their priors (Lemma~\ref{lem:fs_invariance}), so $T^{(\infty)}(\theta^{(\infty)})\ge N_0+NP_R \ge NP_R$, and the bound \eqref{eq:fs_perstep_sharp} gives, pointwise on $\mathcal{G}$,
\begin{equation*}
    \|\theta^{(\infty)} - \theta^\star\|^2 \;\le\; \frac{K}{P_R^{2}}\sum_{k=1}^K (\gamma_k^{(\infty)})^2\,\|\hat\theta_k - \theta^\star\|^2.
\end{equation*}
On $\mathcal{G}$ the fixed point lies in $\mathcal{B}$ and $\mathcal{G}\subseteq\mathcal{E}_{\mathrm{unif}}$, so Theorem~\ref{thm:fs_weight}(i) applies at $\theta^{(\infty)}$ with tempering $\beta^\infty$: the signed sigmoid input is at least $c_2\beta^\infty N\Delta^*+b_k\ge c_2\mu-B_-$. Hence, pointwise on $\mathcal{G}$,
\begin{equation*}
    |\gamma_k^{(\infty)}| \;\le\; G_\infty \;\triangleq\; e^{-c_2\mu+B_-}.
\end{equation*}
Let $M_k^2\triangleq\mathbb{E}\|\hat\theta_k-\theta^\star\|^2$. Since the configuration is fixed (Assumption~\ref{ass:gen_model}), $\mathbb{E}\|\hat\theta_k-\theta_0\|^2=\|\theta_k-\theta_0\|^2+d\sigma^2/N\le\Delta_{\max}^2+d\sigma^2/N$. Applying $\|a+b\|^2\le2\|a\|^2+2\|b\|^2$ to $\hat\theta_k-\theta^\star=(\hat\theta_k-\theta_0)+(\theta_0-\theta^\star)$ then gives $M_k^2\le2(\Delta_{\max}^2+d\sigma^2/N)+2\,\mathrm{MSE}_\star$. As $|\gamma_k^{(\infty)}|\le G_\infty$ on $\mathcal{G}$,
\begin{equation*}
    \mathbb{E}\bigl[(\gamma_k^{(\infty)})^2\,\|\hat\theta_k-\theta^\star\|^2\mathbf{1}_{\mathcal{G}}\bigr]
    \le
    G_\infty^2\,M_k^2.
\end{equation*}
Summing over the $K$ sources, and using $\mathrm{MSE}_\star\le d\sigma^2/N$ (as $|R|\ge1$) so that $M_k^2\le4(\Delta_{\max}^2+d\sigma^2/N)$, we obtain
\begin{equation*}
    \mathbb{E}\bigl[\|\theta^{(\infty)} - \theta^\star\|^2\mathbf{1}_{\mathcal{G}}\bigr]
    \le
    \frac{K}{P_R^2}\sum_k \mathbb{E}\bigl[(\gamma_k^{(\infty)})^2\|\hat\theta_k-\theta^\star\|^2\mathbf{1}_{\mathcal{G}}\bigr]
    \le
    C_1\,e^{-2c_2\mu},
\end{equation*}
where $C_1 = O\bigl(K^2 P_R^{-2}(\Delta_{\max}^2+d\sigma^2/N)e^{2B_-}\bigr)$.
Substituting this bound into the expansion above gives
\begin{equation*}
    \mathbb{E}\bigl[\|\theta^{(\infty)} - \theta_0\|^2\mathbf{1}_{\mathcal{G}}\bigr]
    \le
    \mathrm{MSE}_\star
    + 2\sqrt{\mathrm{MSE}_\star C_1}\,e^{-c_2\mu}
    + C_1 e^{-2c_2\mu}
    =
    \frac{d\sigma^2}{N_0+N|R|}
    + O(e^{-\Omega(\mu)}).
\end{equation*}

It remains to bound the $\mathcal{G}^c$ term of \eqref{eq:fs_main_decomp}. Let
\[
    D_{\mathrm{hull}}\triangleq \max\Bigl(\|\hat\theta_0-\theta_0\|,\max_k\|\hat\theta_k-\theta_0\|\Bigr).
\]
Every limit point off $\mathcal{G}$ belongs to the convex hull of the MLEs, so $\|\theta^{(\infty)}-\theta_0\|\le D_{\mathrm{hull}}$. By Cauchy--Schwarz and \eqref{eq:fs_good_prob},
\begin{equation*}
    \mathbb{E}\bigl[\|\theta^{(\infty)}-\theta_0\|^2\mathbf{1}_{\mathcal{G}^c}\bigr]
    \le
    \bigl(\mathbb{E}D_{\mathrm{hull}}^4\bigr)^{1/2}\sqrt{\delta_{\mathrm{conc}}}.
\end{equation*}
Since $D_{\mathrm{hull}}^2\le\|\hat\theta_0-\theta_0\|^2+\sum_k\|\hat\theta_k-\theta_0\|^2$, Minkowski's inequality and Gaussian fourth moments give
\[
    \bigl(\mathbb{E}D_{\mathrm{hull}}^4\bigr)^{1/2}
    \le
    C_{\mathrm{tail}}
    =
    O\bigl(d\sigma^2/N_0 + K(\Delta_{\max}^2+d\sigma^2/N)\bigr).
\]
The $\mathcal{G}^c$ term is therefore at most $C_{\mathrm{tail}}\sqrt{\delta_{\mathrm{conc}}}=O(e^{-\Omega(N)})$, the third term of \eqref{eq:fs_main}. Substituting the two bounds into \eqref{eq:fs_main_decomp} yields \eqref{eq:fs_main}.
\end{proof}

\subsubsection{Extension: random irrelevant source locations}

The corollary below drops the assumption of Theorem~\ref{thm:fs_main} that the irrelevant locations are fixed (Assumption~\ref{ass:gen_model}) and instead draws them from the null component $\phi_{\mathrm{null}}$ of Fig.~\ref{fig:dag}. The E-step still uses the fixed null density $q$, which need not match the distribution of the irrelevant sources.

\begin{corollary}[Expected error under random irrelevant source locations]
\label{cor:fs_random_config}
Suppose the irrelevant locations are random: $\theta_k\sim\phi_{\mathrm{null}}$ i.i.d.\ for $k\in\bar R$, with $\phi_{\mathrm{null}}$ supported in $\{\|\theta-\theta_0\|\le R_\phi\}$, second moment $\bar D^2\triangleq\mathbb{E}_{\phi_{\mathrm{null}}}\|\theta-\theta_0\|^2$, and, for some $\alpha\le\tfrac{1}{2K}$, a separation tail
\begin{equation}
\label{eq:fs_separation}
    \mathbb{P}_{\theta\sim\phi_{\mathrm{null}}}\bigl(\|\theta-\theta_0\|<r_{\mathrm{sep}}\bigr)\;\le\;\alpha.
\end{equation}
Keep Assumption~\ref{ass:regularity}, with constants uniform over the support, and the uniform concentration and basin entry clauses of Assumption~\ref{ass:sample_size}. Replace its contraction clause by $N\ge N^\dagger(R_\phi)$, where $N^\dagger(R_\phi)$ is the threshold of Lemma~\ref{lem:fs_contraction} with $\Delta_{\max}$ replaced by $R_\phi$. Finally, impose three conditions tying the null $q$ and the prior $\pi$ to $\theta_0$ and to the separated support $\mathcal{A}_{\mathrm{sep}}\triangleq\{\theta\in\operatorname{supp}\phi_{\mathrm{null}}:r_{\mathrm{sep}}\le\|\theta-\theta_0\|\le R_\phi\}$, where $\operatorname{supp}$ denotes the closed support:
\begin{enumerate}[label=(\roman*),leftmargin=*]
\item \emph{uniform margin:} there is $\Delta^*>0$ with
$\sup_{\theta\in\mathcal{A}_{\mathrm{sep}}}\bigl[\mathrm{KL}(\mathcal{N}(\theta,\sigma^2 I)\,\|\,q)-\tfrac{\|\theta-\theta_0\|^2}{2\sigma^2}\bigr]\le-\Delta^*$,
so every irrelevant source with $\theta_k\in\mathcal{A}_{\mathrm{sep}}$ has $\rho_k(\theta_0)=\mathrm{KL}(\mathcal{N}(\theta_k,\sigma^2 I)\,\|\,q)-\|\theta_k-\theta_0\|^2/(2\sigma^2)\le-\Delta^*$.
\item \emph{uniform calibration:} $\dfrac{|\bar R|\,(\max_{k\in\bar R}\pi_k)\,R_\phi}{\sum_{k\in R}\pi_k}\le\dfrac{c_1\Delta^*}{2L}$, so the calibration bound \eqref{eq:lip_calibration} holds for \emph{every} configuration in the support (using $\|\theta_k-\theta_0\|\le R_\phi$).
\item \emph{relevant margin:} $\mathrm{KL}(\mathcal{N}(\theta_0,\sigma^2 I)\,\|\,q)\ge\Delta^*$. This is the relevant half of Assumption~\ref{ass:margin}, because relevant sources still have $\theta_k=\theta_0$ and hence $\rho_k(\theta_0)=\mathrm{KL}(\mathcal{N}(\theta_0,\sigma^2 I)\,\|\,q)$.
\end{enumerate}
(For a Gaussian null $q=\mathcal{N}(c,\sigma^2 I)$ with $c\neq\theta_0$, a source at $\theta_k$ has $\rho_k(\theta_0)=(\|\theta_k-c\|^2-\|\theta_k-\theta_0\|^2)/(2\sigma^2)$, so (i) and (iii) hold with $\Delta^*=\min\{r_{\mathrm{sep}}^2-r_c^2,\|c-\theta_0\|^2\}/(2\sigma^2)$ whenever $\mathcal{A}_{\mathrm{sep}}\subseteq\{\theta:\|\theta-c\|\le r_c\}$ for some $0\le r_c<r_{\mathrm{sep}}$.)
Then the MSE averaged over both the configuration and the data satisfies
\begin{equation}
\label{eq:fs_random_config}
\begin{split}
    \mathbb{E}\bigl\|\theta^{(\infty)}-\theta_0\bigr\|^2
    \;\le{}&
    \frac{d\sigma^2}{N_0+N|R|}
    \;+\;
    O\!\bigl(e^{-\Omega(\mu)}\bigr)\\
    &+\;
    \underbrace{O\!\bigl(K\alpha\,r_{\mathrm{sep}}^2+K^2\alpha(\bar D^2+d\sigma^2/N)+K\alpha\,d\sigma^2/N_0\bigr)}_{\text{separation failure}}
    \;+\;
    O\!\bigl(e^{-\Omega(N)}\bigr),
\end{split}
\end{equation}
where $\mu=\Delta^*\sqrt{NN_0/d}$ and the residual $O(e^{-\Omega(\mu)})$ carries the constant $\bar C_1=O\bigl(K^2 P_R^{-2}(K\bar D^2+d\sigma^2/N)e^{2B_-}\bigr)$.
\end{corollary}

\begin{proof}
Let $\mathcal{S}\triangleq\{\|\theta_k-\theta_0\|\ge r_{\mathrm{sep}}\ \text{for all }k\in\bar R\}$, so $\mathbb{P}(\mathcal{S})\ge1-K\alpha\ge\tfrac12$ by \eqref{eq:fs_separation}, a union bound and $\alpha\le\tfrac{1}{2K}$. By the law of total expectation,
\begin{equation*}
    \mathbb{E}\|\theta^{(\infty)}-\theta_0\|^2
    =\mathbb{E}\bigl[\|\theta^{(\infty)}-\theta_0\|^2\mathbf{1}_{\mathcal{S}}\bigr]
    +\mathbb{E}\bigl[\|\theta^{(\infty)}-\theta_0\|^2\mathbf{1}_{\mathcal{S}^c}\bigr].
\end{equation*}
\emph{On $\mathcal{S}$}, every irrelevant location is at distance at least $r_{\mathrm{sep}}$ from $\theta_0$ and, since $\phi_{\mathrm{null}}(\operatorname{supp}\phi_{\mathrm{null}})=1$, almost surely lies in $\mathcal{A}_{\mathrm{sep}}$, so hypotheses (i)--(iii) supply the identifiability margin (Assumption~\ref{ass:margin}) and the calibration (Assumption~\ref{ass:lip_calibration}). Since, in addition, $\Delta_{\max}\le R_\phi$ and $N\ge N^\dagger(R_\phi)\ge N^\dagger(\Delta_{\max})$ ($N^\dagger$ is increasing in $\Delta_{\max}$), Theorem~\ref{thm:fs_main} applies to every such configuration. The bound of Theorem~\ref{thm:fs_main} is an expectation over the data for the given configuration, and it depends on the configuration only through $\Delta_{\max}^2$, which enters linearly in the weight error constant $C_1$ and in the prefactor $C_{\mathrm{tail}}=O(d\sigma^2/N_0+K(\Delta_{\max}^2+d\sigma^2/N))$ of the concentration failure term. It remains to average this bound over the configuration on $\mathcal{S}$. Since $\mathbb{E}[\Delta_{\max}^2\mid\mathcal{S}]\le\mathbb{E}[\sum_{k\in\bar R}\|\theta_k-\theta_0\|^2]/\mathbb{P}(\mathcal{S})\le 2K\bar D^2$ (the factor $2$ from $\mathbb{P}(\mathcal{S})\ge\tfrac12$), the oracle rate $d\sigma^2/(N_0+N|R|)$ carries over, the weight error residual averages to $O(e^{-\Omega(\mu)})$ with constant $\bar C_1$ (for its term $2\sqrt{\mathrm{MSE}_\star C_1}\,e^{-c_2\mu}$, Jensen's inequality gives $\mathbb{E}\sqrt{C_1}\le\sqrt{\mathbb{E}C_1}$), and the concentration failure term stays $O(e^{-\Omega(N)})$, with a prefactor that is finite because $\bar D^2$ is finite.

\emph{On $\mathcal{S}^c$} ($\mathbb{P}(\mathcal{S}^c)\le K\alpha$), every M-step iterate from $t\ge1$ is a convex combination of $\{\hat\theta_0,\hat\theta_1,\dots,\hat\theta_K\}$, so, as in the proof of Theorem~\ref{thm:fs_main}, any limit point obeys
\begin{equation*}
    \|\theta^{(\infty)}-\theta_0\|^2\le D_{\mathrm{hull}}^2\le\|\hat\theta_0-\theta_0\|^2+\textstyle\sum_k\|\hat\theta_k-\theta_0\|^2.
\end{equation*}
Since $\hat\theta_0$ is independent of the configuration, the target term is $\mathbb{E}[\|\hat\theta_0-\theta_0\|^2\mathbf{1}_{\mathcal{S}^c}]=\tfrac{d\sigma^2}{N_0}\,\mathbb{P}(\mathcal{S}^c)\le K\alpha\,d\sigma^2/N_0$, and likewise each relevant source ($\theta_k=\theta_0$) contributes $\mathbb{E}[\|\hat\theta_k-\theta_0\|^2\mathbf{1}_{\mathcal{S}^c}]=\tfrac{d\sigma^2}{N}\,\mathbb{P}(\mathcal{S}^c)\le K\alpha\,d\sigma^2/N$. For each irrelevant source $k\in\bar R$, the event $\mathcal{S}^c$ is the union of $\{\|\theta_k-\theta_0\|<r_{\mathrm{sep}}\}$, which has probability at most $\alpha$, and the event that $\|\theta_j-\theta_0\|<r_{\mathrm{sep}}$ for some $j\in\bar R\setminus\{k\}$, which has probability at most $K\alpha$ by a union bound. On the first event, $\|\hat\theta_k-\theta_0\|^2\le2r_{\mathrm{sep}}^2+2\|\hat\theta_k-\theta_k\|^2$, where $\hat\theta_k-\theta_k$ is independent of $\theta_k$ and has $\mathbb{E}\|\hat\theta_k-\theta_k\|^2=d\sigma^2/N$. The second event is independent of $\hat\theta_k$, and $\mathbb{E}\|\hat\theta_k-\theta_0\|^2=\bar D^2+d\sigma^2/N$. Hence
\begin{equation*}
    \mathbb{E}\bigl[\|\hat\theta_k-\theta_0\|^2\mathbf{1}_{\mathcal{S}^c}\bigr]
    \;\le\;
    (2r_{\mathrm{sep}}^2+2d\sigma^2/N)\,\alpha
    +(\bar D^2+d\sigma^2/N)\,K\alpha .
\end{equation*}
Summing the target term and the $K$ source terms gives the separation failure term, and adding the bounds on $\mathcal{S}$ and $\mathcal{S}^c$ yields \eqref{eq:fs_random_config}.
\end{proof}

\subsection{Proof of Theorem~\ref{thm:asymptotic_consistency}}
\label{sec:proof_asymptotic_consistency}

\renewcommand{\formalthmname}{Theorem~\ref{thm:asymptotic_consistency}}
\begin{formalthm}[formal statement]
Fix any finite, possibly data-dependent, iterate $t$, the source sizes $\{N_k\}$, and any fixed prior $\pi\in(0,1)^K$. Assume that the target MLE is regular ($\hat\theta_0\xrightarrow{p}\theta_0$ and $H_0/N_0\xrightarrow{p}\mathcal{I}(\theta_0)\succ0$), that the source MLEs $\hat\theta_k$ and the matrices $(I+\tau^2H_k)^{-1}H_k$ remain $O_p(1)$ as $N_0\to\infty$, and that the matrices in \eqref{eq:m_step_update} are well defined. Then both M-step updates \eqref{eq:m_step_update} and \eqref{eq:prox_update} satisfy $\theta^{(t+1)} \xrightarrow{p} \theta_0$ as $N_0 \to \infty$.
\end{formalthm}

We prove convergence of both update rules by showing each is a continuous function of inputs that converge to limits that recover $\theta_0$. Both arguments use the weights only through $w^{(t)}\in[0,1]^K$, so they hold verbatim for a data-dependent $t$.

\begin{proof}[Hessian reuse M-step \eqref{eq:m_step_update}]
The regularity conditions give $H_0/N_0 \xrightarrow{p} \mathcal{I}(\theta_0)$ with $\mathcal{I}(\theta_0) \succ 0$. Matrix inversion is continuous on the open set of positive definite symmetric matrices, so by the continuous mapping theorem $(H_0/N_0)^{-1} \xrightarrow{p} \mathcal{I}(\theta_0)^{-1}$. The deterministic identity $H_0^{-1} = N_0^{-1}(H_0/N_0)^{-1}$ combined with Slutsky's theorem (a deterministic null sequence times a sequence converging in probability to a finite limit converges in probability to zero) gives $H_0^{-1} \xrightarrow{p} \mathbf{0}$.

By assumption, $\mathcal{D}_k$ has fixed size $N_k$ independent of $N_0$, so $\hat{\theta}_k$ and $H_k$ are random objects whose distributions do not depend on $N_0$. By the theorem's hypotheses, $C_k \triangleq (I + \tau^2 H_k)^{-1} H_k$ is well defined, has $\|C_k\|_{\mathrm{op}}=O_p(1)$, and has a distribution that does not depend on $N_0$. Since $w_k^{(t)} \in [0,1]$ by construction, submultiplicativity of the operator norm gives
\begin{equation*}
    \big\|\Lambda_k^{(t)}\big\|_{\mathrm{op}} = \big\|w_k^{(t)} H_0^{-1} C_k\big\|_{\mathrm{op}} \leq \|H_0^{-1}\|_{\mathrm{op}} \cdot \|C_k\|_{\mathrm{op}}.
\end{equation*}
By Slutsky's theorem the right-hand side converges to $0$ in probability, so $\Lambda_k^{(t)} \xrightarrow{p} \mathbf{0}$.

The source MLEs $\hat\theta_k$ for $k\geq 1$ have fixed sample size $N_k$ and so are bounded in probability with distributions independent of $N_0$. Combined with $\Lambda_k^{(t)} \xrightarrow{p}\mathbf{0}$ from the previous paragraph, Slutsky's theorem yields $\Lambda_k^{(t)}\hat\theta_k \xrightarrow{p}\mathbf{0}$ and $I + \sum_k \Lambda_k^{(t)} \xrightarrow{p} I$. The mapping $A \mapsto A^{-1}$ is continuous at $A = I$, so by the continuous mapping theorem $(I + \sum_k \Lambda_k^{(t)})^{-1} \xrightarrow{p} I$. Substituting these limits and $\hat\theta_0 \xrightarrow{p} \theta_0$ into the M-step \eqref{eq:m_step_update}:
\begin{equation*}
    \theta^{(t+1)} = \Big(I + \sum_{k=1}^K \Lambda_k^{(t)}\Big)^{-1}\Big(\hat\theta_0 + \sum_{k=1}^K \Lambda_k^{(t)}\hat\theta_k\Big) \xrightarrow{p} I^{-1}\bigl(\theta_0 + 0\bigr) = \theta_0. \qedhere
\end{equation*}
\end{proof}

\begin{proof}[Small-$\tau$ M-step \eqref{eq:prox_update}]
Define $r_k^{(t)} \triangleq w_k^{(t)} N_k / N_0$ for each $k$. Dividing the numerator and denominator of \eqref{eq:prox_update} by $N_0$ gives
\begin{equation}
\label{eq:prox_update_ratio}
    \theta^{(t+1)} = \frac{\hat{\theta}_0 + \sum_{k=1}^K r_k^{(t)}\, \hat{\theta}_k}{1 + \sum_{k=1}^K r_k^{(t)}}.
\end{equation}
Since $w_k^{(t)}\in[0,1]$ a.s.\ and $N_k$ is fixed,
\begin{equation*}
    0 \leq r_k^{(t)} \leq N_k/N_0 \xrightarrow{N_0\to\infty} 0,
\end{equation*}
so the squeeze theorem gives $r_k^{(t)} \xrightarrow{a.s.} 0$, hence in probability. Combined with $\hat{\theta}_0\xrightarrow{p}\theta_0$ and the constancy of $\hat{\theta}_k$ in $N_0$, the continuous mapping theorem applied to \eqref{eq:prox_update_ratio}'s numerator and denominator gives, with the denominator limit $1$ bounded away from zero,
\begin{equation*}
    \theta^{(t+1)} \xrightarrow{p} \frac{\theta_0 + \sum_k 0 \cdot \hat\theta_k}{1 + \sum_k 0} = \theta_0. \qedhere
\end{equation*}
\end{proof}

\subsection{Proof of Theorem~\ref{thm:asymptotic_relevance}}
\label{sec:proof_asymptotic_relevance}

\renewcommand{\formalthmname}{Theorem~\ref{thm:asymptotic_relevance}}
\begin{formalthm}[formal statement]
Assume $\tau=0$ and a fixed null density $q$ with $p(\mathcal{D}_k\mid c_k=0)=\prod_i q(x_i)$, and let $\rho_k$ be the population log-LR of \eqref{eq:fs_rho}. For every source $k$, fix a prior $\pi_k\in(0,1)$ and let $w_k(\theta)$ denote the relevance weight \eqref{eq:approx_e_step} evaluated at $\theta$. Let $p_k$ be the data distribution of source $k$, which need not be Gaussian. Suppose $\ell(\cdot\,;\theta_0)$ is integrable under each $p_k$, and assume the identifiability margin of Assumption~\ref{ass:margin} holds at $\theta_0$ with constant $\Delta^*>0$. Then, for each fixed source $k$, as $N_k\to\infty$,
\(
    w_k(\theta_0)\xrightarrow{p}\mathbf{1}_R(k).
\)
\end{formalthm}

\begin{proof}
The relevance weight \eqref{eq:approx_e_step} evaluated at $\theta_0$ reads
\begin{equation*}
    w_k(\theta_0) = \sigma\left( \log \frac{\mathcal{L}(\theta_0;\mathcal{D}_k)}{p(\mathcal{D}_k \mid c_k = 0)} + \log \frac{\pi_k}{1 - \pi_k} \right),
\end{equation*}
where $\mathcal{L}(\theta_0;\mathcal{D}_k) = \prod_i p(x_i\mid\theta_0)$ is the relevant likelihood (by Remark~\ref{rmk:zero_shift_limit} at $\tau = 0$). The null factor is $p(\mathcal{D}_k \mid c_k=0) = \prod_i q(x_i)$ by the product form assumption of the theorem. Hence the log-ratio is the sum $\sum_{i=1}^{N_k}\log[p(x_i\mid\theta_0)/q(x_i)]$. Define
\begin{equation}
\label{eq:thm3_A}
    A_{N_k} \triangleq \sum_{i=1}^{N_k} \log\frac{p(x_i \mid \theta_0)}{q(x_i)} + \log\frac{\pi_k}{1-\pi_k},
\end{equation}
so $w_k(\theta_0) = \sigma(A_{N_k})$.

By assumption, $\ell(\cdot\,;\theta_0)$ is integrable under $p_k$, with mean $\rho_k(\theta_0)$ as in \eqref{eq:fs_rho}.

Since $\theta_0$ is fixed and the samples $\{x_i^{(k)}\}_{i=1}^{N_k}$ are i.i.d.\ under $p_k$, the summands $\ell(x_i^{(k)};\theta_0)$ of \eqref{eq:thm3_A} are i.i.d.\ with mean $\rho_k(\theta_0)$. The strong law of large numbers gives
\begin{equation}
\label{eq:thm3_partial_sum}
    \sum_{i=1}^{N_k} \log\frac{p(x_i\mid\theta_0)}{q(x_i)} = N_k\bigl(\rho_k(\theta_0) + \delta_{N_k}\bigr) \quad \text{a.s.,}
\end{equation}
where $\delta_{N_k}\xrightarrow{a.s.} 0$.

Substituting \eqref{eq:thm3_partial_sum} into \eqref{eq:thm3_A},
\begin{equation*}
    A_{N_k} = N_k\bigl(\rho_k(\theta_0) + \delta_{N_k}\bigr) + \log\frac{\pi_k}{1-\pi_k} \quad \text{a.s.}
\end{equation*}
The last term is a finite constant under $\pi_k\in(0,1)$. Since $\delta_{N_k}\xrightarrow{a.s.}0$, eventually $|\delta_{N_k}|<\Delta^*/2$ a.s. The identifiability margin (Assumption~\ref{ass:margin}) gives $(2\mathbf{1}_R(k)-1)\rho_k(\theta_0)\ge\Delta^*$, so eventually $(2\mathbf{1}_R(k)-1)(\rho_k(\theta_0)+\delta_{N_k})\ge\Delta^*/2$ a.s. Multiplying by $N_k\to\infty$ gives $A_{N_k}\to+\infty$ a.s.\ when $k\in R$, and $A_{N_k}\to-\infty$ a.s.\ when $k\in\bar R$.

The sigmoid extends continuously to the extended real line via $\sigma(\pm\infty)\in\{0,1\}$, so
\begin{equation*}
    w_k(\theta_0) = \sigma(A_{N_k}) \xrightarrow{a.s.} \mathbf{1}_R(k).
\end{equation*}
Almost sure convergence implies convergence in probability.

\emph{Tempered weights.} For the tempered weight \eqref{eq:tempering} at a fixed iteration $t\ge1$, the sigmoid argument becomes $\beta_k^{(t)}N_k(\rho_k(\theta_0)+\delta_{N_k})+\log\frac{\pi_k}{1-\pi_k}$ with $\beta_k^{(t)}=(1-e^{-\nu t})/\varepsilon_k$. The tempering scale $\varepsilon_k^2=dN_k/N_0$ of Algorithms~\ref{alg:lip_em_approx} and~\ref{alg:lip_em_nn} (Appendix~\ref{sec:bayesian_tempering}) satisfies $\varepsilon_k^2=O_p(N_k)$ because $N_0$ is fixed. The scale $\varepsilon_k^2=\mathrm{Tr}(H_0^{+}H_k)$ of Algorithm~\ref{alg:lip_em} satisfies $\varepsilon_k^2\le\|H_0^{+}\|_{\mathrm{op}}\mathrm{Tr}(H_k)$ because $H_k\succeq0$, so it does too whenever $H_k=O_p(N_k)$, as under standard regularity conditions. Whenever $\varepsilon_k^2=O_p(N_k)$, we have $\beta_k^{(t)}N_k\to\infty$ in probability, so the argument again diverges with the sign of $\rho_k(\theta_0)$ and $w_k(\theta_0)\xrightarrow{p}\mathbf{1}_R(k)$. At $t=0$, $\beta_k^{(0)}=0$ and the weight equals $\pi_k$.
\end{proof}

\section{Experimental Results}

\subsection{Detailed Experimental Setup}
\label{sec:appendix_hyperparameter}

The three experiments of Sec.~\ref{sec:experiments} share the same LIP construction protocol and EM convergence criterion. We use Claude Opus 4.7 as a local Claude Code agent and Google's \texttt{gemini-3-flash-preview} model (temperature $0$, JSON-constrained outputs) as the LLM judges. The LIP is fitted by minimizing \eqref{eq:LIP_opt} with L-BFGS and a strong Wolfe line search. EM iterates until $\lVert w^{(t)} - w^{(t-1)}\rVert_\infty \leq 10^{-3}$ for five consecutive iterations with $t \geq \lceil \nu^{-1}\log 2\rceil$ ($14$ for $\nu = 0.05$ and $7$ for $\nu = 0.1$), or until the iteration cap of the experiment is reached. The lower bound on $t$ keeps EM from stopping before the tempering factor $1-e^{-\nu t}$ reaches $1/2$. All experiments use random seed $42$ unless otherwise stated. All hyperparameters were set before the experiments and were not tuned on target data.

\subsubsection{Gaussian Estimation}
\label{sec:appendix_gauss}

In 1-d, the target is $\mathcal{N}(2, 5^2)$ and the sources are $\mathcal{N}(2, 5^2)$, $\mathcal{N}(0, 5^2)$ and $\mathcal{N}(-5, 5^2)$. In \mbox{2-d}, the target is $\mathcal{N}\bigl((55, 120)^\intercal, \Sigma_0\bigr)$ with $\Sigma_0 = \bigl(\begin{smallmatrix} 20 & 5 \\ 5 & 20 \end{smallmatrix}\bigr)$, and the sources are $\mathcal{N}\bigl((55, 120)^\intercal, \Sigma_0\bigr)$, $\mathcal{N}\bigl((51, 120)^\intercal, \Sigma_0\bigr)$ and $\mathcal{N}\bigl((49, 120)^\intercal, \Sigma_0\bigr)$. Each source has $N_k = 200$ samples and the target has $N_0 = 4$. Both the mean and the (co)variance are estimated: $\theta$ stacks the mean and the distinct (co)variance entries ($d = 2$ in 1-d, $d = 5$ in 2-d), and EM updates all of them. The LIP is elicited from \texttt{gemini-3-flash-preview} with all $3$ pairwise subgroup queries and fitted with $p_0 = 0.01$ and $\lambda_{\mathrm{LIP}} = 0.01$. EM runs Algorithm~\ref{alg:lip_em} (closed-form M-step \eqref{eq:m_step_update} and $\tau^2$-corrected E-step \eqref{eq:rel_prox}) with $\tau = 10^{-3}$, $\nu = 0.05$, the empirical Bayes null \eqref{eq:eb-null} with self-exclusion, and at most $1000$ iterations.

\subsubsection{C-MAPSS}

\paragraph{Data.} We use the FD001 subset \citep{saxena2008damage}. We treat each of the $100$ engines as one domain. Ten target engines ($\{4, 9, 18, 26, 27, 48, 51, 53, 55, 80\}$) are drawn uniformly without replacement (seed $42$) from a pool of $12$ engines that we identified as having a fast-growing physical core speed. For each target, the remaining $99$ engines serve as the source pool. We predict the HPC physical core speed (sensor 9) from the cycle index. The backbone is a generalized linear model with a natural cubic regression spline basis: $5$ equally spaced interior knots at cycles $\{50,100,150,200,250\}$ in a fixed raw cycle window $[0,300]$, with the truncated power form constrained to remain linear outside the outer knots. We sweep nine cutoffs at which $\{10\%, 20\%, \dots, 90\%\}$ of the target life is observed.

\paragraph{LIP construction.} The dataset paper \citep{saxena2008damage} is uploaded once as the LLM context, and the target description is the simulated desert description quoted in Sec.~\ref{sec:experiments} (also reproduced verbatim in our open-source release). We issue $200$ subgroup queries with $|S_m| \in \{3, 4, 5\}$ sampled uniformly. For LIP-G, each query is sent to \texttt{gemini-3-flash-preview} with the full core speed series of every source in the subgroup as (cycle, rpm) pairs in plain text under anonymous letter labels, using $10$ parallel workers and retries of requests that hit the API rate limit (HTTP error 429). For LIP-C, a local Claude Code agent (Claude Opus 4.7) answers each of the same $200$ subgroups in one session with one source or the null option, and chooses the null option $9$ times. Both LIPs use $p_0 = 0.01$ and $\lambda_{\mathrm{LIP}} = 0.1$. The same responses are reused for all ten target engines: for each target, we drop the queries that contain it and re-index the remaining sources. Both LIPs favor engines whose core speed rises: $96.5\%$ of Gemini's $200$ selections and all of Claude's $191$ non-null selections have a positive slope over the full trajectory, compared with $71\%$ of all $100$ engines.

\paragraph{EM.} EM uses the closed-form M-step on the GLM coefficients with $\tau = 10^{-3}$, $\nu = 0.05$, the $\tau^2$-corrected E-step \eqref{eq:rel_prox}, and the empirical Bayes null \eqref{eq:eb-null} with self-exclusion. The $\tau^2$-corrected E-step is exact given the fitted noise variances $\hat\sigma_k^2$, because the log-likelihood is then quadratic in the coefficients, so \eqref{prox:taylor} holds with equality and $H_k^{(t)} = H_k$. Until the observed part of the target trajectory passes the third interior knot (cycle $150$), the target design is rank-deficient and $H_0$ is singular. We therefore compute the tempering scale of Algorithm~\ref{alg:lip_em} with the pseudo-inverse, $\varepsilon_k^2 = \mathrm{Tr}(H_0^{+}H_k)$, and evaluate the M-step in the precision form of Appendix~\ref{sec:pseudocode}. Seven source engines have at most $150$ cycles, so their designs are also rank-deficient, and we take the least squares solution of minimum norm as their MLE. The outer loop is capped at $1000$ iterations, and the convergence criterion follows the shared protocol. The Pooled and Target Only baselines use ridge regression ($\lambda = 10^4$, intercept unpenalized) so that the truncated power basis remains numerically stable at intermediate cutoffs. Table~\ref{tab:cmapss_full} reports the results at all nine cutoffs.

\begin{table}[ht]
    \centering
    \small
    \caption{C-MAPSS physical core speed RMSE (rpm) at all nine cutoffs: mean (standard error) over the $10$ target engines (lower is better). The $10\%$, $30\%$, $50\%$, $70\%$ and $90\%$ rows are those of Table~\ref{tab:cmapss}. At $10\%$, $30\%$ and $50\%$ observed life, LIP-C and LIP-G are significantly better than the baselines in $14$ of the $18$ paired comparisons (paired $t$-test over the $10$ targets, $p<0.05$). The exceptions are both LIPs against Uniform EM at $10\%$ ($p=0.072$ and $0.144$) and LIP-G against Uniform EM and Target Only at $50\%$ ($p=0.193$ and $0.052$).}
    \label{tab:cmapss_full}
    \begin{tabular}{lccccc}
    \toprule
    Observed life & LIP-C & LIP-G & Uniform EM & Pooled & Target Only\\
    \midrule
    10\% & 14.3\,(2.1) & 15.0\,(3.5) & 21.4\,(3.2) & 33.9\,(2.8) & 43.7\,(2.9) \\
    20\% & 14.2\,(1.6) & 15.2\,(2.5) & 23.5\,(3.0) & 35.8\,(2.9) & 45.3\,(3.1) \\
    30\% & 15.8\,(2.1) & 15.2\,(3.8) & 35.0\,(3.2) & 38.1\,(3.1) & 58.7\,(11.4) \\
    40\% & 16.8\,(2.7) & 16.0\,(4.5) & 33.4\,(3.5) & 41.1\,(3.4) & 90.1\,(16.0) \\
    50\% & 17.7\,(3.7) & 17.9\,(6.5) & 24.8\,(4.6) & 44.7\,(3.7) & 36.6\,(5.3) \\
    60\% & 21.0\,(3.9) & 19.7\,(5.1) & 25.5\,(4.4) & 49.6\,(4.1) & 30.8\,(11.7) \\
    70\% & 20.1\,(3.4) & 16.4\,(3.4) & 15.2\,(2.3) & 56.5\,(4.6) & 27.0\,(6.1) \\
    80\% & 15.3\,(2.4) & 13.4\,(1.9) & 13.1\,(2.2) & 66.7\,(5.5) & 18.1\,(2.7) \\
    90\% & 11.5\,(1.7) & 9.0\,(1.4) & 7.3\,(0.8) & 82.9\,(6.9) & 10.1\,(2.8) \\
    \bottomrule
    \end{tabular}
\end{table}

\subsubsection{MuJoCo Hopper}

\paragraph{Data.} The $K = 10$ sources are replay buffers collected by training soft actor-critic (SAC) \citep{haarnoja2018soft} in Hopper-v5 with gravity $g \in \{1, 2, \dots, 10\}~\mathrm{m/s^2}$ (one buffer per gravity, $\sim$$10^6$ transitions each). The target environment uses Venus-like gravity $g = 8.87~\mathrm{m/s^2}$, and the target dataset is a separate SAC replay buffer of $10^6$ transitions (seed $100$) from which we draw $N_0 \in \{64, 128, 256, 512, 1024, 2048, 4096\}$ transitions. Every cell of Table~\ref{tab:hopper} is repeated for $10$ seeds ($42$--$51$). Seed $s$ sets the target sample $\mathcal{D}_0$, the EM minibatches and the IQL training. The sample $\mathcal{D}_0$ consists of the first $N_0$ transitions of a seed-$s$ random permutation of the target buffer, so the samples are nested in $N_0$. EM fits the dynamics on the non-terminal transitions of $\mathcal{D}_0$.

\paragraph{Pool dynamics ($\theta^{\mathrm{pool}}$).} The dynamics model is a diagonal Gaussian over the change in state between consecutive time steps. Its mean and log-variance are output by a multi-layer perceptron (MLP) with hidden width $512$ and $4$ hidden layers (Linear + LayerNorm + SiLU), with input/output normalization buffers and fixed softplus clamps on the log-variance. We pretrain $\theta^{\mathrm{pool}}$ on the union of all $10$ source replay buffers for $1000$ epochs (batch $1024$, AdamW with weight decay $10^{-4}$, learning rate $3\!\times\!10^{-4}$ with cosine decay to $10^{-5}$, gradient clipping at $1.0$), holding out $20{,}000$ transitions per source for diagnostics. The terminal learning rate $10^{-5}$ matches the EM M-step's constant rate, so the M-step continues training without a jump in the learning rate.

\paragraph{LIP construction.} A technical report passed to the LLM describes the Hopper environment and its gravity-driven dynamics, and the target description is \textit{The hopper is deployed to a planet with Venus-like gravity} (Sec.~\ref{sec:hopper}). Each source dataset is an anonymized file \texttt{source\_NN.csv} of raw replay buffer transitions (state, action and terminal flag per time step). The file numbers are assigned to the gravities at random, so the LLM does not see the gravity values. Gemini receives the first $2{,}500$ rows of each candidate's CSV in its prompt, whereas the Claude Code agent reads the full files with its own tools. Both LIPs are fitted from $50$ subgroup queries with $p_0 = 0.01$ and $\lambda_{\mathrm{LIP}} = 1.0$. The Claude Code agent drew its own $50$ subgroups (seed $730438886$). LIP-C is elicited from the Claude Code local agent (Appendix~\ref{sec:appendix_agent}). LIP-G, the false LIP of Sec.~\ref{sec:hopper}, is elicited from Gemini 3 Flash through its API, with at most $3$ concurrent requests and a $120$-second interval between request launches. Its argmax is at $g = 7~\mathrm{m/s^2}$ instead of $g = 9~\mathrm{m/s^2}$, and it serves as the misspecified prior ablation.

\paragraph{EM.} For neural network dynamics, the closed-form weighted average M-step is meaningless across re-parameterizations, so we set $\tau = 0$ (Remark~\ref{rmk:zero_shift_limit}) and use a gradient-based M-step (Algorithm~\ref{alg:lip_em_nn} in Appendix~\ref{sec:appendix_algorithms}). No neural network Hessian is formed. At iterate $\theta^{(t)}$, the M-step takes $100$ minibatch gradient steps on the weighted negative log-likelihood
\[
    J(\theta) = -\log\mathcal{L}(\theta;\mathcal{D}_0) - \sum_k w_k^{(t)}\log\mathcal{L}(\theta;\mathcal{D}_k),
\]
with batch size $1024$, AdamW (weight decay $10^{-4}$), the constant learning rate $10^{-5}$ given above, and gradient clipping at $1.0$. The outer loop runs at most $100$ iterations. The tempering schedule uses $\nu = 0.1$ and $\beta_k^{(t)} = (1 - e^{-\nu t})/\sqrt{d_{\mathrm{eff}} N_k / N_0}$, where $d_{\mathrm{eff}} = 811{,}030$ is the trainable parameter count of $\theta^{\mathrm{pool}}$ (the $d$ of Algorithm~\ref{alg:lip_em_nn}). This is a parameter count heuristic motivated by the approximation of Appendix~\ref{sec:bayesian_tempering}: if $H_0 = N_0\mathcal{I}$ and $H_k = N_k\mathcal{I}$ for a common invertible per-sample Fisher information $\mathcal{I}$, then $\mathrm{Tr}(H_0^{-1}H_k) = dN_k/N_0$. The null is the parametric null of Appendix~\ref{sec:null-heuristic}, $p(\mathcal{D}_k \mid c_k=0) = \mathcal{L}(\theta^{\mathrm{pool}};\mathcal{D}_k)$, computed once from the source data before the EM loop and held fixed. This replaces the empirical Bayes null \eqref{eq:eb-null}, which is unstable in the over-parameterized regime. The EM weights at convergence support the reading of Table~\ref{tab:hopper} in Sec.~\ref{sec:hopper}. At $N_0=512$, Uniform EM gives the $g=9~\mathrm{m/s^2}$ source a weight above $0.995$ in all $10$ seeds, while LIP-G still gives the $g=7~\mathrm{m/s^2}$ source a weight above $0.999$. Every other source gets a weight near zero. From $N_0=1024$ on, LIP-G also selects $g=9~\mathrm{m/s^2}$ in all $10$ seeds.

\paragraph{Weighted IQL.} After EM, we train implicit Q-learning (IQL) \citep{kostrikov2022offline} on the mixture sampling distribution $p(\text{target}) \propto N_0$, $p(\text{source }k) \propto w_k^{(\infty)} N_k$, where $w_k^{(\infty)}$ is the weight of source $k$ at convergence. The policy is a tanh-squashed Gaussian with hidden width $256$, and the twin Q-networks and the value network have hidden width $256$ and LayerNorm. Training runs for $200{,}000$ steps with batch $256$, Adam at $3\!\times\!10^{-4}$ with cosine decay to zero, target soft-update rate $\tau_{\mathrm{soft}} = 0.005$, discount $\gamma_{\mathrm{disc}} = 0.99$, expectile $0.7$, advantage-weighted regression (AWR) temperature $\beta_{\mathrm{AWR}} = 3.0$, and AWR weight clamp $100$. Each run is scored by its mean return over $200$ evaluation episodes at $g = 8.87~\mathrm{m/s^2}$.

\subsection{Hardware and Training Time}\label{sec:appendix_hardware} Unless stated otherwise, all experiments run on a single workstation with an NVIDIA RTX PRO 6000 Blackwell GPU, an AMD Ryzen 9 9950X3D CPU, and 64\,GB of RAM. The Gaussian and C-MAPSS pipelines complete in seconds on CPU. For Hopper, the replay buffers were collected and the pooled dynamics model was pretrained (in $\sim$$14$ hours) once in advance on an NVIDIA RTX 5090 GPU, and both are reused by every run. Each EM run completes in $0.2$--$6$ minutes and each IQL run in about $10$ minutes when run alone. The full Hopper grid ($7$ values of $N_0$, $5$ methods and $10$ seeds, i.e., $350$ IQL runs) took about $27$ hours of wall-clock time with several runs sharing the GPU.

\subsection{Sensitivity to the Null Likelihood (C-MAPSS)}
\label{sec:appendix_null_sensitivity}
The null enters LIP-EM only through $p(\mathcal{D}_k\mid c_k=0)$. To raise or lower it, we multiply the null likelihood by $e^{c}$ per observation, which subtracts $N_k c$ from the log-LR of source $k$ in \eqref{eq:E-step}, and sweep $c\in\{-10,-3,-1,-0.3,0,0.3,1,3,10\}$ nats per observation on the C-MAPSS pipeline of Appendix~\ref{sec:appendix_hyperparameter}, reusing the cached LLM responses. The level $c=0$ is the empirical Bayes null \eqref{eq:eb-null} and reproduces Table~\ref{tab:cmapss}. It was never tuned, and another level of $c$ gives a lower RMSE in $6$ of the $9$ columns of Table~\ref{tab:null_sweep}. We also replace the empirical Bayes null by the parametric pooled null of Appendix~\ref{sec:null-heuristic}. Table~\ref{tab:null_sweep} reports the three cold-start cutoffs.

\begin{table}[htb]
    \centering
    \small
    \setlength{\tabcolsep}{4pt}
    \caption{C-MAPSS physical core speed RMSE (rpm), mean over the $10$ target engines, at $10\%$, $30\%$ and $50\%$ observed life when the null likelihood is multiplied by $e^{c}$ per observation. The last row uses the parametric pooled null of Appendix~\ref{sec:null-heuristic} instead of the empirical Bayes null. Pooled ($33.9$, $38.1$, $44.7$) and Target Only ($43.7$, $58.7$, $36.6$) do not depend on the null.}
    \label{tab:null_sweep}
    \begin{tabular}{lccccccccc}
    \toprule
    & \multicolumn{3}{c}{LIP-C} & \multicolumn{3}{c}{LIP-G} & \multicolumn{3}{c}{Uniform EM} \\
    \cmidrule(lr){2-4} \cmidrule(lr){5-7} \cmidrule(lr){8-10}
    $c$ (nats/obs) & 10\% & 30\% & 50\% & 10\% & 30\% & 50\% & 10\% & 30\% & 50\% \\
    \midrule
    $-10$  & 32.7 & 17.5 & 34.4 & 32.7 & 15.9 & 35.7 & 32.7 & 39.1 & 38.5 \\
    $-3$   & 28.6 & 17.2 & 30.1 & 28.8 & 16.7 & 32.6 & 29.7 & 36.1 & 31.5 \\
    $-1$   & 16.3 & 16.8 & 19.0 & 14.2 & 16.2 & 20.5 & 30.8 & 35.7 & 39.7 \\
    $-0.3$ & 14.5 & 16.1 & 16.7 & 13.5 & 15.5 & 18.8 & 23.2 & 35.2 & 26.5 \\
    $0$    & 14.3 & 15.8 & 17.7 & 15.0 & 15.2 & 17.9 & 21.4 & 35.0 & 24.8 \\
    $+0.3$ & 14.5 & 15.4 & 18.2 & 16.3 & 14.7 & 18.5 & 23.2 & 34.2 & 26.1 \\
    $+1$   & 16.1 & 15.1 & 19.3 & 17.3 & 14.5 & 19.2 & 26.3 & 34.1 & 26.0 \\
    $+3$   & 16.4 & 14.9 & 22.4 & 16.8 & 14.3 & 24.1 & 31.7 & 35.7 & 24.2 \\
    $+10$  & 27.2 & 15.6 & 20.0 & 27.6 & 14.7 & 25.2 & 40.0 & 35.5 & 25.1 \\
    \midrule
    Pooled null & 13.7 & 17.0 & 17.2 & 13.2 & 16.4 & 20.3 & 18.2 & 35.6 & 21.6 \\
    \bottomrule
    \end{tabular}
\end{table}

LIP-EM stays below Pooled and Target Only in all $54$ cells (two LIPs, three cutoffs, nine null levels), and at or below Uniform EM in $52$ of them. Across all null levels and the pooled null, its RMSE lies in $[13.2, 35.7]$. At 30\% observed life the null level has little effect ($14.3$--$17.5$), and for $c\in[-1,1]$ the RMSE stays in $[13.5, 20.5]$ at all three cutoffs. Lowering the null to $c\le-3$ admits many more sources at 10\% and 50\% observed life and moves EM toward Pooled. At $c=-10$ and 10\% observed life, all three EM variants keep nearly all $99$ sources (mean $\sum_k w_k$ of $98.99$--$99.00$) and reach $32.7$, close to Pooled ($33.9$). Raising the null by the same amount ($c=+3$ and $c=+10$) gives both LIPs a lower RMSE than lowering it at all three cutoffs, and the largest increase over $c=0$ among the raised levels is at 10\% observed life and $c=+10$ ($27.2$ and $27.6$, vs.\ $14.3$ and $15.0$ at $c=0$). This agrees with the advice in Sec.~\ref{sec:E-step} to err on the side of a larger null likelihood. Replacing the empirical Bayes null by the pooled null gives a similar picture and improves all three EM variants at 10\% observed life ($13.7$, $13.2$ and $18.2$, vs.\ $14.3$, $15.0$ and $21.4$ at $c=0$).

\subsection{Shuffled LIP Control (C-MAPSS)}
\label{sec:appendix_shuffled}
To separate the effect of placing prior mass on the right sources from that of prior concentration alone, we randomly permute, for each target engine, the fitted LIP over the $99$ source identities. A permutation preserves the multiset of prior values, and hence its mean, entropy and concentration, and destroys only the mapping between prior mass and sources. We draw $100$ permutations per target (seed $42$), apply each permutation to both the LIP-C and the LIP-G vector, and rerun the C-MAPSS EM of Appendix~\ref{sec:appendix_hyperparameter} with all other settings unchanged (Table~\ref{tab:shuffled}).

\begin{table}[htb]
    \centering
    \small
    \caption{C-MAPSS physical core speed RMSE (rpm) under shuffled LIPs: mean (SE) over the $10$ target engines. Shuffled entries first average the $100$ permutations within each engine and then take the mean and SE across engines.}
    \label{tab:shuffled}
    \begin{tabular}{lccccc}
    \toprule
    Observed life & LIP-C & Shuffled LIP-C & LIP-G & Shuffled LIP-G & Uniform EM \\
    \midrule
    10\% & 14.3\,(2.1) & 28.8\,(2.0) & 15.0\,(3.5) & 27.9\,(1.8) & 21.4\,(3.2) \\
    30\% & 15.8\,(2.1) & 36.3\,(3.1) & 15.2\,(3.8) & 36.1\,(3.0) & 35.0\,(3.2) \\
    50\% & 17.7\,(3.7) & 32.6\,(3.1) & 17.9\,(6.5) & 31.9\,(2.9) & 24.8\,(4.6) \\
    70\% & 20.1\,(3.4) & 20.2\,(2.9) & 16.4\,(3.4) & 19.8\,(2.8) & 15.2\,(2.3) \\
    90\% & 11.5\,(1.7) &  9.3\,(1.2) &  9.0\,(1.4) &  9.0\,(1.2) &  7.3\,(0.8) \\
    \bottomrule
    \end{tabular}
\end{table}

At every cutoff from 10\% to 60\% observed life, including the 20\%, 40\% and 60\% cutoffs that Table~\ref{tab:shuffled} does not show, both real LIPs have a lower mean RMSE than each of the $100$ shuffles (one-sided permutation test, $p=1/101$). Prior concentration alone therefore does not explain the gain. A shuffled LIP is never better than Uniform EM on average at any of the nine cutoffs and is clearly worse at 10\% and 50\% observed life ($27.9$--$28.8$ vs.\ $21.4$ and $31.9$--$32.6$ vs.\ $24.8$, with $99$ or $100$ of the $100$ shuffles worse than Uniform EM), because it puts high prior mass on sources that do not match the target. At 70\% and 90\% observed life the gap between the real and the shuffled LIPs shrinks or reverses as the target data become more informative. At 70\% it disappears for LIP-C ($20.1$ vs.\ $20.2$, with $43$ of the $100$ shuffles doing as well or better) but remains for LIP-G ($16.4$ vs.\ $19.8$, with $2$ of the $100$), and at 90\% it reverses for LIP-C ($11.5$ vs.\ $9.3$, with every shuffle better) and disappears for LIP-G ($9.0$ vs.\ $9.0$). This is consistent with Theorem~\ref{thm:asymptotic_consistency}, under which the target data override any fixed prior as $N_0$ grows.

\subsection{Non-Gaussian Likelihoods and Controlled Prior Corruption}
\label{sec:appendix_nongaussian}
\paragraph{Setup.} We run LIP-EM with exact Bernoulli and Poisson likelihoods. Each replication has $K=20$ sources with $N_k=80$ observations each. Five sources are relevant and share the target parameter $\theta_0$, which is drawn per replication from $\{0.2,0.8\}$ (Bernoulli) or $\{2,4\}$ (Poisson). The other $15$ sources share a null parameter ($0.5$ for Bernoulli and $3$ for Poisson). The E-step uses the exact Bernoulli or Poisson likelihood ratio at $\tau=0$ and the M-step the exact weighted MLE: this is Algorithm~\ref{alg:lip_em_approx} with $d=1$, whose weighted average update is exact for these mean parameters, so no Gaussian approximation or Hessian is involved. We use $\nu=0.05$, at most $500$ iterations and the stopping rule of Appendix~\ref{sec:appendix_hyperparameter} with tolerance $10^{-8}$. All EM variants use the true null model of the irrelevant sources as the null density. We compare Target Only, Pooled, Uniform EM ($p_0=0.01$), EM with a synthetic \emph{aligned} prior that is informative but not an oracle ($0.15$--$0.35$ on the relevant sources and $0.004$--$0.018$ on the others), and EM with a \emph{shuffled} copy of that prior (a random permutation per replication, which preserves its concentration). The metric is the excess test negative log-likelihood (NLL) per target observation, i.e., the KL divergence from the true target distribution to the fitted one, averaged over $1{,}000$ replications (seed $42$, Table~\ref{tab:nongauss}). At small $N_0$ the Target Only MLE can fall on the boundary of the parameter space (e.g., when all target observations are $0$), so it is clipped to $[10^{-6},1-10^{-6}]$ (Bernoulli) or to at least $10^{-8}$ (Poisson), which keeps its divergence finite.

\begin{table}[htb]
    \centering
    \small
    \setlength{\tabcolsep}{4pt}
    \caption{Exact Bernoulli and Poisson likelihoods: excess test negative log-likelihood ($\times10^{3}$), mean (SE) over $1{,}000$ replications (lower is better).}
    \label{tab:nongauss}
    \begin{tabular}{lrccccc}
    \toprule
    Likelihood & $N_0$ & Target Only & Pooled & Uniform EM & Aligned prior & Shuffled prior \\
    \midrule
    Bernoulli & 2  & 1940.2\,(63.9) & 113.4\,(0.4) & 6.5\,(0.6)  & 1.5\,(0.1) & 44.7\,(2.2) \\
              & 8  & 443.2\,(27.2)  & 112.7\,(0.3) & 3.5\,(0.4)  & 1.2\,(0.1) & 10.4\,(1.3) \\
              & 64 & 7.8\,(0.4)     & 105.8\,(0.3) & 1.1\,(0.1)  & 1.1\,(0.0) & 1.1\,(0.0) \\
    \midrule
    Poisson   & 2  & 631.8\,(113.9) & 96.7\,(0.6)  & 15.5\,(5.4) & 1.3\,(0.1) & 41.0\,(1.9) \\
              & 8  & 64.0\,(3.0)    & 95.5\,(0.6)  & 7.2\,(1.6)  & 1.3\,(0.1) & 13.1\,(1.3) \\
              & 64 & 7.7\,(0.4)     & 89.6\,(0.6)  & 1.1\,(0.0)  & 1.0\,(0.0) & 1.0\,(0.0) \\
    \bottomrule
    \end{tabular}
\end{table}

\paragraph{Controlled corruption.} Because relevance is known here, we also start from an oracle prior ($0.99$ on the five relevant sources and $0.01$ on the others) and, at each corruption level, exchange the prior values of that fraction of the five relevant sources with those of randomly chosen irrelevant sources, which again preserves the multiset of prior values (Table~\ref{tab:corruption}).

\begin{table}[htb]
    \centering
    \small
    \caption{Controlled corruption of an oracle prior at $N_0=2$: excess test negative log-likelihood ($\times10^{3}$), mean over $1{,}000$ replications (all SEs $\leq 1.1$). Pooled: $113.4$ (Bernoulli) and $96.7$ (Poisson).}
    \label{tab:corruption}
    \begin{tabular}{lcccccc}
    \toprule
    Corruption & 0\% & 20\% & 40\% & 60\% & 80\% & 100\% \\
    \midrule
    Bernoulli & 1.4 & 9.5 & 34.0 & 72.3 & 123.6 & 189.9 \\
    Poisson   & 1.2 & 7.8 & 27.6 & 60.3 & 106.4 & 166.5 \\
    \bottomrule
    \end{tabular}
\end{table}

\paragraph{Findings.} The same algorithm works without modification for both discrete families. The aligned prior helps most when target data are scarce: at $N_0=2$ its excess NLL is $1.5$ vs.\ $6.5$ for Uniform EM (Bernoulli) and $1.3$ vs.\ $15.5$ (Poisson), and its weights separate relevant from irrelevant sources perfectly (area under the receiver operating characteristic curve, AUC, of $1.00$, vs.\ $0.85$ and $0.82$ for the shuffled prior). By $N_0=64$ the three EM variants of Table~\ref{tab:nongauss} coincide. The shuffled prior, with exactly the same concentration, is worse than Uniform EM at $N_0\in\{2,8\}$ (e.g., $44.7$ vs.\ $6.5$ and $41.0$ vs.\ $15.5$ at $N_0=2$), and the effect vanishes by $N_0=64$. Under controlled corruption, the excess NLL grows monotonically with the corruption level, and from $80\%$ corruption on it exceeds Pooled (Table~\ref{tab:corruption}). Together with Appendix~\ref{sec:appendix_shuffled}, these controls show that the gain comes from assigning prior mass to the relevant sources, not from prior concentration alone. They also show that a misassigned prior can make EM worse than Uniform EM at cold start, and that its influence fades as target evidence accumulates.

\subsection{Sensitivity to the Target Description (C-MAPSS)}
\label{sec:appendix_prompt}
We test how much the C-MAPSS LIP, which favors sources with a rising core speed (Appendix~\ref{sec:appendix_hyperparameter}), depends on the wording of the simulated desert description. To do so, we re-elicit the LIP with Claude Opus 4.7 for the alternatives below and keep everything else fixed: the anonymized source trajectories, the same $200$ subgroups and their order, the attached dataset paper \citep{saxena2008damage}, the LIP fit, the EM of Algorithm~\ref{alg:lip_em}, the spline model, the ten target engines and all hyperparameters. For each description, a fresh Claude Code session at maximum effort answers the $200$ queries using only the trajectories, the dataset paper and the description.
\begin{enumerate}[leftmargin=*,nosep]
    \item \emph{No-diagnostics} (drops the ingestion explanation): \textit{This aircraft is operating in a desert environment. Routine visual inspections confirm severe aerodynamic degradation of the high-pressure compressor.}
    \item \emph{Desert only}: \textit{This aircraft is operating in a desert environment.}
    \item \emph{No-desert wording}: \textit{Inspection records identify severe aerodynamic degradation of the aircraft's high-pressure compressor, consistent with sustained ingestion of fine, dry airborne abrasives at high velocity.}
    \item \emph{Alternative inspection wording}: \textit{Inspection findings for this desert-operated aircraft show severe aerodynamic degradation of the high-pressure compressor, consistent with prolonged ingestion of fine, dry abrasive dust at high velocity.}
    \item \emph{No-information control}: \textit{Maintenance records are available for the target aircraft, but they contain no reliable information about the condition or trajectory of the high-pressure compressor.}
\end{enumerate}

\begin{table}[htb]
    \centering
    \small
    \setlength{\tabcolsep}{3pt}
    \caption{C-MAPSS sensitivity to the target description (Claude Opus 4.7). RMSE: mean (SE) over the ten target engines, averaged over all nine cutoffs (\mbox{10\%--90\%} observed life) or over the six cutoffs from \mbox{10\%} to \mbox{60\%}. $\Delta$: paired difference from the original LIP-C in RMSE over all cutoffs, with a 95\% target engine bootstrap interval (seed $20260802$). $\rho$: mean per-target Spearman correlation of the LIP with the original LIP-C. Nulls: queries (of $200$) answered with the null option.}
    \label{tab:prompt}
    \begin{tabular}{lrcccr}
    \toprule
    Description & Nulls & All cutoffs & 10\%--60\% & $\Delta$ [95\% CI] & $\rho$ \\
    \midrule
    Original (Table~\ref{tab:cmapss_full}) & 9 & 16.30\,(2.08) & 16.64\,(2.37) & -- & 1.000 \\
    No-diagnostics           & 12  & 15.27\,(2.94) & 17.01\,(3.71) & $-1.03$ [$-4.88$, $+3.40$] & 0.904 \\
    Desert only              & 16  & 16.28\,(2.82) & 18.46\,(3.31) & $-0.03$ [$-3.32$, $+3.62$] & 0.880 \\
    No-desert wording        & 18  & 15.87\,(2.26) & 17.82\,(2.58) & $-0.43$ [$-2.59$, $+1.39$] & 0.878 \\
    Alt.\ inspection wording & 8   & 15.87\,(2.31) & 17.82\,(2.65) & $-0.43$ [$-2.68$, $+1.58$] & 0.930 \\
    No-information control   & 200 & 22.12\,(1.38) & 27.25\,(1.59) & $+5.81$ [$+2.77$, $+8.73$] & $-0.038$ \\
    \midrule
    Uniform EM               & --  & 22.12\,(1.38) & 27.25\,(1.59) & $+5.81$ [$+2.77$, $+8.73$] & -- \\
    \bottomrule
    \end{tabular}
\end{table}

The four informative alternatives produce LIPs that are strongly correlated with the original LIP-C ($\rho=0.88$--$0.93$), and all of Claude's non-null selections under them have a positive slope of core speed over the full trajectory. Their RMSE is lower than that of Uniform EM by $5.84$--$6.85$\,rpm averaged over all cutoffs and by $8.79$--$10.24$\,rpm at 10\%--60\% observed life (all paired 95\% bootstrap intervals against Uniform EM exclude zero), and none of them differs significantly from the original LIP-C (Table~\ref{tab:prompt}). When the description explicitly carries no information, Claude returns the null option for all $200$ queries. The fitted LIP then equals the uniform prior $p_0$ up to numerical error (Appendix~\ref{sec:appendix_lip}), so its $\rho$ in Table~\ref{tab:prompt} reflects only this error, and LIP-EM coincides with Uniform EM.

This audit shows what the elicited LIP encodes, but it does not establish that the mapping from a desert description to a rising core speed trajectory is physically unique or invariant to wording.

\subsection{Direct RUL Prediction on N-CMAPSS DS03}
\label{sec:appendix_ds03}
In FD001, an engine reaches its end of life when the first of several operability margins is exhausted \citep{saxena2008damage}, not when its core speed crosses a threshold, so the single sensor task of the C-MAPSS case study cannot simply be relabeled as RUL prediction. We therefore evaluate direct RUL prediction on DS03, a subset of the N-CMAPSS dataset \citep{ariaschao2021aircraft}, whose native flight classes also provide a real, rather than simulated, target description. Since DS03 has only $15$ engines, five of which serve as targets, the sample is too small for firm conclusions, and we therefore keep this study in the appendix.

\paragraph{Setup.} DS03 contains run-to-failure trajectories of $15$ engines in three flight classes, with five short-haul, five medium-haul and five long-haul engines. Each engine is one domain. Each long-haul engine in turn is the target, and the other $14$ engines (four of them long-haul) are anonymized sources. The only target text is a sentence stating the native class label: \textit{The target aircraft belongs to the long-haul flight class}. Source labels are never shown. Each source is shown to the LLM as a table with one row per flight, giving the flight cycle, the flight duration and the $14$ sensor values adjusted for operating conditions. Both LLMs mostly identify long-haul sources by their flight durations. Following the prognostics protocol of \citet{liu2013datalevel}, we adjust the $14$ sensor paths for operating conditions and jointly regress them on raw flight cycles with a natural cubic spline (five equally spaced interior knots in $[0,100]$) and one relevance weight per source shared across sensors. We then map the predicted paths to a composite health index whose fusion weights are fitted on the sources only, and predict the RUL as the time until this index first crosses the failure threshold (the mean terminal health index of the sources), capped at the longest source life. At the nine cutoffs, $10\%,\dots,90\%$ of the target life is observed, but the model works in raw cycles and never sees the normalized lifetime. Target RUL, future target sensors and source labels are never used for fitting. We issue $30$ subgroup queries of size $3$--$5$ over the $15$ engines (instead of the $200$ used for C-MAPSS) and drop, for each target, the queries that contain it ($19$--$24$ remain). LIP-G uses \texttt{gemini-3-flash-preview} at temperature $0$. LIP-C uses $30$ independent calls to Claude Opus 4.7 at maximum effort with no context shared between calls, not the local agent of Appendix~\ref{sec:appendix_agent}. Both LLMs answer $26$ of the $30$ queries correctly (a long-haul source if the subgroup contains one, the null option otherwise), and for every target both fitted LIPs rank the four long-haul sources above the other ten (AUC $1.0$), except LIP-G for one target (AUC $0.9$). The LIP fit, EM and the baselines use the C-MAPSS settings of Appendix~\ref{sec:appendix_hyperparameter}: $p_0=0.01$ and $\lambda_{\mathrm{LIP}}=0.1$ for the LIP, Algorithm~\ref{alg:lip_em} with $\tau=10^{-3}$ and $\nu=0.05$ for EM, and ridge regression with $\lambda=10^4$ for Pooled and Target Only.

\begin{table}[htb]
    \centering
    \small
    \setlength{\tabcolsep}{4pt}
    \caption{N-CMAPSS DS03, long-haul targets: RUL RMSE in flight cycles over the five target engines (engine bootstrap SE). Here LIP-C is elicited by independent per-query calls to Claude Opus 4.7.}
    \label{tab:ds03}
    \begin{tabular}{rccccc}
    \toprule
    Observed life & LIP-C & LIP-G & Uniform EM & Pooled & Target Only \\
    \midrule
    10\% & 5.92\,(1.04) & 5.79\,(1.10) & 8.23\,(2.10) & 32.73\,(2.55) & 32.73\,(2.55) \\
    20\% & 6.71\,(1.70) & 6.49\,(1.69) & 8.46\,(1.72) & 39.10\,(2.14) & 39.10\,(2.14) \\
    30\% & 7.34\,(1.61) & 7.49\,(1.80) & 8.00\,(1.77) & 45.59\,(1.83) & 45.59\,(1.83) \\
    40\% & 6.97\,(1.47) & 7.17\,(1.65) & 7.89\,(1.70) & 49.12\,(2.21) & 38.27\,(9.10) \\
    50\% & 6.48\,(1.33) & 6.61\,(1.52) & 6.60\,(1.37) & 49.15\,(2.22) & 7.54\,(1.20) \\
    60\% & 4.96\,(1.06) & 5.24\,(1.23) & 3.56\,(1.01) & 49.15\,(2.20) & 5.00\,(0.39) \\
    70\% & 2.45\,(0.66) & 2.85\,(0.74) & 1.15\,(0.41) & 48.90\,(2.11) & 7.76\,(3.34) \\
    80\% & 3.40\,(1.10) & 3.47\,(1.12) & 3.33\,(1.19) & 48.27\,(2.03) & 4.11\,(1.05) \\
    90\% & 1.61\,(0.42) & 1.63\,(0.41) & 1.74\,(0.33) & 48.37\,(2.03) & 1.57\,(0.26) \\
    \bottomrule
    \end{tabular}
\end{table}

\paragraph{Results.} Table~\ref{tab:ds03} reports the RUL RMSE with engine bootstrap standard errors ($20{,}000$ resamples of the five targets, seed $42$). The trends are consistent with those of the C-MAPSS case study. The LIPs help most when target data are limited: at 10\%--40\% observed life both have a lower RMSE than Uniform EM, although with five targets the paired 95\% bootstrap intervals of the differences include zero (e.g., $-2.31$ $[-5.73, +2.23]$ for LIP-C at 10\%). Uniform EM catches up once the target path becomes informative, ties with both LIPs at 50\% and is better at 60\% and 70\%. Pooled is severely biased at every cutoff, and Pooled and Target Only coincide at 10\%--30\% because all their predictions hit the RUL cap.

\subsection{Assessment of LLM-Generated Priors via Reasoning Analysis}
\label{sec:appendix_compare_LLM}
As we observed in the experiments, the quality of the LIP generated by the LLM is crucial for cold-start performance. In practice, we suggest that the best way to audit the quality of the LIP is to directly inspect the LLM's reasoning and check its logic. We provide a detailed analysis of the LLM's reasoning in the Hopper experiment, where we observed three significant failure modes in the Gemini 3 Flash model. We also analyzed the reasoning patterns of Claude Opus 4.7 as a local Claude Code agent to understand why it produced a more accurate LIP. 

\subsubsection{Using a Local Agent for Relevance Assessment}
\label{sec:appendix_agent}
For Claude Opus 4.7, we used \textbf{Local Agent Assessment}, where a local agent evaluates the relevance of each source based on the contextual information and provides a more structured analysis. One huge advantage of this approach is that the dataset does not need to be directly uploaded to the LLM, which can be a significant bottleneck in terms of both time and cost. In the Hopper experiment, the $10$ sources add up to $1$\,GB of data, of which we pass only the first $2{,}500$ rows per source to Gemini. Instead of uploading the entire dataset, the local agent can analyze the dataset and extract relevant features or summaries that are then fed into the LLM for relevance assessment.

We now summarize the key steps of the Local Agent approach. The agent (Claude Opus 4.7 in Claude Code) is given filesystem access and a Python execution tool, but no internet access and no exposure to the mapping from sources to gravities.

\paragraph{Step 1: Read context.} The agent reads three input files: a task specification, a technical report describing the Hopper telemetry ($z$, $\dot{z}$, joint angles/velocities, actions, terminal flag at $\Delta t = 0.008$\,s), and a target description identifying the goal as Venus-like gravity ($\sim$$8.87~\mathrm{m/s^2}$).

\paragraph{Step 2: Inspect the source data.} The agent lists the data directory and inspects the column format of the $10$ anonymized CSV files (\texttt{source\_01.csv}, \ldots, \texttt{source\_10.csv}).

\paragraph{Step 3: Estimate gravity per source.} The agent iteratively writes and runs six gravity estimation scripts and compares their performance. The methods are: 
\begin{enumerate}[leftmargin=*]
\item Average $\ddot{z}$ at hop apex (where $\dot{z} \approx 0$ and $z$ is high). \item Histogram mode of $\ddot{z}$ over high-$z$ samples. 
\item Short-window quadratic fit to $z(t)$ in a $\pm 5$-step window around each apex. 
\item Joint-motion-minimal filter, restricting to samples whose joint velocities are small \emph{and} $z$ is high. 
\item Approximate center-of-mass acceleration from link kinematics.
\item Long-flight ballistic fit over runs of $\geq 20$ consecutive high-$z$ timesteps with RMSE $<5$\,cm. 
\end{enumerate}

Methods 1--3 disagreed on \texttt{source\_04} (e.g., apex $5.02$ vs.\ mode $8.19~\mathrm{m/s^2}$). Method 4 was inconclusive due to perpetual joint motion in \texttt{source\_10}. Method 5 was also inconclusive due to an inaccurate kinematic chain approximation. Method 6 yielded tight per-source standard deviations ($0.16$--$0.84~\mathrm{m/s^2}$), and the agent kept it. It placed \texttt{source\_10} at $9.10~\mathrm{m/s^2}$ (closest to Venus) and \texttt{source\_08} at $9.50~\mathrm{m/s^2}$ (closer to Earth), but also \texttt{source\_04} (true $g = 8$) at $6.31~\mathrm{m/s^2}$.

\subsubsection{Failure Modes of Gemini 3 Flash on Hopper}
\label{sec:appendix_failure_modes}

This subsection catalogs three recurring reasoning errors observed in Gemini 3 Flash subgroup queries that produced the false LIP (LIP-G) used in the Hopper ablation. Although EM overrides the resulting false LIP once target data accumulate ($N_0\geq1024$ in Table~\ref{tab:hopper}), in line with Sec.~\ref{sec:asymptotic_consistency}, this false LIP slows the cold-start recovery. We document these errors here because they are not random noise but specific reasoning patterns that recur across queries.

\paragraph{Failure mode 1: Anchoring bias on Earth gravity.} This is the dominant systematic error. Gemini 3 Flash picks one source from the subgroup, assumes it represents Earth ($g = 9.81~\mathrm{m/s^2}$), and then identifies the source approximately $1~\mathrm{m/s^2}$ below it as Venus. Such a selection rule has no physical or statistical basis.

When the anchor is wrong, the entire chain of inference is wrong. In query Q11 (of the $50$ Hopper queries), the subgroup contained \texttt{source\_10} (true $g = 9$), \texttt{source\_04} (true $g = 8$), and \texttt{source\_08} (true $g = 10$):
\begin{quote}\itshape
\small
\texttt{source\_04} has the lowest gravity, with \texttt{source\_10} being $1.0~\mathrm{m/s^2}$ higher and \texttt{source\_08} being $2.0~\mathrm{m/s^2}$ higher. Given that Venus gravity ($8.87~\mathrm{m/s^2}$) is approximately $1.0~\mathrm{m/s^2}$ less than Earth gravity ($9.81~\mathrm{m/s^2}$), \texttt{source\_04} is the most dynamically consistent candidate for a Venus-like environment, assuming \texttt{source\_10} represents the Earth baseline.
\end{quote}

Even worse, in Q48, the subgroup was $[g{=}1, g{=}4, g{=}5, g{=}10]$, in which \texttt{source\_08} ($g = 10$) is the only proximate match for Venus. Gemini 3 Flash nonetheless took \texttt{source\_07} ($g = 5$) as Earth and picked \texttt{source\_05} ($g = 4$):
\begin{quote}\itshape
\small
\texttt{Source\,05} consistently shows an acceleration $1.0~\mathrm{m/s^2}$ greater (less negative) than \texttt{Source\,07}. Assuming \texttt{Source\,07} represents Earth gravity ($9.81~\mathrm{m/s^2}$), \texttt{Source\,05}'s gravity is approximately $8.81~\mathrm{m/s^2}$.
\end{quote}

\paragraph{Failure mode 2: Noisy selection from a single row.} A correct estimate of gravity from the replay buffer requires aggregating across many free-fall transitions. Gemini 3 Flash sometimes commits to a single row whose acceleration matches the target, ignoring the rest of the trajectory. In Q27, the subgroup was $[g{=}8, g{=}5, g{=}2]$ and Gemini 3 Flash chose $g = 2$:
\begin{quote}\itshape
\small
In \texttt{source\_09}, analyzing Episode~$19$ from step~$7$ to $8$, the vertical velocity ($\dot{z}$) changes from $-0.1471~\mathrm{m/s}$ to $-0.2181~\mathrm{m/s}$ over a $0.008\,$s interval, yielding an acceleration of $(-0.2181 - (-0.1471))/0.008 = -8.875~\mathrm{m/s^2}$, which is dynamically consistent with the target environment.
\end{quote}

\paragraph{Failure mode 3: Arithmetic errors mid-reasoning.} Gemini 3 Flash sometimes makes arithmetic errors in its own chain of thought and does not notice. Example from Q23:
\begin{quote}\itshape
\small
\texttt{Source\,08} consistently shows a vertical acceleration of approximately $-8.89~\mathrm{m/s^2}$ (e.g., in Episode~$77$, $\dot{z}$ changes from $-0.0021$ to $-0.0396$ to $-0.0771$, yielding a constant $\Delta\dot{z}$ of $-0.0375$ per $0.008\,$s step), which closely matches the target Venus-like gravity of $8.87~\mathrm{m/s^2}$.
\end{quote}
The stated $\Delta\dot{z}$ in fact gives $0.0375/0.008\approx4.69~\mathrm{m/s^2}$, not $8.89~\mathrm{m/s^2}$.

\end{document}